\documentclass{article}
\usepackage{kr}

\usepackage{times}
\usepackage{soul}
\usepackage{url}
\usepackage[hidelinks]{hyperref}
\usepackage[utf8]{inputenc}
\usepackage[small]{caption}
\usepackage{graphicx}
\usepackage{amsmath}
\usepackage{amsthm}
\usepackage{booktabs}
\usepackage{algorithm}
\usepackage{algorithmic}
\newtheorem{example}{Example}
\newtheorem{theorem}{Theorem}
\newtheorem{lemma}{Lemma}

\usepackage{etoolbox}
\makeatletter
\let\llncs@addcontentsline\addcontentsline
\patchcmd{\maketitle}{\addcontentsline}{\llncs@addcontentsline}{}{}
\patchcmd{\maketitle}{\addcontentsline}{\llncs@addcontentsline}{}{}
\patchcmd{\maketitle}{\addcontentsline}{\llncs@addcontentsline}{}{}
\makeatother

\usepackage{xspace}
\usepackage{pifont}
\usepackage{thm-restate}
\usepackage{color}
\usepackage{amssymb}
\usepackage{misc}   %custom for logic
\usepackage{xargs}
\usepackage[pdftex,dvipsnames]{xcolor}
\usepackage{mathtools}
\usepackage{stmaryrd}               % short arrows
\usepackage[capitalise]{cleveref}
\usepackage[shortlabels]{enumitem} 				   	% empty item labels
\usepackage{mathrsfs}               	% provides \mathscr math font
\usepackage{etoolbox}             % for setting align spacing
\usepackage[normalem]{ulem}              % strikethrough text

\usepackage[colorinlistoftodos,prependcaption,textsize=tiny]{todonotes}
\newcommandx{\unsure}[2][1=]{\todo[linecolor=red,backgroundcolor=red!25,bordercolor=red,#1]{#2  }}
\newcommandx{\change}[2][1=]{\todo[linecolor=blue,backgroundcolor=blue!25,bordercolor=blue,#1]{#2}}
\newcommandx{\info}[2][1=]{\todo[linecolor=OliveGreen,backgroundcolor=OliveGreen!25,bordercolor=OliveGreen,#1]{#2}}
\newcommandx{\improvement}[2][1=]{\todo[linecolor=Plum,backgroundcolor=Plum!25,bordercolor=Plum,#1]{#2}}

\newcommand{\displayskips}{%
  \setlength{\abovedisplayskip}{4pt}%
  \setlength{\belowdisplayskip}{4pt}%
  \setlength{\abovedisplayshortskip}{4pt}%
  \setlength{\belowdisplayshortskip}{4pt}}
\appto{\normalsize}{\displayskips}
\appto{\small}{\displayskips}
\appto{\footnotesize}{\displayskips}

\newcommand{\rfn}{\mathrel{\underline{\mn{refines}}}}
\newcommand{\abs}{\mathrel{\underline{\mn{abstracts}}}}

\title{Description Logics with Abstraction and Refinement}

\author{%
  Carsten Lutz$^{1,2}$
  \and
Lukas Schulze$^1$ \\
\affiliations
$^1$
Department of Computer Science, Leipzig University, Germany\\
$^2$Center for Scalable Data Analytics and Artificial Intelligence (ScaDS.AI)\\
\emails
\{clu, lschulze\}@informatik.uni-leipzig.de
}

\begin{document}

\maketitle

\begin{abstract}
  Ontologies often require knowledge representation on multiple
  levels of abstraction, but description logics (DLs) are not
  well-equipped for supporting this. We propose an extension of DLs in
  which abstraction levels are first-class citizens and which provides
  explicit operators for the abstraction and refinement of concepts
  and roles across multiple abstraction levels, based on conjunctive
  queries. We prove that reasoning in the resulting family of DLs is
  decidable while several seemingly harmless variations turn out
  to be undecidable.  We also pinpoint the precise complexity of our
  logics and several relevant fragments.
\end{abstract}

\section{Introduction}

Abstraction and refinement is an important topic in many subfields of
computer science such as systems
verification~\cite{DBLP:reference/mc/DamsG18}. The same is true
for ontology design because ontologies often refer to different levels
of abstraction (or equivalently, levels of granularity). To name only
one example, the widely known medical ontology {\sc SnoMed CT}
contains the concepts \mn{Arm}, \mn{Hand}, \mn{Finger}, \mn{Phalanx},
\mn{Osteocyte}, and \mn{Mitochondrion} which intuitively all belong to
different (increasingly finer) levels of abstraction
\cite{DBLP:conf/amia/StearnsPSW01}.
%
% {\color{red}More such examples
% would be useful. Gene Ontology? Other domains? SOMA?}
%
Existing ontology languages, however, do not provide explicit support
for modeling across different abstraction levels. The aim of this
paper is to introduce a family of description logics (DLs) that provide
such support in the form of abstraction and refinement operators.

We define the \emph{abstraction DL} $\ALCHI^{\mn{abs}}$ as an extension of
the familiar description logic \ALCHI, which may be viewed as a modest
and tame core of the OWL 2 DL ontology language. %~\cite{owl}.
In
principle, however, the same extension can be applied to any other DL,
both more and less expressive than \ALCHI. Abstraction levels are
explicitly named and referred to in the ontology, and we provide
explicit operators for the abstraction and refinement of concepts and
roles.  For example, the concept refinement
  $$
    L_2{:}q_A \rfn L_1{:}\mn{Arm},
  $$
  where $q_A$ denotes the conjunctive query (CQ)
  $$
  \begin{array}{r@{\;}c@{\;}l}
     q_A &=&\mn{UArm}(x_1) \wedge \mn{LArm}(x_2) \wedge \mn{Hand}(x_3)\, \wedge
    \\[1mm]
     &&\mn{joins}(x_1,x_2) \wedge \mn{joins}(x_2,x_3),
  \end{array}
  $$
  expresses that every instance of \mn{Arm} on the coarser
  abstraction level $L_1$ decomposes into an ensemble of three objects on
  the finer level $L_2$ as described by $q_A$.
  % the conjunctive query (CQ)
  % %
  % $$
  % \begin{array}{r@{\;}c@{\;}l}
  %    q_A &=&\mn{UArm}(x_1) \wedge \mn{LArm}(x_2) \wedge \mn{Hand}(x_3)\, \wedge
  %   \\[1mm]
  %    &&\mn{joins}(x_1,x_2) \wedge \mn{joins}(x_2,x_3).
  % \end{array}
  % $$
%
  Concept abstractions are dual to refinements and state that every
  ensemble of a certain kind on a finer level gives rise to an
  abstracting object on a coarser level.  Semantically, there is
  one classical DL interpretation for each abstraction level and a
  (partial) refinement function that associates objects on coarser
  levels with ensembles on finer levels.  Every object may
  participate in at most one ensemble and we require that abstraction
  levels are organized in the form of a tree. We believe that the abstraction
  DLs defined along these lines are useful for many application
  domains, examples are given in the paper. Though not
  limited to it, our DLs are particularly well-suited for capturing
  the mereological (part-whole) aspect of abstraction and refinement \cite{artale1996part}.

  Our main technical contribution is to show that adding abstraction
  and refinement to \ALCHI preserves decidability of the base logic,
  and to provide a detailed analysis of its complexity, also
  considering many relevant fragments. It turns out that
  satisfiability in $\ALCHI^{\mn{abs}}$ is 2\ExpTime-complete.  Note
  that this is in line with the fact that CQ evaluation in \ALCHI is
  2\ExpTime-complete \cite{DBLP:conf/cade/Lutz08}. For the fragments,
  however, such parallels to evaluation cease to hold. We use
  $\ALCHI^{\mn{abs}}[\textnormal{cr}]$ to denote $\ALCHI^{\mn{abs}}$
  in which only concept refinement is admitted, and likewise ca
  denotes concept abstraction and rr, ra denote role refinement and
  abstraction. We recall that \ALC is \ALCHI without inverse roles
  and role hierarchies.

  We show that satisfiability in the natural fragment
  $\ALCHI^{\mn{abs}}[\textnormal{cr}]$ is only \ExpTime-complete,
  despite the fact that it still comprises CQs. Moreover,
  2\ExpTime-hardness already holds for $\ALC^{\mn{abs}}$ in contrast
  to the fact that CQ evaluation in \ALC is in \ExpTime. There are
  actually three different sources of complexity as
  %which become visible when
  %considering its fragments:
%
\begin{figure}[t]
  \centering
    \begin{tabular}{|c|c|c|c|c|c|}
      \hline
      \rule{0pt}{1.8ex}Base DL  & cr & ca & rr & ra &  Complexity \\\hline \hline
      \rule{0pt}{2.1ex}\ALCHI    & X  &    &    &    &  \hyperref[sec:alci_cr_exp_ub]{in \ExpTime} \\\hline
      \rule{0pt}{2.1ex}\ALCHI    & X  &  X & X  &  X &  \hyperref[sec:alci_all_ar_statements_ub]{in 2\ExpTime} \\\hline
      \rule{0pt}{2.1ex}\ALC     &    &  X &    &    &  \hyperref[sec:alc_ca_2exp_lb]{2\ExpTime-hard} \\\hline
      \rule{0pt}{2.1ex}\ALC     &    &    &  X &    &  \hyperref[sec:rr_2exp_lb]{2\ExpTime-hard} \\\hline
      \rule{0pt}{2.1ex}\ALC     &    &    &    & X  &  \hyperref[sec:alc_ra_2exp_lb]{2\ExpTime-hard} \\\hline
    \end{tabular}
  \caption{The complexity of satisfiability in abstraction DLs.}
  \label{overview}
\vspace*{-5mm}
\end{figure}
satisfiability is 2\ExpTime-hard already in each of the fragments
$\ALC^{\mn{abs}}[\textnormal{ca}]$,
$\ALC^{\mn{abs}}[\textnormal{rr}]$, and
$\ALC^{\mn{abs}}[\textnormal{ra}]$.
In
$\ALC^{\mn{abs}}[\textnormal{ra}]$, role abstractions allow us to
recover inverse roles. The same is true for
$\ALC^{\mn{abs}}[\textnormal{ca}]$ that, however, requires a more
subtle reduction relying on the fact that ensembles must not overlap.
Finally, $\ALC^{\mn{abs}}[\textnormal{rr}]$ is 2\ExpTime-hard because
role refinements allow us to generate objects interlinked in a complex
way. See Figure~\ref{overview} for a summary.

We then observe that the decidability of $\ALCHI^{\mn{abs}}$ is more
fragile than it may seem on first sight and actually depends on a
number of careful design decisions. In particular, we consider three
natural extensions and variations and show that each of them is
undecidable.  The first variation is to drop the requirement that
abstraction levels are organized in a tree. The second is to add the
requirement that ensembles (which are tuples rather than sets) must
not contain repeated elements. And the third is to drop the
requirement that CQs in abstraction and refinement statements must be
full, that is, to admit quantified variables in such CQs.

Proofs are in the appendix.

\medskip
\noindent
{\bf Related Work.}
A classic early article on granularity in AI is
\cite{DBLP:conf/ijcai/Hobbs85}.  Granularity has later been studied in
the area of foundational ontologies, focussing on a philosophically
adequate modeling in first-order logic. Examples include granular
partitions \cite{DBLP:books/tf/03/Bittner003}, the descriptions and
situations framework \cite{DBLP:conf/coopis/GangemiM03}, and
domain-specific approaches
\cite{DBLP:journals/amai/FonsecaEDC02,DBLP:conf/mie/SchulzBS08,DBLP:journals/biomedsem/Vogt19}.
% biology, medicine
% \cite{DBLP:conf/mie/SchulzBS08}, and geographic information systems
% \cite{DBLP:journals/amai/FonsecaEDC02}.

Existing approaches to representing abstraction and refinement /
granularity in DLs and in OWL are rather different in spirit. Some
are based on rough or fuzzy set theory
\cite{DBLP:journals/wias/KlinovTM08,DBLP:journals/fuin/LisiM18}, some
provide mainly a modeling discipline
\cite{DBLP:journals/ijar/CalegariC10}, some aim at the spatial domain
\cite{DBLP:conf/esws/HbeichRB21} or at speeding up reasoning
\cite{DBLP:conf/aaai/GlimmKT17}, and some take abstraction to mean the
translation of queries between different data schemas
\cite{DBLP:conf/aaai/CimaCLP22}. We also mention description logics of
context \cite{DBLP:journals/logcom/KlarmanG16}; an abstraction level
can be seen as a context, but the notion of context is more general
and governed by looser principles. A categorization
of different forms of granularity is in \cite{Keet}.

There is a close connection between DLs with
abstraction as proposed in this paper and the unary negation fragment
of first-order logic (UNFO). In fact, UNFO encompasses ontologies
formulated in DLs such as \ALCI and conjunctive queries. UNFO
satisfiability is decidable and 2\ExpTime-complete
\cite{DBLP:journals/corr/SegoufinC13}. This does, however, not imply any of the
results in this paper due to the use of refinement functions
in the semantics of our DLs and the fact that UNFO extended with
functional relations is undecidable \cite{DBLP:journals/corr/SegoufinC13}.

\section{Preliminaries}
\label{sect:prelims}

\paragraph{Base DLs.}
% We first recall the syntax and semantics of the classical DL \ALCIH
% and some of its fragments.
Fix countably infinite sets \Cbf and \Rbf of \emph{concept names}
and \emph{role names}. A \emph{role} is a role name or an
\emph{inverse role}, that is, an expression $r^-$ with $r$ a role
name. If $R=r^-$ is an inverse role, then we set $R^- =
r$. \emph{\ALCI-concepts} $C,D$ are built according to the syntax rule
$$
     C,D::= A \mid \neg C \mid C \sqcap D \mid C \sqcup D \mid \exists
     R . C \mid \forall R . C
$$
where $A$ ranges over concept names and $R$ over roles. We use $\top$
as an abbreviation for $A \sqcup \neg A$ with $A$ a fixed concept name
and $\bot$ for $\neg \top$. An \emph{\ALC-concept} is an \ALCI-concept
that does not use inverse roles and an \emph{\EL-concept} is an
\ALC-concepts that uses none of $\neg$, $\sqcup$, and $\forall$.

An \ALCHI-\emph{ontology} is a finite set of \emph{concept inclusions
  (CIs)} $C \sqsubseteq D$ with $C$ and~$D$ \ALCI-concepts and
\emph{role inclusions (RIs)}, $R \sqsubseteq S$ with $R,S$ roles. The
letter \Imc indicates the presence of inverse roles and \Hmc indicates
the presence of role inclusions (also called role hierarchies), and
thus it should also be clear what we mean e.g.\ by an \ALCI-ontology
and an $\mathcal{ALCH}$-ontology. An \emph{\EL-ontology} is a finite
set of CIs $C \sqsubseteq D$ with $C,D$ \EL-concepts.

An \emph{interpretation} is a pair $\Imc = (\Delta^\Imc, \cdot^\Imc)$
with $\Delta^\Imc$ a non-empty set (the \emph{domain}) and
$\cdot^\Imc$ an \emph{interpretation function} that maps every concept
name $A \in \textbf{C}$ to a set
$A^\Imc \subseteq \Delta^\Imc$ and every role name $r \in \textbf{R}$ to a binary relation \mbox{$r^\Imc \subseteq \Delta^\Imc \times \Delta^\Imc$}.
			The interpretation function is extended to
                        compound concepts as usual, c.f.~\cite{DBLP:books/daglib/0041477}.
%
%                         by setting:
% 	\begin{align*}
% 		% \top^\Imc &= \Delta^\Imc\\
% 		% \bot^\Imc &= \emptyset\\
% 		(\neg C)^\Imc &= \Delta^\Imc \setminus C^\Imc\\
% 		(C \sqcap D)^\Imc &= C^\Imc \cap D^\Imc\qquad
% 		(C \sqcup D)^\Imc = C^\Imc \cup D^\Imc\\
% 		(\exists r.C)^\Imc &= \{d \mid \exists
%                                      e \in \Delta^\Imc: (d,e) \in r^\Imc \text{ and } e \in C^\Imc\}\\
% 		(\forall r.C)^\Imc &= \{d \mid \forall e \in
%                                      \Delta^\Imc: (d,e) \in r^\Imc
%                                      \text{ implies } e \in C^\Imc\}.
% 	\end{align*}
% %
%
An interpretation \Imc \emph{satisfies} a CI $C \sqsubseteq D$ if
$C^\Imc \subseteq D^\Imc$ and
an RI $R \sqsubseteq S$ if $R^\Imc \subseteq S^\Imc$. It is a
\emph{model} of an ontology \Omc if it satisfies all CIs and RIs
in it.%  For a CI or RI $\alpha \sqsubseteq \beta$, we write $\Omc
% \models \alpha \sqsubseteq \beta$ if every model of \Omc
% satisfies $\alpha \sqsubseteq \beta$.

For any syntactic object $O$ such as an ontology or a concept,
we use $||O||$ to denote the \emph{size} of $O$,
that is, the number of symbols needed to write $O$ over a suitable
alphabet.

\medskip
\noindent
{\bf Conjunctive Queries.}
Let \Vbf be a countably infinite set of variables.  A
\emph{conjunctive query (CQ)} takes the form
$q(\bar x) = \exists \bar y \, \vp(\bar x, \bar y)$ with $\varphi$ a
conjunction of \emph{concept atoms} $C(x)$ and \emph{role atoms}
$r(x,y)$, % , and {\color{blue}inequalities
% $x \neq y$}
$C$ a (possibly compound) concept, $r$ a role name, and $x,y$
variables from $\bar x \cup \bar y$. We may write $\alpha \in q$ to
indicate that $\alpha$ is an atom in~$\varphi$ and $r^-(x,y) \in q$ in place
of $r(y,x) \in q$. The variables in $\bar x$ are the \emph{answer
  variables} of~$q$. We require that every answer variable $x$ occurs
in some atom of~$q$, but omit this atom in writing in case it is
$\top(x)$.  With $\mn{var}(q)$, we denote the set of all (answer and
quantified) variables in $q$. If $q$ has no answer variables then it
is \emph{Boolean}.
%
% We do not generally require that CQs
% are \emph{safe}, which means that every answer variable of $q$ must occur
% in some atom of~$q$.
% If a variables does not occur in an atom, one may
% think of it as occurring in an atom 
 We mostly restrict our
attention to CQs $q$ that are \emph{full}, meaning that $q$ has
no quantified variables.  A CQ $q$ is
\emph{connected} if the undirected graph with node set $\mn{var}(q)$
and edge set
$\{\{v, v'\} \mid r(v, v') \in q \text{ for any } r \in \Rbf\}$ is.
% and sometimes write $r^-(x,y) \in q$ in place of \mbox{$r(y,x) \in q$}.
% CQ $q$ is a \emph{tree} if $\Dmc_q$ is.
% For
% $V \subseteq \mn{var}(q)$, we use $q|_V$ to denote the restriction of
% $q$ to the atoms that use only variables in~$V$.
%
A CQ $q$ is a \emph{subquery} of a CQ $q'$ if $q$ can be obtained
from $q'$ by dropping atoms.

Let $q(\bar x) = \exists \bar y \, \vp(\bar x, \bar y)$ be a CQ and
$\Imc$ an interpretation. A mapping
$h:\bar x \cup \bar y \rightarrow \Delta^\Imc$ is a
\emph{homomorphism} from $q$ to \Imc if $C(x) \in q$ implies
$h(x) \in C^{\Imc}$ and $r(x,y) \in q$ implies
$(h(x),h(y)) \in r^{\Imc}$.  A tuple
$\bar d \in (\Delta^\Imc)^{|\bar x|}$
% (with $|\bar x|$ the length of tuple $\bar x$)
is an \emph{answer} to $q$ on \Imc %, written
% $\Imc \models q(\bar d)$,
if there is a homomorphism $h$ from $q$ to
\Imc with $h(\bar x)=\bar d$. 
We use $q(\Imc)$ to denote the set of all answers to $q$ on \Imc.
If $q$ is Boolean, we write $\Imc \models q$ to indicate the existence
of a homomorphism from $q$ to \Imc.

% {\color{blue} For any function $f$ we use $\dom(f)$ and $\ran(f)$ to denote the domain and range of $f$ respectively.}

% \medskip 

%{\color{blue} add negated role atoms? maybe not a big deal?}

\section{DLs with Abstraction and Refinement}
\label{sect:abstractionDLs}

We extend \ALCIH to the DL $\ALCIH^{\mn{abs}}$ that supports
abstraction and refinement.
% Other DLs can be extended in an
% analogous way.
%
% Let \Lmc be one of the description logics introduced above. We extend
% \Lmc to a DL $\Lmc^{\mn{abs}}$ that supports abstraction and
% refinement.
Fix a countable set $\Abf$ of \emph{abstraction
  levels}. 
%
% The elements $L,L'$ of $\Abf$
% represent different abstraction levels with different modeling
% granularity and $L \leq L'$ means that $L$ is less abstract than $L'$
% or, in other words, that the modeling granularity of $L$ is finer than
% that of $L'$.  We write $L \prec L'$ if $L$ is an immediate
% abstraction of $L'$, that is, $L \leq L'$, $L \neq L'$, and there is
% no $L'' \in \Abf$ with $L \leq L'' \leq L'$.
% Let us also
% fix a set of pairwise disjoint signatures (finite set of concept and
% role names) $\Sigma_\Lscr = \{\Sigma_1, \dots, \Sigma_n\}$.
%
An $\ALCIH^{\mn{abs}}$-ontology is a finite set of statements of the
following form:
\begin{itemize}
\item \emph{labeled concept inclusions} $C \sqsubseteq_L D$, 
\item \emph{labeled role inclusions} $R \sqsubseteq_L S$, 
\item \emph{concept refinements}
$L{:}q(\bar x) \rfn L'{:}C$,
\item \emph{concept abstractions}
  $L'{:}C \abs L{:}q(\bar x)$,
\item \emph{role refinements} $L{:}q(\bar x, \bar y)
  \rfn L'{:}q_R(x,y)$,
\item \emph{role abstractions}
  $L'{:}R \abs  L{:}q(\bar x, \bar y)$
\end{itemize}
where $L,L'$ range over \Abf, $C,D$ over
\ALCI-concepts, $R,S$ over roles, $q$ over full conjunctive queries,
and $q_R$ over full conjunctive queries of the form
$C_1(x) \wedge R(x,y) \wedge C_2(y)$. In concept and role abstraction
statements, we additionally require the CQ $q$ to be connected.
We may write $C \equiv_L D$ as shorthand for the two CIs
$C \sqsubseteq_L D$ and $D \sqsubseteq_L C$. We underline abstraction
and refinement operators to ensure better readability throughout the paper.

Intuitively, a concept refinement $L{:}q(\bar x) \rfn L'{:}C$
expresses that any instance of $C$ on abstraction level $L'$ refines
into an \emph{ensemble} of $|\bar x|$ objects on abstraction level
$L$ which satisfies all properties expressed by CQ
$q$. Conversely, a concept abstraction $L'{:}C \abs L{:}q(\bar x)$
says that any ensemble of $|\bar x|$ objects on abstraction level $L$
that satisfies $q$ abstracts into a single instance of $C$ on
abstraction level $L'$. Role refinements and abstractions can be 
understood in a similar way, where each of the two elements
that participate in a role relationship refines into its own ensemble.

Note that in role refinements, we consider CQs 
$q_R= C_1(x) \wedge R(x,y) \wedge C_2(y)$ rather than only the role 
$R$. This is because roles are often of a general 
kind such as \mn{partOf} or \mn{interactsWith} and need proper context 
to be meaningfully refined. This context is provided by the concepts 
$C_1,C_2$. 
\begin{example}
  \label{ex:first}
  Granularity is important in many domains. Anatomy has already
  been mentioned in the introduction. The concept refinement
  given there may be complemented by choosing $q_A$ as in the
  introduction and adding
  the concept abstraction
  % the concepts \mn{Body}, \mn{Arm}, \mn{Hand}, \mn{Finger},
  % \mn{Phalange}, \mn{BoneSubstance}, \mn{Osteocyte},
  % \mn{Mitochondrion} are naturally conceived as belonging to different
  % (increasingly finer) levels of abstraction. A
  %
  $$
    L_1{:}\mn{Arm} \abs L_2{:}q_A.
  $$

  We next consider bikes as a simple example for a technical domain.  Let
  us first say how wheels refine into components:
  $L_2{:}q_W \rfn L_1{:}\mn{Wheel}$ 
  where
  $$
  \begin{array}{r@{\;}c@{\;}l}
     q_W &=&\mn{Axle}(x_1) \wedge \mn{Spokes}(x_2) \wedge
             \mn{Rim}(x_3) \wedge \mn{Tire}(x_4)\, \wedge
    \\[1mm]
         &&\mn{join}(x_2,x_1) \wedge \mn{join}(x_2,x_3) \wedge
            \mn{carries}(x_3,x_4).
  \end{array}
  $$
  We may then use the following role refinement to express how
  frames connect to wheels:
  $$
  L_2{:}q_{FW} \rfn L_1{:}\, \mn{Wheel}(x) \wedge \mn{connTo}(x,y)
  \wedge \mn{Frame}(y)
  $$
  where, for $\bar x = x_1 \cdots x_4$ and $\bar y = 
   y_1 \cdots y_7$ (assuming that frames have seven components),
  $$
  \begin{array}{@{}r@{\;}c@{\;}l}
  q_{FW}(\bar x,\bar y) &=& \mn{Axle}(x_1) \wedge
  \mn{connTo}(x_1,y_1) \wedge
    \mn{Dropout}(y_1).
  \end{array}
  $$
  This expresses that if a wheel is connected to a frame, then the
  axle of the wheel is connected to the dropout of the frame.
  % The
  % reader
  % might want to return to the example after having read the semantics.
  % We emphasize that every object refines into a unique ensemble and
  % that the Axle $x_1$ in $q_W$ (intuitively) denotes the same object as the axle
  % $x_1$ in $q_{FW}$.
\end{example}
Extensions $\Lmc^{\mn{abs}}$ of other DLs \Lmc introduced in
Section~\ref{sect:prelims}, such as \ALC and $\mathcal{ALCH}$, may be
defined in the expected way.  We also consider various fragments of
$\ALCIH^{\mn{abs}}$. With $\ALCIH^{\mn{abs}}[\textnormal{cr,rr}]$, for
example, we mean the fragment of $\ALCIH^{\mn{abs}}$ that admits concept
refinement and role refinement, but neither concept abstraction nor
role abstraction (identified by ca and ra).

%
  % contain an atom of the form $R(x,y)$. The use of
  % this CQs allows us to further qualify role edges to be refined,
  % e.g.\ by using $q_R=C(x) \wedge R(x,y) \wedge C'(y)$ for some
  % suitable $C,C'$.}
% All abstraction and refinement statements also
% exist in a strong form indicated by superscript $\cdot^\ast$, as in
% $L'{:}C ~\mn{refinedBy}^\ast~ L{:}q(\bar x)$.

We next define the semantics of $\ALCIH^{\mn{abs}}$, based on
\emph{A-interpretations} which include one traditional DL
interpretation for each abstraction level. Formally, an
A-interpretation takes the form
\newcommand{\AI}{\Abf_\Imc}
$\Imc = (\AI,\prec,(\Imc_L)_{L \in {\AI}},\rho)$, where
\begin{itemize}

\item $\AI \subseteq \Abf$ is the set of relevant abstraction
  levels;
  
\item ${\prec} \subseteq \AI \times \AI$ is such that the
  directed graph $(\AI,\{(L',L)\mid L \prec L' \})$ is a
  tree; intuitively, $L \prec L'$ means that $L$ is less abstract than
  $L'$ or, in other words, that the modeling granularity of $L$ is
  finer than that of $L'$;
  
  \item $(\Imc_L)_{L \in \AI}$ is a collection of interpretations
    $\Imc_L$, one for every $L \in \AI$, with pairwise disjoint
    domains; we use $L(d)$ to denote the unique $L \in \AI$ with
    $d \in \Delta^{\Imc_L}$; 

  \item $\rho$ is the \emph{refinement function}, a partial function
    that associates pairs $(d,L) \in \Delta^\Imc \times \AI$ such that
    $L \prec L(d)$ with an \emph{$L$-ensemble} $\rho(d,L)$, that is,
    with a non-empty tuple over $\Delta^{\Imc_L}$. We want every
    object to participate in only one ensemble and thus require
    that
    \begin{itemize}

    \item[$(*)$] for all $d \in \Delta^\Imc$, there
    is at most one $e \in \Delta^{\Imc}$ such that $d$ occurs in
    $\rho(e,{L(d)})$.
    
    \end{itemize}
    For readability, we may write $\rho_L(d)$ in place of $\rho(d,L)$.

\end{itemize}
%
% Intuitively, $L \prec L'$ means that $L$ is less abstract than $L'$ or,
% in other words, that the modeling granularity of $L$ is finer than
% that of $L'$. 
% Also, $\rho(d)=\bar e$ means that the object $d$ on
% level $L(d)$ refines into the ensemble of objects $\bar e$ on level
% $L(\bar e)$.
% Note that $(\Delta^\Imc,\cdot^\Imc)$ is a classical
% interpretation where $\cdot^\Imc=\bigcup_{L \in \Abf} \cdot^{\Imc_L}$.
%
%\medskip
%
% We next define the semantics of $\Lmc^{\mn{abs}}$-ontologies. 
An A-interpretation $\Imc = (\AI,\prec,(\Imc_L)_{L \in \AI},\rho)$
\emph{satisfies} a
\begin{itemize}

\item labeled
  concept or role inclusion $\alpha \sqsubseteq_L \beta$
  if
  $L \in \AI$ and 
  $\alpha^{\Imc_L} \subseteq \beta^{\Imc_L}$;

\item concept refinement
  $L{:}q(\bar x) \rfn L'{:}C$ if $L \prec L'$ and  for all
  $d \in C^{\Imc_{L'}}$, there is an $\bar e \in q(\Imc_L)$ such that
  $\rho_L(d) =\bar e$;

\item concept abstraction
  $L'{:}C \abs L{:}q(\bar x)$ if $L \prec L'$ and  for all
  $\bar e \in q(\Imc_{L})$, there is a $d \in C^{\Imc_{L'}}$ s.t.\
  $\rho_L(d) = \bar e$;
  
\item role refinement
  $L{:}q(\bar x, \bar y) \rfn L'{:}q_R(x,y)$ if $L \prec L'$ and  for all
  $(d_1,d_2) \in q_R(\Imc_{L'})$, there is an
  $(\bar e_1,\bar e_2) \in q(\Imc_L)$ such that $\rho_L(d_1) =\bar e_1$
  and $\rho_L(d_2) =\bar e_2$;

\item role abstraction
  $L'{:}R \abs  L{:}q(\bar x, \bar y)$ if $L \prec L'$ and  for
  all $(\bar e_1,\bar e_2) \in q(\Imc_{L})$, there is a
  $(d_1,d_2) \in R^{\Imc_{L'}}$ such that
  $\rho_L(d_1) =\bar e_1$ and
  $\rho_L(d_2) =\bar e_2$.

\end{itemize}
An A-interpretation is a \emph{model} of an
$\ALCHI^{\mn{abs}}$-ontology if it satisfies all inclusions,
refinements, and abstractions in it.  % Let $L \in \Abf$, $C$ be an
\begin{example}
  We consider the domain of (robotic) actions. Assume that there is a
  \mn{Fetch} action that refines into subactions:
  $L_2{:}q_F \rfn L_1{:}\mn{Fetch}$ where
  $$
  \begin{array}{r@{\;}c@{\;}l}
     q_F &=&\mn{Locate}(x_1) \wedge \mn{Move}(x_2) \wedge
             \mn{Grasp}(x_3) \, \wedge
    \\[1mm]
         &&\mn{precedes}(x_1,x_2) \wedge \mn{precedes}(x_2,x_3).
  \end{array}
  $$
  We might have a safe version of the fetching action and a
  two-handed grasping action:
  $$
  \begin{array}{r@{\;}c@{\;}l}
    \mn{SFetch} &\sqsubseteq_{L_1}& \mn{Fetch} \\[1mm]
    \mn{TwoHandedGrasp} &\sqsubseteq_{L_1}& \mn{Grasp}
  \end{array}
  $$
  A safe fetch requires a two-handed grasping subaction:
  $L_2{:}q_S \rfn L_1{:}\mn{SFetch}$ where for $\bar x = x_1 x_2 x_3$,
  $$
  q_S(\bar x) = \mn{TwoHandedGrasp}(x_3).
  $$
  We remark that abstraction statements need to be used with care
  since ensembles may not overlap, c.f.\ Condition~($*$). For example,
  the reader may want to verify that the following CI and concept
  abstraction have no model:
  $$
%  \begin{array}{l}
  \top \sqsubseteq_{L_2} \exists r . \exists r . \top
  \qquad
  %\\[1mm]
    L_1{:}\top \abs L_2{:}r(x,y).
 % \end{array}
  $$
\end{example}
We are interested in the problem of \emph{(concept)
  satisfiability} % defined as follows. A \emph{labeled
%   \ALCHI-concept} takes the form $L:C$ with $L \in \Abf$ and $C$ a
% concept. Now concept satisfiability in $\ALCHI^{\mn{abs}}$
which means to decide, given an $\ALCHI^{\mn{abs}}$-ontology \Omc, an
\ALCI-concept $C$, and an abstraction level $L \in \Abf$, whether
there is a model $\Imc$ of $\Omc$ such that
$C^{\Imc_{L}} \neq \emptyset$. We then say that $C$ is
\emph{$L$-satisfiable w.r.t.~$\Omc$}.  As usual, the related
reasoning problems of subsumption can be reduced to satisfiability in
polynomial time, and vice versa \cite{DBLP:books/daglib/0041477}.
%
% Another problem is
% \emph{subsumption} where we are given \Omc and $L$ as for
% satisfiability, along with two \ALCI-concept $C_1,C_2$, and have to
% decide whether $C_1^{\Imc_L} \subseteq C_2^{\Imc_L}$ for all models
% \Imc of~\Omc. We then say that $C_1$ is \emph{$L$-subsumed} by $C_2$
% w.r.t.\ \Omc. In $\ALCHI^{\mn{abs}}$, subsumption can be reduced to
% satisfiability in the usual way \cite{DBLP:books/daglib/0041477}.

\section{Upper Bounds}

We prove that satisfiability in $\ALCHI^\mn{abs}$ is decidable in
2\ExpTime. Before
approaching this general case, however, we consider the fragment
$\ALCHI^\mn{abs}[\textnormal{cr}]$ and show that it is only
\ExpTime-complete.
%We start with the latter since its proof is simpler
%and prepares for the proof of the former result.

\subsection{\texorpdfstring{$\ALCHI^\mn{abs}[\textnormal{cr}]$ in ExpTime}{ALCHI\textasciicircum abs[cr] in ExpTime}}
\label{sec:alci_cr_exp_ub}

% {\color{blue} not sure if this section comes before or after the CR and RR algo, so i named it here ``mosaic'' as well
% for convenience, but it also uses some notation from that section e.g. $\mn{TP}_L(\Omc)$}
Our aim is to prove the following.
\begin{theorem}
  Satisfiability in $\ALCHI^\mn{abs}[\textnormal{cr}]$ is \ExpTime-complete.
\end{theorem}
The lower bound is inherited from \ALCHI without abstraction and
refinement \cite{DBLP:books/daglib/0041477}. We prove the upper bound
by a mosaic-based approach, that is, we decide the existence of a
model \Imc by trying to assemble \Imc from small fragments called
mosaics. Essentially, a mosaic describes a single ensemble on a single
level of abstraction.

Assume that we are given as input an
$\ALCHI^\mn{abs}[\textnormal{cr}]$-ontology \Omc, an \ALCI-concept
$C_0$, and an abstraction level~$L_0$. We may assume w.l.o.g.\ that
$C_0$ is a concept name as we can extend \Omc with
$A_0 \sqsubseteq_{L_0} C_0$ and test satisfiability of the fresh
concept name $A_0$. We also assume w.l.o.g.\ that \Omc is in
\emph{normal form}, meaning that
\begin{enumerate}
  
\item every CI has one of the forms
	\[
	\begin{array}{r@{\qquad}l@{\qquad}r}
		\top \sqsubseteq_L A
		&
		A \sqsubseteq_L \exists R . B
		&
		\exists R . B \sqsubseteq_L A
		\smallskip\\
		A_1 \sqcap A_2 \sqsubseteq_L A
		&
		A \sqsubseteq_L \neg B
		&
		\neg B \sqsubseteq_L A
	\end{array}
	\]
	where $A, A_1, A_2, B$ are concept names and $R$ is a
        role;

      \item in every concept refinement $L{:}q(\bar x) \rfn L'{:}C$,
        $C$ is a concept name, and so is $D$ in all concept atoms $D(x) \in q$.

\end{enumerate}
%
% \begin{lemma}
%   Every $\ALCHI^{\mn{abs}}$-ontology \Omc can be converted in
%   polynomial time into an $\ALCHI^{\mn{abs}}$-ontology $\Omc'$ in
%   normal form that is a conservative extension of \Omc.
% \end{lemma}
           %
It is in fact routine to show that every
$\ALCHI^\mn{abs}[\textnormal{cr}]$ ontology \Omc can be converted in
polynomial time into an $\ALCHI^\mn{abs}[\textnormal{cr}]$ ontology
$\Omc'$ in normal form that is a conservative extension of \Omc, see
e.g.\ \cite{maniere:thesis}.
We also assume that (i)~\Omc contains $R \sqsubseteq_L R$ for all
roles $R$ and abstraction levels $L$ in~\Omc,
(ii)~$R \sqsubseteq_L S$, $S \sqsubseteq_L T \in \Omc$ implies
$R \sqsubseteq_L T \in \Omc$, and % for all roles $R,S,T$, and
(iii)~$R \sqsubseteq_L S \in \Omc$ implies
\mbox{$R^- \sqsubseteq_L S^- \in \Omc$}.  With ${\prec}$ we denote the
smallest relation on $\Abf_\Omc$ such that $L \prec L'$ for all
$L{:}q(\bar x) \rfn L'{:}C$ in \Omc.

% An \emph{$L$-type for \Omc}, with $L \in \Abf_\Omc$, is a set
% $t$ of concept names from \Omc such that
% %
% \begin{enumerate}
% \item if $\top \sqsubseteq_L A \in \Omc$, then $A \in t$;

% \item if $A_1\sqcap A_2 \sqsubseteq_L A \in \Omc$ and $A_1,A_2 \in t$,
%   then $A \in t$;

% \item if $A \sqsubseteq_L \neg B \in \Omc$ and $A \in t$, then $B 
%   \notin t$;

% \item if $\neg B \sqsubseteq_L A \in \Omc$ and $B \notin t$, then $A \in t$.

% \end{enumerate}
% We use $\mn{TP}_L(\Omc)$ to denote the set of all $L$-types
% for \Omc.

Fix a domain $\Delta$ of cardinality $||\Omc||$.
A \emph{mosaic} is a pair $M = (L, \Imc)$ where
%\begin{itemize}
% \item
$L \in \Abf_\Omc$ is the abstraction level of the mosaic and
% \item
$\Imc$ is an interpretation with $\Delta^\Imc \subseteq \Delta$
  % \item $t:\Delta^\Imc \rightarrow \mn{TP}_L(\Omc)$ assigns an $L$-type
  % to every $d \in \Delta^{\Imc}$.
  % \end{itemize}
such that \Imc satisfies all CIs $C \sqsubseteq_L D$ in
\Omc and all RIs $R \sqsubseteq_L S$ in \Omc, with the
possible exception of CIs of the  form $A \sqsubseteq_L \exists r
  . B$.
% \begin{enumerate}
%
% \item \Imc satisfies all concept inclusions $C \sqsubseteq_L D$ in
%   \Omc, possibly except those of the form $A \sqsubseteq_L \exists r
%   . B$;
  %
% \item 
%   $R^{\Imc} \subseteq S^{\Imc}$, for all $R \sqsubseteq_L S \in \Omc$;
%
% % \item 
% %   $A \in t(d)$ iff
% %   $d \in A^{\Imc}$, for all concept names $A$;
%
% \item if $\exists R . A \sqsubseteq_L B \in \Omc$, $(d, e) \in R^{\Imc}$, and $e \in A^{\Imc}$, 
% then $d \in B^{\Imc}$.
%
% % \item for all $(d,e) \in R^{\Imc}$, $C \in t(e)$, and $\exists R .C \in
% %   \mn{cl}(\Omc)$, we have $\exists R . C \in t(d)$.{\color{blue} kann wahrscheinlich weg, siehe vorige algo section}
% \end{enumerate} 
%
We may write $L^M$ to denote $L$, 
and likewise for $\Imc^M$.
%Let $M = (L, \Imc)$ be a mosaic and $q(\bar x)$ a full conjunctive query.
% We call a mapping $h$ from $\mn{var}(q)$ to $\Delta^{\Imc}$ a \emph{type homomorphism from $q$ to $M$} if
% $C(x) \in q$ implies $C \in t(h(x))$, and $R(x,y) \in q$ implies
% $(h(x),h(y)) \in R^{\Imc}$. 
% A tuple
% $\bar d \in (\Delta^\Imc)^{|\bar x|}$
% is an \emph{answer} to $q$ on $M$ 
% if there is a type homomorphism $h$ from $\mn{var}(q)$ to $M$
% with $h(\bar x)=\bar d$. 
% We use $q(M)$ to denote the set of all answers to $q$ on $M$.
%
Let $\Mmc$ be a set of mosaics. 
We say that a mosaic $M = (L, \Imc)$ is \emph{good} in $\Mmc$ if for all 
$d \in \Delta^\Imc$ the following hold:
\begin{enumerate}
\item if $A \sqsubseteq_L \exists R . B \in \Omc$, $d \in A^\Imc$, and
  $d \notin (\exists R.B)^\Imc$, then there is an $M'=(L,\Imc') \in \Mmc$ and a
  $d' \in \Delta^{\Imc^{M'}}$ such that
  \begin{enumerate}
    \item $d' \in B^{\Imc'}$,
    \item if  $\exists S . A \sqsubseteq_L B \in \Omc$, $R \sqsubseteq_L S \in \Omc$, and $d' \in A^{\Imc'}$, then $d \in B^\Imc$;
    \item if $\exists S . A \sqsubseteq_L B \in \Omc$, $R^- \sqsubseteq_L S \in \Omc$, and $d \in A^{\Imc}$, then $d' \in B^{\Imc'}$;
  \end{enumerate}
  \item for every level $L' \in \Abf_\Omc$ such that 
   $$Q = \{q \mid L'{:}q(\bar x) \rfn L{:}A \in \Omc \text{ and } d
   \in A^\Imc\} \neq \emptyset,$$
  there is a mosaic $M' \in \Mmc$ with $M' = (L', \Imc')$ and a tuple $\bar e$ over $\Delta^{\Imc'}$ 
  such that  $\bar e \in q(\Imc')$ for all $q \in Q$. 
\end{enumerate}
%
% Regarding Condition~1, $\Omc \models R \sqsubseteq_L S$ means
% that every model \Imc of \Omc satisfies
% $R^{\Imc_L} \subseteq S^{\Imc_L}$.  It is easy to prove that
% $\Omc \models R \sqsubseteq_L S$ if $(R,S)$ is in the
% reflexive-transitive closure of the relation
% $\{ (R',S') \mid L{:}\, R' \sqsubseteq S' \in \Omc \}$.

We now formulate the actual decision procedure.
If the directed graph $(\Abf_\Omc,\{(L',L)\mid L \prec L' \})$
% $(\Abf_\Omc,\prec^{-1})$
is not a tree, we directly 
return `unsatisfiable'. 
Our algorithm first computes the set $\Mmc_0$ of all mosaics for $\Omc$
and
then repeatedly and exhaustively eliminates mosaics that are not good. Let
$\Mmc^*$ denote the set of mosaics at which this
process stabilizes.
\begin{restatable}{lemma}{lemcorrone}
  \label{lem:corr1}
  $C_0$ is $L_0$-satisfiable w.r.t.\ \Omc iff $\Mmc^*$ contains (i)~a mosaic
  $M$ with $L^M= L_0$ and $C_0^{\Imc^M} \neq \emptyset$ and (ii)~a mosaic $M$ with $L^M=L$, for every
  $L$ in $\Abf_\Omc$.
  % $C_0$ is satisfiable w.r.t.\ \Omc iff $\Mmc^*$ contains (i)~a mosaic
  % $M = (L_0,\Imc,t)$ with $C_0 \in t(d)$ for some
  % $d \in \Delta^{\Imc}$ and (ii)~a mosaic $M = (L,\Imc,t)$ for every
  % $L$ in $\Abf_\Omc$.
\end{restatable}
The algorithm thus returns `satisfiable' if Conditions~(i) and~(ii)
from Lemma~\ref{lem:corr1} are satisfied and `unsatisfiable'
otherwise. It is easy to see that the algorithm runs in single
exponential
time.

% next algo
% next algo
% next algo
% next algo
% next algo
% next algo
% next algo

\subsection{\texorpdfstring{$\ALCHI^\mn{abs}$}{ALCHI\textasciicircum
    abs} in 2ExpTime}
\label{sec:alci_all_ar_statements_ub}

Our aim is to prove the following.
\begin{theorem}
  \label{thm:twoexpupper}
  Satisfiability in $\ALCHI^\mn{abs}$ is decidable in 2\ExpTime.
\end{theorem}
A matching lower bound will be provided later on. We prove
Theorem~\ref{thm:twoexpupper} by a mosaic-based approach which is,
however, significantly more complex than the one used in the previous
section. In particular, a mosaic now represents a `slice' through an
A-interpretation that includes multiple abstraction levels and
multiple ensembles. % We again start with introducing relevant
% notions. 

Assume that we are given as input an $\ALCHI^\mn{abs}$-ontology \Omc,
an \ALCI-concept $C_0$, and an abstraction level~$L_0$. We again
assume that $C_0$ is a concept name. and \Omc is in normal form,
defined as in the previous section, but with the obvious counterparts
of Point~2 for role refinements and (concept and role)
abstractions. We also define the relation $\prec$ on $\Abf_\Omc$ as in
the previous section, except that we now consider concept and role
refinements, as well as concept and role abstractions, in the obvious
way.  Again, if the directed graph $(\Abf_\Omc,\{(L',L)\mid L \prec L' \})$
% $(\Abf_\Omc,\prec^{-1})$
is not a tree, then we directly
return `unsatisfiable'.

Fix a set $\Delta$ of cardinality $||\Omc||^{||\Omc||}$.
% A \emph{decorated ensemble} is a tuple $E = (L, \epsilon, \Imc, Q^\mn{acc})$ where
% \begin{itemize}
%   \item $L \in \Abf_\Omc$ is the abstraction level of the decorated ensemble,
%   \item $\epsilon$ is a non-empty tuple over $\Delta_\mn{mos}$ of length at
%     most $|\Omc|$ that constitutes the actual ensemble,
%   \item $\Imc$ is an interpretation such that $\Delta^\Imc$ contains
%     exactly
%     the elements of $\epsilon$,
%   \item $Q^\text{acc} \subseteq Q_L$ is a set of \emph{promised CQs}.
% \end{itemize}
% %
%We usually use $\Emc$ to denote a set of decorated ensembles.
A \emph{mosaic} is a tuple
$$M = ((\Imc_{L})_{L \in \Abf_\Omc},\rho, %f) %Q^\mn{pro},
f_\mn{in}, f_\mn{out})
$$
where
\begin{itemize}

% \item $L_0 \in \Abf_\Omc$ is the \emph{abstraction level} of the mosaic,

\item $(\Imc_{L})_{L \in \Abf_\Omc}$ is a collection of
    interpretations and $\rho$ is a partial function such that
    $(\Abf_\Omc,\prec, (\Imc_{L})_{L \in \Abf_\Omc},\rho)$ is an
    A-interpretation except that some interpretation domains
    $\Delta^{\Imc_L}$ may be empty; the length of tuples in the range
    of $\rho$ may be at most $||\Omc||$;
  
% \item $(\Imc_{L})_{L \in \Abf_\Omc}$ is a collection of (possibly empty) interpretations
%   $\Imc_L$, one for every $L \in \Abf_\Omc$, with pairwise disjoint
%   domains and $\Delta^{\Imc_L} \subseteq \Delta$; % for all
% %  $L \in \Abf_\Omc$, 
%   we set $\Delta^\Imc := \bigcup_{L \in \Abf_\Omc} \Delta^{\Imc_L}$
%   and $L(d)$ to denote the unique $L$ with $d \in \Delta^{\Imc_L}$,
  
% \item $\rho$ is a partial function that associates pairs
%   $(d, L) \in \Delta^\Imc \times \Abf_\Omc$, where $L \prec L(d)$,
%   with tuples over $\Delta^{\Imc_L}$ of length at most $|\Omc|$; 
%   we require that no two distinct tuples in the range of $\rho$ share an
%   element;
  
% \item $t$ is a function that associates every $d \in \Delta^\Imc$ with
%   an $L(d)$-type $t(d)$;
  
  \item $f_\mn{in}$ and $f_\mn{out}$ are functions that associate
    every $L \in \Abf_\Omc$ with a set of pairs $(q,h)$ where $q$ is a
    CQ from an abstraction %or refinement
    statement in \Omc or a subquery thereof, and $h$ is a partial
    function from $\mn{var}(q)$ to $\Delta^{\Imc_L}$; we call these
    pairs the \emph{forbidden incoming queries} in the case of
    $f_\mn{in}$ and the \emph{forbidden outgoing queries} in the case
    of $f_\mn{out}$.
%  {\color{red} $q$ ist nur aus abstraction statements (+ subqueries), oder?}

\end{itemize}
%
% {\color{blue}For easier reference, we use $\eta_L^\epsilon(d)$ to denote $\epsilon_E$ if $\eta_L(d)=E$.
%
We may write $\Imc_L^M$ to denote $\Imc_L$, for any $L \in \Abf_\Omc$,
and likewise for $\rho^M$, $f^M_\mn{in}$, and $f^M_\mn{out}$.

Every mosaic has to satisfy several additional conditions. Before we
can state them, we introduce some notation.
% We
% can see a CQ $q$ as an undirected graph $G_q = (V_q, E_q)$ with
% $V_q = \mn{var}(q)$ and
% $E_q = \{\{v, v'\} \mid R(v, v') \in q \text{ for any } r \in
% \Rbf\}$.
For $V \subseteq \mn{var}(q)$, we use $q|_V$ to denote the restriction
of $q$ to the variables in $V$ and write $\overline{V}$ as shorthand
for \mbox{$\overline V = \mn{var}(q) \setminus V$}.  A \emph{maximally
  connected component (MCC)} of $q$ is a CQ $q|_V$ that is connected
and such that $V$ is maximal with this property.  A % (non-empty)
CQ $p = E \uplus p_0$ is a \emph{component of $q$ w.r.t.\
  $V \subseteq \mn{var}(q)$} if $p_0 $ is an MCC of $q|_{\overline V}$
and $E$ is the set of all atoms from $q$ that contain one variable
from $V$ and one variable from~$\overline V$.
% We say that $p$ is \emph{one-suspended} if it
% takes the form $p = \{R(x,y)\} \cup p'$ with
% $\mn{var}(p) \cap V = \{x\}$ and $x \not \in \mn{var}(p')$.  We say
% that $V$ is \emph{decomposition admissible} if $V \neq \emptyset$ and
% every component of $q$ w.r.t.\ $V$ is one-suspended. Note that
% $V=\mn{var}(q)$ is decomposition admissible.
%{\color{red}I think an example would be good! I made some room for it}
%
\begin{example}
  The following CQ has two components w.r.t.\ $V=\{x,y\}$, which are
  displayed in dashed and dotted lines:
  \begin{center} 
  \includegraphics[width=0.55\columnwidth]{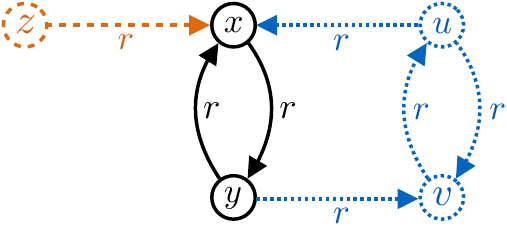}
  \end{center}
  For example, the dotted component is defined by $p = E \uplus p_0$ with $E = r(u,x) \land r(y,v)$ and $p_0 = r(u,v) \land r(v,u)$.
  %
  % Consider $q=r(x,y) \wedge r(y,x) \wedge r(u,x) \wedge r(y,v) \wedge
  % r(u,v) \wedge r(v,u)$ and $V=\{x,y\}$ Then there is one component
  % of $q$ w.r.t.\ $V$, which is $q$ without the first two atoms.
\end{example}
With these notions at hand, let us explain the intuition of the
$f_\mn{in}$ and $f_\mn{out}$ components of mosaics.  Our decomposition
of A-interpretations into sets of mosaics is such that every ensemble
falls within a single mosaic. This means that we must avoid
%applications of concept abstractions across mosaics, i.e.,
homomorphisms from the CQs in concept abstractions that hit multiple
mosaics: such homomorphisms would hit elements from multiple ensembles
while also turning the set of all elements that are hit into an
ensemble; they thus generate overlapping ensembles which is forbidden.
Almost the same holds for role abstractions where however the CQ takes
the form $q(\bar x,\bar y)$ with each of $\bar x$ and $\bar y$
describing an ensemble, and we must only avoid homomorphisms that hit
multiple mosaics from the variables in $\bar x$, or from the
variables in $\bar y$.

Query avoidance is implemented by the $f_\mn{in}$ and $f_\mn{out}$
components. In brief and for CQs $q(\bar x)$ from concept
abstractions, we consider any non-empty subset
$V \subsetneq \mn{var}(q)$ of variables and homomorphism $h$ from
$q|_V$ to the current mosaic.  We then have to avoid any homomorphism
$g$ from $q \setminus q|_V$ that is compatible with $h$ and hits at
least one mosaic other than the current one. The choice of $V$
decomposes $q$ into remaining components, which are exactly the
components of $q$ w.r.t.\ $V$ defined above. We choose one such
component $p$ and put $(p,h')$ into $f_\mn{out}$, $h'$ the restriction
of $h$ to the variables in $p$, to `send' the information to other
mosaics that this query is forbidden. The $f_\mn{in}$ component, in
contrast, contains forbidden queries that we `receive' from
other mosaics.

% Let $M$ be a mosaic, $q$ a full CQ, and
% \mbox{$L \in \Abf_\Omc$}.  A \emph{type homomorphism from $q$ to
%   $\Imc_L^M$} is mapping $h$ from $\mn{var}(q)$ to
% $\Delta^{\Imc_L^M}$ such that $C(x) \in q$ implies $C \in t^M(h(x))$, and
% $R(x,y) \in q$ implies $(h(x),h(y)) \in R^{\Imc_{L}^M}$.

% {\color{blue} {\color{red}used only twice} We use $\bar x = x_1 \cdots x_n$ with $x_i \in \Vbf$ for all $i \leq n$ to denote a tuple of variables 
% and $\mn{set}(\bar x) = \{x \mid x \text{ is part of } \bar x\}$ 
% to denote the set of variables in $\bar x$.}

We now formulate the additional conditions on mosaics.
We require that
$M = ((\Imc_{L})_{L \in \Abf_\Omc},\rho, %Q^\mn{pro},
f_\mn{in},
f_\mn{out})$ satisfies the following conditions, for all
$L \in \Abf_\Omc$: % and $d,e \in \Delta^{\Imc_{L}}$:
\begin{enumerate}

\item the $A$-interpretation
  $(\Abf_\Omc,\prec, (\Imc_{L})_{L \in \Abf_\Omc},\rho)$ satisfies all
  inclusions, refinements, and abstractions in \Omc with the possible
  exception of CIs the form $A \sqsubseteq_L \exists r
  . B$;

\item for all concept abstractions
  $ L'{:}A \abs L{:}q(\bar x)$ in $\Omc$, 
  all non-empty $V \subsetneq \mn{var}(q)$, and all homomorphisms $h$
  from $q|_V$ to $\Imc_L$:
  there is a component $p$ of $q$
  w.r.t.\ $V$ such that
  $(p,h|_{V \cap \mn{var}(p)}) \in f_{\mn{out}}(L)$
\item for all role abstractions
  $L'{:}R \abs L{:}q(\bar x, \bar y)$ in $\Omc$, 
  all non-empty $V \subsetneq \mn{var}(q)$ with $V \neq \bar x$ and $V
  \neq \bar y$,\footnote{Here we view $\bar x$ and $\bar y$ as sets.}
  and all homomorphisms $h$ from $q|_V$ to $\Imc_L$:
  there is a component $p$ of $q$
  w.r.t.\ $V$ such that
  $(p,h|_{V \cap \mn{var}(p)}) \in f_{\mn{out}}(L)$;
\item for all $(q, h) \in
  f_\mn{in}(L)$, all $V \subseteq
  \mn{var}(q)$, and all homomorphisms $g$ from
  $q|_V$ to $\Imc_L$ that extend
  $h$, there is a component $p$ of $q$ w.r.t.\ $V$ such that
  $(p,g|_{V \cap \mn{var}(p)}) \in f_{\mn{out}}(L)$.

\end{enumerate} 
We next need a mechanism to interconnect mosaics. This is driven by
concept names $A$ and elements $d \in A^{\Imc_L}$ such that
$A \sqsubseteq_L \exists R . B \in \Omc$ and $d$ lacks a witness
inside the mosaic. In principle, we would simply like to find a
mosaic $M'$ that has some element $e$ on level $L$ such that
$e \in B^{\Imc^{M'}}$ and an $R$-edge can be put between $d$ and
$e$. The situation is complicated, however, by the presence of role
refinements and role abstractions, which might enforce
additional edges that link the two mosaics. We must also be careful to
synchronize the $f_\mn{in}, f_\mn{out}$ components of the two involved
mosaics across the connecting edges.

Consider mosaics
$M = ((\Imc_{L})_{L \in \Abf_\Omc},\rho, f_\mn{in}, f_\mn{out})$
and
$M' = ((\Imc'_{L})_{L \in \Abf_\Omc},\rho', f'_{\mn{in}},f'_{\mn{out}})$. 
An \emph{$M,M'$-edge} is an expression
$R(d,d')$ such that $R$ is a role, $d \in \Delta^{\Imc_L}$, and
$d' \in \Delta^{\Imc'_L}$ for some $L \in \Abf_\Omc$. A set $E$ of
$M,M'$-edges is an \emph{edge candidate} if the following conditions
are satisfied:
\begin{enumerate}

\item $R(d,e) \in E$ and $L(d) = L$ implies $S(d,e) \in E$, for all
  $R\sqsubseteq_L S \in \Omc$;

  % \item $\{D \mid \forall R.D \in t(d)\} \subseteq t'(d')$ for all
  %   $R(d,d') \in E$;

\item if $\exists R . A \sqsubseteq_L B \in \Omc$, $R(d,d') \in E$, 
  and  $d' \in A^{\Imc'_L}$, then $d \in B^{\Imc_L}$;
    
  % \item $\{D \mid \forall R^-.D \in t'(d')\} \subseteq t(d)$ for all $R(d,d') \in E$;
  \item for all $L \in \Abf_\Omc$, all $(q, h) \in f_\mn{out}(L)$, where 
    $q = E_q \uplus q|_{\overline V}$ for $V = \mn{dom}(h)$,
    and all functions $g$ from $\overline V \cap \mn{var}(E_q)$ to $\Delta^{\Imc'_L}$
    such that $R(h(x), g(y)) \in E$ for all $R(x,y) \in E_q$,
    we have $(q|_{\overline V}, g) \in f'_{\mn{in}}(L)$;
  % \item for all $L \in \Abf_\Omc$, all $(q, h) \in f'_{\mn{out}}(L)$, where 
  % $q = E_q \uplus q|_{\overline V}$ for $V = \mn{dom}(h)$,
  % and all functions $g$ from $\overline V \cap \mn{var}(E_q)$ to $\Delta^{\Imc_L}$
  %  such that $R(h(x), g(y)) \in E$ for all $R(x,y) \in E_q$,
  % we have $(q|_{\overline V}, g) \in f_{\mn{in}}(L)$;
  %
  \item for all $R(d,d') \in E$ and all $L{:}q(\bar x, \bar y) \rfn L'{:}q_R(x,y) \in \Omc$
 such that $q = q|_{\bar x} \uplus E_q \uplus q|_{\bar y}$,
   $q_R=C_x(x) \wedge R(x,y) \wedge C_y(y)$, 
   $d \in C_x^{\Imc_{L'}}$, and $d' \in C_y^{\Imc_{L'}}$:
  \begin{enumerate}
    \item $\rho_{L}(d)$ and $\rho'_{L}(d')$ are defined;
    \item $h:\bar x \mapsto \rho_{L}(d)$ is a homomorphism from
      $q|_{\bar x}$ to $\Imc_L$;
    \item $h':\bar y \mapsto \rho'_{L}(d')$ is a homomorphism
      from $q|_{\bar y}$ to $\Imc'_L$;
    \item $\{R(h(x), h'(y)) \mid R(x,y) \in E_q\} \subseteq E$;
  \end{enumerate}
  \item for all role abstractions $L'{:}R \abs L{:}q(\bar x, \bar y) \in \Omc$,
  where $q = q|_{\bar x} \uplus E_q \uplus q|_{\bar y}$,
  all homomorphisms $h$ from $q|_{\bar x}$ to $\Imc_L$,
  and all homomorphisms $g$ from $q|_{\bar y}$ to $\Imc'_L$
  such that $\{S(h(x), g(y)) \mid S(x,y) \in E_q\} \subseteq E$, 
  there are $d \in \Delta^{\Imc_{L'}}$ and
  $d' \in \Delta^{\Imc'_{L'}}$ with $\rho_L(d) = h(\bar x)$, 
  $\rho'_L(d') = g(\bar y)$, and $R(d,d') \in E$;
  %
  % \item for all role abstractions $L'{:}R \abs L{:}q(\bar x, \bar y) \in \Omc$,
  % where $q = q|_{\bar x} \uplus E_q \uplus q|_{\bar y}$,
  % all type homomorphisms $h$ from $q|_{\bar x}$ to $\Imc'_L$,
  % and all type homomorphisms $g$ from $q|_{\bar y}$ to $\Imc_L$
  % such that $\{S(h(x), g(y)) \mid S(x,y) \in E_q\} \subseteq E$, 
  % there are $d \in \Delta^{\Imc_{L'}}$ and
  % $d' \in \Delta^{\Imc'_{L'}}$ with $\rho_L(d) = g(\bar y)$, 
  % $\rho'_L(d') = h(\bar x)$, and $R(d',d) \in E$.

\item Converses of Conditions~2-5 above that go from $M'$ to $M$
  instead of from $M$ to $M'$; details are in the appendix.

\end{enumerate}
Let \Mmc be a set of mosaics. A mosaic $M$ is \emph{good in \Mmc} if
for all $A \sqsubseteq_L \exists R . B \in \Omc$ and
$d \in (A \sqcap \neg \exists R . B)^{\Imc^M_L}$:
\begin{enumerate}
  \item[$(*)$] there is a mosaic $M' \in \Mmc$,  %with 
  a $d' \in B^{\Imc^{M'}_{L}}$, % such that $C \in t^{M'}(d')$, 
  and an edge candidate $E$ such that $R(d,d') \in E$.
\end{enumerate}

The actual algorithm is now identical to that from the previous
section.  We first compute the set $\Mmc_0$ of all mosaics and
then repeatedly and exhaustively eliminate mosaics that are not
good. Let $\Mmc^*$ denote the set of mosaics at which this process
stabilizes.
\begin{restatable}{lemma}{lemcorrtwo}
  \label{lem:corr2}
  $C_0$ is $L_0$-satisfiable w.r.t.\ \Omc iff $\Mmc^*$ contains (i)~a mosaic
  $M$ with $C_0^{\Imc^M_{L_0}} \neq \emptyset$ and
  (ii)~a mosaic $M$ with \mbox{$\Delta^{\Imc^M_L} \neq \emptyset$}, for every $L$ in $\Abf_\Omc$.
\end{restatable}
The algorithm thus returns `satisfiable' if Conditions~(i) and~(ii)
from Lemma~\ref{lem:corr2} are satisfied and `unsatisfiable'
otherwise. It can be verified that the algorithm runs in double
exponential time.

\section{Lower Bounds}

%In Section~\ref{sec:alci_cr_exp_ub},
We have seen that the fragment $\ALCHI^\mn{abs}[\textnormal{cr}]$ of
$\ALCHI^{\mn{abs}}$ which focusses on concept refinement is only
\ExpTime-complete. Here we show that all other fragments that contain
only a single form of abstraction/refinement are
2\ExpTime-hard, and consequently 2\ExpTime-complete. This of course
also provides a matching lower bound for
Theorem~\ref{thm:twoexpupper}---actually three rather different lower
bounds, each one exploiting a different effect. All of our lower
bounds apply already when \ALCHI is replaced with \ALC as the
underlying DL.

\subsection{\texorpdfstring{Role Abstraction: $\ALC^\mn{abs}[\textnormal{ra}]$}{Role Abstraction}}
\label{sec:alc_ra_2exp_lb}

The 2\ExpTime-hardness of satisfiability in $\ALCHI^\mn{abs}$ is not
entirely surprising given that we have built conjunctive queries into
the logic and CQ evaluation on \ALCI knowledge bases is known to be
2\ExpTime-hard~\cite{DBLP:conf/cade/Lutz08}. In fact, this is already
the case for the following \emph{simple} version of the latter
problem: given an \ALCI ontology \Omc, a concept name $A_0$, and a
Boolean CQ $q$, decide whether $\Imc \models q$ for all models \Imc of
\Omc with $A_0^\Imc \neq \emptyset$. We write $\Omc, A_0 \models q$ if
this is the case.

It is easy to reduce the (complement of the) simple CQ evaluation
problem to satisfiability in $\ALCI^\mn{abs}[\textnormal{ca}]$.
% We call this problem \emph{restricted CQ entailment}.
% Restricted CQ entailment on $\ALCHI$ ontologies is 2\ExpTime-hard (even complete) \cite{lutz2007inverse}.
%
% We reduce from the complement of restricted CQ entailment on $\ALCHI$ ontologies to%
% concept satisfiability w.r.t.\ $\ALCHI^\mn{abs}[\textnormal{ca}]$.
% Specifically, we construct an ontology $\Omc'$ with abstraction levels $L \prec L'$ 
% such that $L{:}\, A$ is satisfiable w.r.t.\ $\Omc'$ iff $\Omc, A \not \models q$.
%
Fix two abstraction levels $L \prec L'$, let $\widehat q$ be the CQ obtained
from $q$ by dequantifying all variables, thus making all variables
answer variables, and let $\Omc'$ be the set of all concept inclusions
$C \sqsubseteq_L D$ with $C \sqsubseteq D \in \Omc$ and the concept
abstraction
$$L'{:}\bot \abs L{:}\widehat q.$$
It is straightforward to show that $A_0$ is $L$-satisfiable
w.r.t.\ $\Omc'$ iff $\Omc, A_0 \not \models q$.

Our aim, however, is to prove 2\ExpTime-hardness without using inverse
roles. The above reduction does not help for this purpose since CQ
evaluation on \ALC knowledge bases is in
\ExpTime~\cite{DBLP:conf/cade/Lutz08}. Also, we wish to use role
abstractions in place of the concept abstraction. 

As above, we reduce from the complement of simple CQ evaluation on
$\ALCI$ ontologies. Given an input $\Omc, A_0, q$, we construct an
$\ALC^\mn{abs}[\textnormal{ra}]$ ontology $\Omc'$ such that $A_0$ is
$L_1$-satisfiable w.r.t.\ $\Omc'$ iff $\Omc, A_0 \not \models q$. The
idea is to use multiple abstraction levels and role
abstractions to simulate inverse roles.  Essentially, we replace every
inverse role $r^-$ in \Omc with a fresh role name $\widehat r$,
obtaining an \ALC-ontology $\Omc_\ALC$, use three abstraction levels
$L_1 \prec L_2 \prec L_3$, and craft $\Omc'$ so that in its models,
the interpretation on level $L_1$ is a model of $\Omc_\ALC$, the
interpretation on level $L_2$ is a copy of the one on level $L_1$, but
additionally has an $r$-edge for every $\widehat r^-$-edge, and $L_3$
is an additional level used to check that there is no homomorphism
from $q$, as in the initial reduction above. Details are in the
appendix.
\begin{restatable}{theorem}{thmfirsthard}
  \label{thm:firsthard}
  Satisfiability in $\ALC^\mn{abs}[\textnormal{ra}]$ is 2\ExpTime-hard.
\end{restatable}

\subsection{\texorpdfstring{Concept Abstraction: $\ALC^\mn{abs}[\textnormal{ca}]$}{Concept Abstraction}}
\label{sec:alc_ca_2exp_lb}

When only concept abstraction is available, the simulation of inverse
roles by abstraction statements presented in the previous section does
not work. We are nevertheless able to craft a
reduction from the simple CQ evaluation problem on \ALCI ontologies,
though in a slightly different version. What is actually proved
in~\cite{DBLP:conf/cade/Lutz08} is that simple CQ evaluation is
2\ExpTime-hard already in the DL $\ALC^\mn{sym}$, which is \ALC with a
single role name $s$ that must be interpreted as a reflexive and
symmetric relation.

We simulate $\ALC^\mn{sym}$-concepts $\exists s . C$ by \ALCI-concepts
$\exists r^-.\exists r.C$, and likewise for $\forall s. C$.  The
reduction then exploits the following effect. Consider an
$\ALC^\mn{abs}[\textnormal{ca}]$-ontology \Omc that contains the CI
$\top \sqsubseteq_L \forall r . \exists \widehat r. \top$ and the concept
abstractions
$$
L'{:}\, \top \abs L{:}\, r(x,y) \ \text{ and } \
L'{:}\, \top \abs L{:}\, \widehat r(y,x).
$$
Then in any model $\Imc$ of \Omc, $(d,e) \in r^{\Imc_L}$ implies
$(e,d) \in \widehat r^{\Imc_L}$ and thus we can use $\widehat r$ as
the inverse of $r$. This is because the first concept abstraction
forces $(d,e)$ to be an ensemble on level $L$, $e$ must have a
$\widehat r$-successor $f$, and the second concept abstraction forces
$(f,e)$ to be an ensemble. But since ensembles cannot overlap, this is
only possible when $d=f$.

We have to be careful, though, to not produce
undesired overlapping ensembles which would result in
unsatisfiability.  As stated, in fact, the ontology \Omc does not
admit models in which $\Imc_L$ contains an $r$-path of length
two. This is why we resort to $\ALC^\mn{sym}$. Very briefly, the
$r^-$-part of the role composition $r^-;r$ falls inside an ensemble
and is implemented based on the trick outlined above while the
$r$-part connects two different ensembles to avoid overlapping. The
details are somewhat subtle and presented in the appendix.
\begin{restatable}{theorem}{thmsecondhard}
  \label{thm:secondhard}
  Satisfiability in $\ALC^\mn{abs}[\textnormal{ca}]$ is 2\ExpTime-hard.
\end{restatable}

\subsection{\texorpdfstring{Role Refinement: $\ALC^\mn{abs}[\textnormal{rr}]$}{Role Refinement}}
\label{sec:rr_2exp_lb}

While concept and role abstractions enable reductions from CQ
evaluation, this does not seem to be the case for concept and role
refinements. Indeed, we have seen in Section~\ref{sec:alci_cr_exp_ub}
that concept refinements do not induce 2\ExpTime-hardness. Somewhat
surprisingly, role refinements behave differently and are a source of
2\ExpTime-hardness, though for rather different reasons than
abstraction statements.

It is well-known that there is an exponentially space-bounded
alternating Turing machine (ATM) that decides a 2\ExpTime-complete
problem and on any input $w$ makes at most $2^{|w|}$ steps
\cite{chandra1981alternation}. We define ATMs in detail in the
appendix and only note here that our ATMs have a one-side infinite
tape and a dedicated accepting state $q_a$ and rejecting state $q_r$,
no successor configuration if its state is $q_a$ or $q_r$,
and exactly two successor configurations otherwise.

Let $M = (Q, \Sigma, \Gamma, q_0, \Delta)$ be a concrete such ATM with
$Q=Q_\exists \uplus Q_\forall \uplus \{q_a,q_r\}$. We may assume
w.l.o.g\ that $M$ never attempts to move left when the head is
positioned on the left-most tape cell.  Let
$w = \sigma_1 \cdots \sigma_n \in \Sigma^*$ be an input for $M$.  We
want to construct an $\ALC^\mn{abs}[\textnormal{rr}]$-ontology~$\Omc$ and
choose a concept name $S$ and abstraction level $L_1$ such that $S$ is
$L_1$-satisfiable w.r.t.\ $\Omc$ iff $w \in L(M)$.
%
% In our construction every model of $S$ and $\Omc$ represents a successful run of $M$ on $w$.
% Any defect (such as changing tape content not under the head or having multiple heads) will immediately become unsatisfiable.
% The signature of $\Omc$ contains the following symbols:
Apart from~$S$, which indicates the starting configuration, we use the
following concept names:
\begin{itemize}
%  \item the concept name $S$ for the initial configuration;
  \item $A_\sigma$, for each $\sigma \in \Gamma$, to
    represent tape content;
  \item $A_q$, for each $q \in Q$, to represent state
    and head position;
  \item $B_{q, \sigma, M}$ for $q \in Q$, $\sigma \in
    \Gamma$, $M \in \{L,R\}$, serving to choose a transition;
  \item $H_\shortleftarrow, H_\shortrightarrow$
    indicating whether a tape cell is to the right or left of the head.
\end{itemize}
plus some auxiliary concept names whose purpose shall be obvious. We
use the role name $t$ for next tape cell and $c_1, c_2$ for successor
configurations.

The ontology $\Omc$ uses the abstraction levels
$\Abf = \{L_1, \dots, L_n\}$ with $L_{i+1} \prec L_{i}$ for
$1 \leq i < n$. While we are interested in $L_1$-satisfiability of
$S$, the computation of $M$ is simulated on level $L_n$. We start
with generating an infinite computation tree on level $L_1$:
$$
  S \sqsubseteq_{L_1} \exists c_1.N \sqcap \exists c_2.N \quad
  N \sqsubseteq_{L_1} \exists c_1.N \sqcap \exists c_2.N.
$$
In the generated tree, each configuration is represented by a single
object. On levels $L_2,\dots,L_n$, we generate similar trees where,
however, configurations are represented by $t$-paths. The length of
these paths doubles with every level and each node on a path is
connected via $c_1$ to the corresponding node in the path that
represents the first successor configuration, and likewise for $c_2$
and the second successor configuration. This is illustrated in
Figure~\ref{fig:rr_lb_structure} where for simplicity we only show a
first successor configuration and three abstraction levels. We use the
following role refinements:
%
% The simulation of computations of $M$ happens only on level $L_n$.
% The more coarse grained levels are there to ensure that on $L_1$ we
% have exponentially large, properly connected configurations
% {\color{blue} figure?}.  We start with configurations of size two on
% layer $L_{p(|w|)}$ and double the size whenever we go to a more fine
% grained level until finally achieving size $2^{p(|w|)}$ on level
% $L_1$.  Formally we use the following role refinements for this, see
% Figure \ref{fig:tc_and_cc} for a visualization.
%
\begin{figure}
  \begin{center}
  \includegraphics[width=0.9\columnwidth]{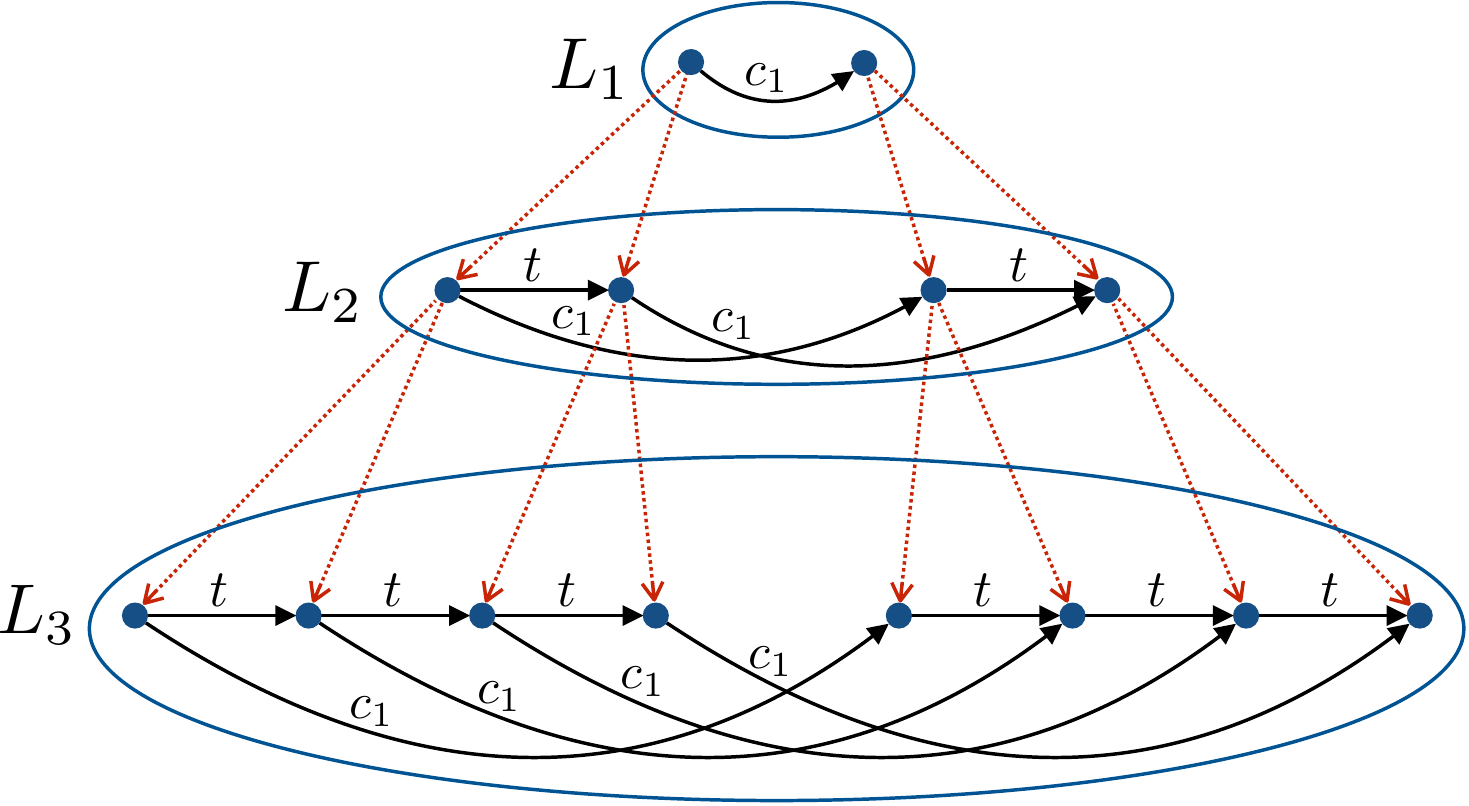}
  \end{center}
  \caption{Dotted lines indicate refinement.}
  \label{fig:rr_lb_structure}
  \vspace*{-5mm}
\end{figure}
%
%Tape connectedness (TC) makes sure that on the more fine grained level we have doubled the size
%of the tape, and all tape cells are connected by a path using the role name $t$:
$$
\begin{array}{r@{\;}l}
L_{i+1}{:}\,q(\bar x, \bar y) &\rfn L_i{:}\,t(x,y) \\[1mm]
L_{i+1}{:}\,q_j(\bar x, \bar y) &\rfn L_i{:}\,c_i(x,y)
\end{array}
$$
for $0 \leq i < n$ and $j \in \{1,2\}$, and where $\bar x = x_1x_2$,
\mbox{$\bar y = y_1y_2$} and
$$
\begin{array}{r@{\,}c@{\,}l}
  q(\bar x, \bar y) &=& t(x_1, x_2) \land t(x_2, y_1) \land t(y_1,
                          y_2) \\[1mm]
  q_j(\bar x, \bar y) &=& t(x_1, x_2) \land t(y_1, y_2) \land c_j(x_1, y_1) \land c_j(x_2, y_2).
\end{array}
$$
To make more precise what we want to achieve,  let the
\emph{$m$-computation tree}, for $m > 0$, be the interpretation
$\Imc_m$ with
$$
\begin{array}{@{}r@{\;}c@{\;}l}
  \Delta^{\Imc_m} &=& \{ c_0,c_1\}^* \cdot \{ 1,\dots,m\} \\[1mm]
  t^{\Imc_m} &=& \{ (wi,wj) \mid w \in \{c_0,c_1\}^*, 1 \leq i < m, j=i+1 
                 \} \\[1mm]
  c_\ell^{\Imc_m} &=& \{ (wj,wc_ij) \mid w \in \{c_0,c_1\}^*, 1 \leq j
                      \leq m, i \in \{0,1\} \}
\end{array}
$$
for $\ell \in \{ 1,2\}$. It can be shown that for any model \Imc
of the $\ALC^\mn{abs}[\textnormal{rr}]$-ontology $\Omc$ constructed so far
and for all $i \in \{1,\dots,n\}$, we must find a (homomorphic image
of a) $2^i$-computation tree in the interpretation $\Imc_{L_i}$. This crucially relies on the fact that ensembles cannot
overlap. In Figure~\ref{fig:rr_lb_structure}, for example, the role refinements
for $t$ and for $c_1$ both apply on level~$L_2$, and
for attaining the structure displayed on level~$L_3$ it is crucial
that in these applications each object on level~$L_2$ refines into
the same ensemble on level~$L_3$.

On level~$L_n$, we thus find a $2^n$-computation tree which we use to
represent the computation of $M$ on input $w$. To start, the concept
name $S$ is copied down from the root of the $1$-computation tree on
level~$L_1$ to that of the $2^n$-computation tree on level $L_n$. To
achieve this, we add a copy of the above role refinements for $c_1$,
but now using the CQ
$$
q_1(\bar x,\bar y) =S(x_1) \land c_1(x_1, y_1) \land c_1(x_2, y_2).
$$
We next describe the initial configuration:
\begin{align*}
  S &\sqsubseteq_{L_n} A_{q_0} \sqcap A_{\sigma_1} \sqcap \forall t.A_{\sigma_2} \sqcap 
  \cdots \sqcap\forall t^{n-1}.(A_{\sigma_n} \sqcap B_\shortrightarrow) \\
  B_\shortrightarrow &\sqsubseteq_{L_n} \forall t.(A_\square \sqcap B_\shortrightarrow)
\end{align*}
For existential states, we consider one of the two possible successor
configurations:
$$
  A_q \sqcap A_\sigma \sqsubseteq_{L_n} (\forall c_1.B_{q', \sigma', M'}) \sqcup
  (\forall c_2.B_{\bar q, \bar \sigma, \bar M})
$$
for all $q \in Q_\exists$ and $\sigma \in \Gamma$ such that 
$\Delta(q, \sigma) = \{(q', \sigma', M'),$ $(\bar q, \bar \sigma, \bar M)\}$.
For universal states, we use both successors:
$$
  A_q \sqcap A_\sigma \sqsubseteq_{L_n} (\forall c_1.B_{q', \sigma', M'}) \sqcap
  (\forall c_2.B_{\bar q, \bar \sigma, \bar M})
$$
for all $q \in Q_\forall$ and $\sigma \in \Gamma$ such that 
$\Delta(q, \sigma) = \{(q', \sigma', M'),$ \\$(\bar q, \bar \sigma, \bar M)\}$.
We next implement the transitions:
$$
\begin{array}{c}
  B_{q, \sigma, M} \sqsubseteq_{L_n} A_\sigma \quad
  \exists t.B_{q, \sigma, L} \sqsubseteq_{L_n} A_q \\[1mm]
  B_{q, \sigma, R} \sqsubseteq_{L_n} \forall t.A_q
\end{array}
$$
for all $q \in Q$, $\sigma \in \Gamma$, and $M \in \{L, R\}$.
We  mark cells that are not under the head:
$$
\begin{array}{r@{\;}c@{\;}lcr@{\;}c@{\;}l}
  A_q &\sqsubseteq_{L_n}& \forall t.H_{\shortleftarrow} &&
                                                              \exists t.A_q &\sqsubseteq_{L_n}& H_{\shortrightarrow} \\[1mm]
    H_{\shortleftarrow} &\sqsubseteq_{L_n}& \forall
                                               t.H_{\shortleftarrow} &&
\exists t.H_{\shortrightarrow} &\sqsubseteq_{L_n}& H_{\shortrightarrow}
\end{array}
$$
for all $q \in Q$.
% , as well as 
% $$
%   L_n{:}\, H_{\shortleftarrow} \sqsubseteq \forall t.H_{\shortleftarrow}, 
%   \quad L_n{:}\,  \exists t.H_{\shortrightarrow} \sqsubseteq H_{\shortrightarrow},
%   \quad
%   %L_n{:}\,  H_\shortleftarrow \sqcup H_\shortrightarrow \sqsubseteq N
% $$
Such cells do not change:
$$
  (H_\shortleftarrow \sqcup H_\shortrightarrow) \sqcap A_\sigma \sqsubseteq_{L_n} \forall c_i.A_\sigma
$$
for all $\sigma \in \Gamma$ and $i \in \{1,2\}$.
State, content of tape, and head position must be unique:
$$
\begin{array}{c}
  A_q \sqcap A_{q'} \sqsubseteq_{L_n} \bot
  \quad  A_\sigma \sqcap A_{\sigma'} \sqsubseteq_{L_n} \bot \\[1mm]
(H_\shortleftarrow \sqcup H_\shortrightarrow) \sqcap A_q \sqsubseteq_{L_n} \bot
\end{array}
$$
for all $q, q' \in Q$ and $\sigma,\sigma' \in \Gamma$ with $q \neq q'$
and $\sigma \neq \sigma'$.
Finally, all followed computation paths must be accepting: 
$$
  A_{q_r} \sqsubseteq_{L_n} \bot.
$$
This finishes the construction of $\Omc$. % It can be verified that \Omc
% is as desired.
%
\begin{lemma}
  \label{lem:2exp_lb_alc_ref}
  $S$ is $L_1$-satisfiable w.r.t.\ $\Omc$ iff $w \in L(M)$.
\end{lemma}
%
% Combining Lemma~\ref{lem:chandra} and
% Lemma~\ref{lem:2exp_lb_alc_ref} give us the desired 2\ExpTime lower
% bound.
We have thus obtained the announced result.
\begin{theorem}
  \label{thm:thirdhard}
  Satisfiability in $\ALC^\mn{abs}[\textnormal{rr}]$ is 2\ExpTime-hard.
\end{theorem}

\section{Undecidability}

One might be tempted to think that the decidability of
$\ALCHI^{\mn{abs}}$ is clear given that only a finite number
of abstraction levels can be mentioned in an ontology. However,
achieving decidability of DLs with abstraction and refinement requires
some careful design choices. In this section, we consider three
seemingly harmless extensions of $\ALCHI^{\mn{abs}}$ and show that
each of them results in undecidability. This is in fact already the
case for $\EL^{\mn{abs}}$ where the underlying DL \ALCHI is replaced
with \EL.

\subsection{Basic Observations}
\label{sect:undecbasic}
We make some basic observations regarding the DL $\EL^{\mn{abs}}$ and
its fragments.  In classical  \EL, concept satisfiability is
not an interesting problem because every concept is satisfiable
w.r.t.\ every ontology.  This is not the case in $\EL^{\mn{abs}}$
where we can express concept inclusions of the form
$C \sqsubseteq_L \bot$ with $C$ an \EL-concept, possibly at the
expense of introducing additional abstraction levels. More precisely,
let $L'$ be the unique abstraction level with $L \prec L'$ if it
exists and a fresh abstraction level otherwise. Then
$C \sqsubseteq_L \bot$ can be simulated by
%
% by
% adding a fresh abstraction level $L'$ as well as
the following CI and concept abstraction:
% %
% $$
% \begin{array}{l}
%   L'{:}\, r(x,y),r(y,z) \rfn L{:}\, C  \\[1mm]
%   L{:}\,\top \abs L'{:}\, r(x,y)
% \end{array}
% $$
% %
% Note that this again relies on the fact that ensembles cannot overlap,
% an thus an $r$-path of length two in level $L'$ results in
% inconsistency. We can actually achieve the same without concept
% refinement. Let $L'$ be the unique abstraction level with $L \prec L'$
% if it exists and a fresh abstraction level otherwise. Then put
%
$$
\begin{array}{l}
  C \sqsubseteq_L \exists r_C . \exists r_C . \top \\[1mm]
%  L'{:}\, r(x,y),r(y,z) \rfn L{:}\, C  \\[1mm]
  L'{:}\top \abs L{:}\, r_C(x,y)
\end{array}
$$
where $r_C$ is a fresh role name.  Note that this again relies on the
fact that ensembles cannot overlap, and thus an $r_C$-path of length
two in level $L$ results in unsatisfiability. The same can be achieved by
using a role abstraction in place of the concept abstraction.

To prove our undecidability results, it will be convenient to have
available concept inclusions of the form
$C \sqsubseteq_L \forall r . D$ with $C, D$ \EL-concepts. Let
$L'$ be a fresh abstraction level. Then
$C \sqsubseteq_L \forall r . D$ can be simulated by the following
role refinement and concept abstraction:
$$
\begin{array}{l}
%  L{:}C(x) \land r(x,y) \land A(y) \abs L'{:}C(x) \land r(x,y)\\[1mm]
 % L'{:} A \abs L{:} D(x)
   L'{:}\,r(x,y) \land A(y) \rfn L{:}\, C(x) \land r(x,y)\\[1mm]
   L'{:}\, D \abs L{:}\, A(x)
\end{array}
$$
where $A$ is a fresh concept name. It is easy to see that the same can
be achieved with a role abstraction in place of the concept
abstraction.

In the following, we thus use inclusions of the forms
$C \sqsubseteq_L \bot$ and $C \sqsubseteq \forall r . D$ in the
context of $\EL^{\mn{abs}}$ and properly keep track of the required
types of abstraction and refinement statements.

\subsection{Repetition-Free Tuples}
\label{sect:repfree}

In the semantics of $\ALCHI^{\mn{abs}}$ as defined in
Section~\ref{sect:abstractionDLs}, ensembles are tuples in which 
elements may occur multiple times. It would arguably be more natural to
%disallow such multiple occurrences and to
define ensembles to be
repetition-free tuples. We refer to this version of the semantics as
the \emph{repetition-free} semantics.

If only concept and role refinement are admitted, then there is no
difference between satisfiability under the original semantics and
under the repetition-free semantics. In fact, any model of \Omc
under the original semantics can be converted into a model of \Omc
under repetition-free semantics by duplicating elements. This gives
the following.
\begin{restatable}{proposition}{thmreprefok}
  \label{thm:reprefok}
  For every \ALCI-concept $C$, abstraction level $L$, and
  $\ALCHI^{\mn{abs}}[\textnormal{cr},\textnormal{rr}]$-ontology \Omc:
  $C$ is $L$-satisfiable w.r.t.\ \Omc iff $C$ is $L$-satisfiable
  w.r.t.\ \Omc under the repetition-free semantics.
\end{restatable}
%
% In
% $\ALCHI^{\mn{abs}}[\textnormal{cr}]$, there is no difference between
% the original semantics and the repetition-free semantics. {\color{red}This follows
% from the fact that in the proof of Lemma~\ref{lem:corr1} we construct
% models in which all ensembles are repetition-free.  We conjecture that
% the same holds for
% $\ALCHI^{\mn{abs}}[\textnormal{cr},\textnormal{rr}]$.}
The situation changes once we admit abstraction.
\begin{restatable}{theorem}{thmfirstundec}
  \label{thm:firstundec}
  Under the repetition-free semantics, satisfiability is undecidable
  in $\ALC^\mn{abs}[\textnormal{ca}]$,
  $\ALC^\mn{abs}[\textnormal{ra}]$,
  $\EL^\mn{abs}[\textnormal{rr},\textnormal{ca}]$, and
  $\EL^\mn{abs}[\textnormal{rr},\textnormal{ra}]$.
\end{restatable}
In the following, we prove undecidability for satisfiability in
$\ALC^\mn{abs}[\textnormal{ca}]$. The result for
$\ALC^\mn{abs}[\textnormal{ra}]$ is a minor variation and
the results for $\EL^{\mn{abs}}$ are obtained by applying the
observations from Section~\ref{sect:undecbasic}.

We reduce the complement of the halting problem for deterministic
Turing machines (DTMs) on the empty tape.  Assume that we are given a
DTM $M = (Q, \Sigma, \Gamma, q_0, \delta)$. % be a DTM and
% $w = \sigma_1 \cdots \sigma_n$ an input for~$M$.
As in
Section~\ref{sec:rr_2exp_lb}, we assume that $M$ has a one-side
infinite tape and never attempts to move left when the head is on the
left end of the tape. We also assume that there is a dedicated halting
state $q_h \in Q$.

We want to construct an $\ALC^\mn{abs}[\textnormal{ca}]$-ontology
$\Omc$ and choose a concept name $S$ and abstraction level $L$ such
that $S$ is $L$-satisfiable w.r.t.\ $\Omc$ iff $M$ does not halt on
the empty tape. We use essentially the same concept and role names as
in Section~\ref{sec:rr_2exp_lb}, except that only a single role name $c$ is
used for transitions to the (unique) next configuration. %  in a
% computation.
Computations are represented in the form of a grid as
shown on the left-hand side of Figure~\ref{fig:undec_part_ord_grid},
where the concept names $X_i$ must be disregarded as they belong to a
different reduction (and so do the queries). We use two abstraction
levels $L$ and $L'$ with $L \prec L'$.  The computation of $M$ is
represented on level~$L$.
%
% The signature of $\Omc$ contains the following symbols:
% \begin{itemize}
%   \item the concept names $S, S_1, \dots, S_{n-1}$ for the initial configuration;
%   \item the concept names $A_\sigma$, for $\sigma \in \Gamma$;
%   \item the concept names $A_q$, for $q \in Q$;
%   \item the concept names $B_{q, \sigma, M}$ for $q \in Q$, $\sigma \in \Gamma$, $M \in \{L,R\}$, 
%   functioning as a state transition markers of the DTM. 
%   Write symbol $\sigma$, move head in direction $M$, and change state to $q$;
%   \item the concept names $H_\shortleftarrow, H_\shortrightarrow$ that point to the head position and the concept name $N$ that marks cells where the head is not;
%   \item the concept name $B_\shortrightarrow$ indicating that only blanks follow;
%   \item auxiliary concept names $X_1, X_2, X_3, X_4, U_1, U_2, U_3, U_4$;
%   \item the role names $t$ for next tape cell, and $c$ for successor configuration.
% \end{itemize}
%

We first generate an infinite binary tree in which every node has one
$t$-successor and one $c$-successor:
$$
\top \sqsubseteq_L \exists t.\top \sqcap \exists c.\top
$$
To create the desired grid structure, it remains to enforce that grid
cells close, that is, the $t$-successor of a $c$-successor of any node
coincides with the $c$-successor of the $t$-successor of that node.
We add the concept abstraction
$$
\begin{array}{l}
L'{:}\, \bot \abs L{:}\,q(\bar x)\text{ where}\\[1mm]
  q(\bar x)=c(x_1,x_2) \wedge t(x_1,x_3) \wedge c(x_3,x_4) \wedge t(x_2,x'_4). 
\end{array}
$$
The idea is that any non-closing grid cell admits a repetition-free
answer to $q$ on $\Imc_L$, thus resulting in unsatisfiability. If all
grid cells close, there will still be answers, but all of them are
repetitive. The above abstraction alone, however, does not suffice to
implement this idea. It still admits, for instance, a non-closing grid
cell in which the two left elements have been identified. We thus need
to rule out such unintended identifications and add the concept
abstraction $L'{:}\, \bot \abs L{:}\, q$ for the following six CQs $q$:
$$
\begin{array}{l}
  t(x_1,x_2) \wedge c(x_1,x_2) \qquad t(x_1,x_2) \wedge c(x_2,x_1)
  \\[1mm]
  t(x_1,x_2) \wedge c(x_1,x_3) \wedge t(x_3,x_2) \qquad c(x_1,x_1)\\[1mm]
  t(x_1,x_2) \wedge c(x_1,x_3) \wedge c(x_2,x_3) \qquad t(x_1,x_1)
\end{array}
$$
\begin{figure}
  \centering
  \includegraphics[width=\columnwidth]{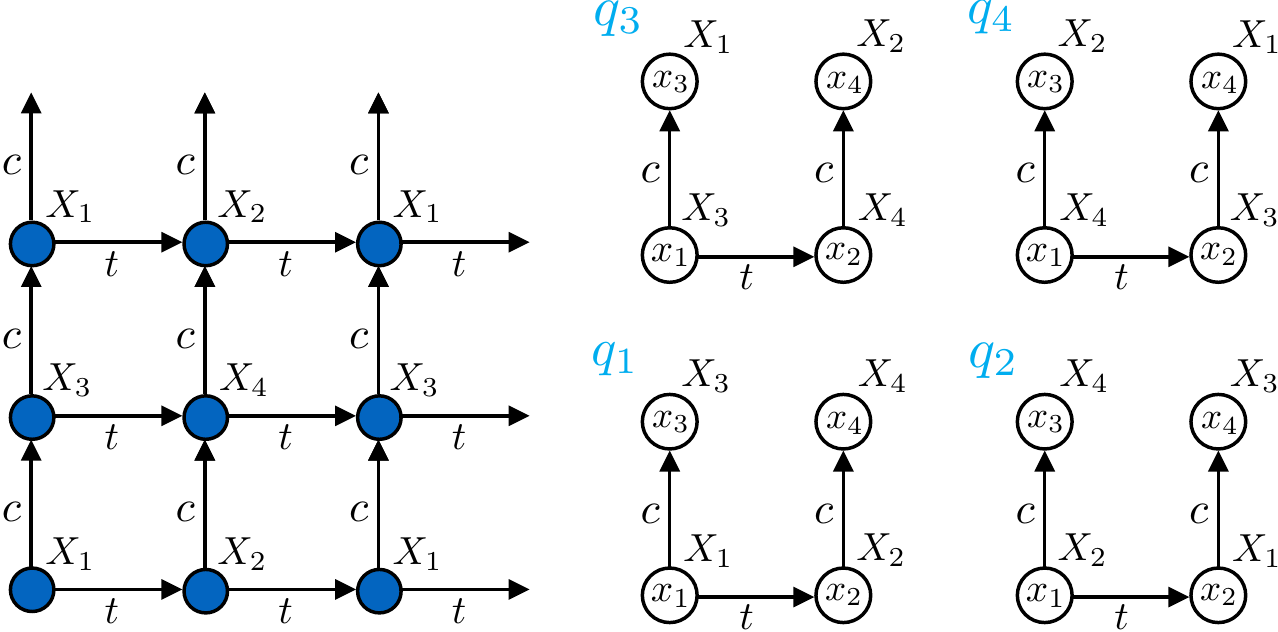}
  \caption{Grid structure and queries for the DAG Semantics.}
  \label{fig:undec_part_ord_grid}
  \vspace*{-5mm}
\end{figure}
% \begin{figure}
%   \centering
%   \includegraphics[width=\columnwidth]{images/partial_orders_grid_and_cqs.pdf}
%   \caption{{\color{red}Remove concept names from left side, add
%       main query relevant to THIS reduction on the right.}}
%   \label{fig:undec_part_ord_gridfirst}
% \end{figure}
% %
The rest of the reduction is now very similar to that given in
Section~\ref{sec:rr_2exp_lb}, details are in the appendix. 

\subsection{DAG Semantics}
\label{sect:nontree}

Our semantics requires abstraction levels to be organized in a
tree. While this is very natural, admitting a DAG structure might also
be useful. This choice, which we refer to as the \emph{DAG semantics},
leads to undecidability.
\begin{restatable}{theorem}{thmnontreeundec}
  \label{thm:nontreeundec}
  Under the DAG semantics, satisfiability is undecidable
  in $\ALC^\mn{abs}[\textnormal{ca,cr}]$ and
  $\EL^\mn{abs}[\textnormal{ca},\textnormal{cr},\textnormal{rr}]$.
\end{restatable}
The result is again proved by a reduction from (the complement of) the
halting problems for DTMs. In fact, the reduction differs from that in
Section~\ref{sect:repfree} only in how the grid is constructed and
thus we focus on that part. We present the reduction for
$\ALC^\mn{abs}[\textnormal{ca},\textnormal{cr}]$.

Assume that we are given a DTM $M = (Q, \Sigma, \Gamma, q_0,
\delta)$. We want to construct an ontology $\Omc$ and choose a concept
name $S$ and abstraction level $L$ such that $S$ is $L$-satisfiable
w.r.t.\ $\Omc$ iff $M$ does not halt on the empty tape.  We use
abstraction levels $L, L_1, L_2, L_3, L_4$ with $L \prec L_i$ for all
$i \in \{1,\dots, 4\}$. The computation of $M$ is simulated on 
level~$L$. We start with generating an infinite $t$-path with outgoing
infinite $c$-paths from every node:
$$
\begin{array}{r@{\;}c@{\;}lcr@{\;}c@{\;}l}
  S &\sqsubseteq_L& A_t && A_t &\sqsubseteq_L& \exists t . A_t
  \\[1mm]
  A_t &\sqsubseteq_L& \exists c . A_c && A_c &\sqsubseteq_L& \exists c . A_c.
\end{array}
$$
In principle, we would like to add the missing $t$-links using the following
concept abstraction and refinement:
$$
\begin{array}{l}
L_1{:}\, U_1 \abs L{:}q \\[1mm]
  L{:}\,q \land t(x_3, x_4) \rfn L_1{:}\,U_1 \ \text{ where} \\[1mm]
    q = c(x_1, x_3) \land t(x_1, x_2) \land c(x_2, x_4).
\end{array}
$$
This would then even show undecidability under the original semantics,
but it does not work because it creates overlapping ensembles and thus
simply results in unsatisfiability. We thus use the four abstraction
levels $L_1,\dots,L_4$ in place of only $L_1$. This results in
different kinds of ensembles on level $L$, one for each level $L_i$,
and an $L_i$-ensemble can overlap with an $L_j$-ensemble if
$i \neq j$. We label the grid with concept names $X_1,\dots,X_4$ as
shown in Figure~\ref{fig:undec_part_ord_grid}, using CIs
$$
\begin{array}{c}
  X_1 \sqsubseteq_L \forall c.X_3 \sqcap \forall t.X_2\qquad
  X_2 \sqsubseteq_L \forall c.X_4 \sqcap \forall t.X_1 \\[1mm]
  X_3 \sqsubseteq_L \forall c.X_1 \qquad
  X_4 \sqsubseteq_L\forall c.X_2 \qquad S \sqsubseteq_L X_1.
\end{array}
$$
% $$
% \begin{array}{r@{\;}c@{\;}lcr@{\;}c@{\;}l}
%   X_1 &\sqsubseteq_L& \forall c.X_3 \sqcap \forall t.X_2&&
%   X_2 &\sqsubseteq_L& \forall c.X_4 \sqcap \forall t.X_1\\
%   X_3 &\sqsubseteq_L& \forall c.X_1 &&
%   X_4 &\sqsubseteq_L& \forall c.X_2.
% \end{array}
% $$
% 
% \begin{figure}
%   \centering
%   \includegraphics[width=\columnwidth]{images/partial_orders_grid_and_cqs.pdf}
%   \caption{Grid structure and queries for the Non-Tree Semantics.}
%   \label{fig:undec_part_ord_grid}
% \end{figure}
%
We define four variations $q_1,\dots,q_4$ of the above CQ $q$, as
shown on the right-hand side of Figure~\ref{fig:undec_part_ord_grid},
and use the following concept abstraction and refinement, for
$i \in \{1, \dots, 4\}$:
$$
\begin{array}{l}
  L_i{:}\, U_i \abs L{:}\,q_i \\[1mm]
  L{:}\, q_i \land t(x_3, x_4) \rfn L_i{:}\,U_i.
\end{array}
$$
It can be verified that this eliminates overlapping ensembles and
indeed generates a grid.

\subsection{Quantified Variables}
\label{sect:qvar}

The final variation that we consider is syntactic rather than
semantic: we admit quantified variables in CQs in abstraction and
refinement statements. %Also this leads to undecidability.
\begin{restatable}{theorem}{thmqvar}
  \label{thm:qvar}
  In the extension with quantified variables, satisfiability is undecidable
  in $\ALCI^\mn{abs}[\textnormal{ca,cr}]$ and
  $\EL^\mn{abs}[\textnormal{ca},\textnormal{cr},\textnormal{rr}]$.
\end{restatable}
We use a DTM reduction that is similar to the previous
reduction, with a slightly different representation of the grid. We again start with
an infinite $t$-path with outgoing infinite $c$-paths from every
node.
% $$
% \begin{array}{r@{\;}c@{\;}lcr@{\;}c@{\;}l}
%   S &\sqsubseteq&\exists t . A_t && A_t &\sqsubseteq& \exists t . A_t
%   \\[1mm]
%   A_t &\sqsubseteq& \exists c . A_c && A_c &\sqsubseteq& \exists c . A_c.
% \end{array}
% $$
%
In the previous reduction, the main issue when adding the missing
$t$-links was that a naive implementation creates overlapping
ensembles. 
It is here that quantified variables help since they allow
us to speak about elements without forcing them to be part of an
ensemble. %We use slight variations of $q_1$ to $q_3$ to generate a sparse grid without overlapping ensembles. 
Details are in the appendix.\footnote{The construction given in
the conference version of this paper contains a mistake. We do 
currently not know whether the result holds for  $\ALC^\mn{abs}[\textnormal{ca,cr}]$,
as stated there.}

% {\color{red}We use the following concept abstraction and
% refinement:
% %
% $$
% \begin{array}{l}
% L_1{:}\, U_1 \abs L{:}q \\[1mm]
%   L{:}\, q \land t(x_3, x_4) \rfn L_1{:}\,U_1 \ \text{ where} \\[1mm]
%     q = \exists x_1\exists x_2\, c(x_1, x_3) \land t(x_1, x_2) \land c(x_2, x_4).
% \end{array}
% $$}

\section{Conclusion}

We have introduced DLs that support multiple levels of
abstraction and include operators based on CQs for relating
these levels. % describing how an
% object on a coarser level of abstraction corresponds to an ensemble of
% objects on a finer level. 
As future work, it would be interesting to
analyse the complexity of abstraction DLs based on Horn DLs such as
\EL, \ELI and Horn-\ALCI. It would also be interesting to % investigate
% variants of abstraction DLs where the length of the ensembles that
% abstraction and refinement operators speak about is variable rather
% than fixed, to
design an ABox formalism suitable for abstraction
DLs, and to use such DLs for ontology-mediated querying. Finally, our work
leaves open some decidability questions such as for
$\ALC^{\mn{abs}}[\textnormal{cr}]$ and
$\EL^{\mn{abs}}[\textnormal{cr},\textnormal{ca}]$ under the DAG
semantics and with quantified variables.

\section*{Acknowledgments}

The research reported in this paper has been supported by the German Research Foundation DFG, 
as part of Collaborative Research Center (Sonderforschungsbereich) 
1320 Project-ID 329551904 “EASE - Everyday Activity Science and Engineering”, University of Bremen (http://www.ease-crc.org/). 
The research was conducted in subproject ``P02 - Ontologies with Abstraction''.

This work is partly supported by BMBF (Federal Ministry of Education and Research)
in DAAD project 57616814 (\href{https://secai.org/}{SECAI, School of Embedded Composite AI})
as part of the program Konrad Zuse Schools of Excellence in Artificial Intelligence.

\bibliographystyle{kr}
\bibliography{literature}

\cleardoublepage

% \end{document}

% \section{Reasoning}

% \section{Examples}
% \subsection{Downwards existential}
% An arm consists of a hand, a forearm and an upper arm that are joined together. 
% Let $L \prec L'$.

% \begin{align*}
%  \text{Arm} \sqsubseteq_{L'} L : \text{UpperArm}(x) \land \text{joinedTo}(x,y) \land \text{Forearm}(y) \land 
%  \text{joinedTo}(y,z) \land \text{Hand}(z)
% \end{align*}

% \subsection{Upwards existential}
% Could use the arm example again, just the other way around.
% Mixing eggs flour and milk gives us pancake batter.
% Let $L \prec L'$.

% \begin{align*}
%   L: \text{Egg}(x) \land \text{Flour}(y) \land \text{Milk}(z) \land \text{mixedWith}(x,y)
%   \land \\\text{mixedWith}(y,z) \land \text{mixedWith}(x,z) \sqsubseteq_\uparrow
%   L': \text{PancakeBatter}
%  \end{align*}

% \subsection{Abstraction layer ``in the middle''}
% Let's say we have 3 layers $L_1, L_2, L_3$ with $L_1 \prec L_2 \prec L_3$.
% We have data about $L_1$ and $L_3$ and want to use reasoning and abstraction to make conclusions about $L_2$.
% For example (if i remember correctly) $L_1$ could be sensor data, $L_2$ something like detailed motion description and $L_3$ general motion description.
% A datapoint in $L_3$ could then be \emph{make pancakes} and in $L_1$ we have the sensor data for the joints and other sensors.
% We then want to reason what steps the person/robot took (go to drawer, get pan, etc.) in $L_2$.

% Here it would be useful to do ``upwards'' reasoning from $L_1$ to $L_2$ and ``downwards'' reasoning from $L_3$ to $L_2$.

\appendix

\section{Preliminaries}

We introduce some additional preliminaries that are used in the
detailed proofs provided in the appendix.

\subsection{Conservative Extensions}
A \emph{signature} is a set of concept and role names that in this 
context are uniformly referred to as \emph{symbols}. The set of 
symbols that occur in an ontology \Omc is denoted by $\mn{sig}(\Omc)$. 
Note that this does not include abstraction levels. 

\medskip

Given two ontologies $\Omc_1$ and $\Omc_2$, we say that $\Omc_2$ is a
\emph{conservative extension} of $\Omc_1$ if
\begin{enumerate}
  \item $\mn{sig}(\Omc_1) \subseteq \mn{sig}(\Omc_2)$, 
  \item every model of $\Omc_2$ is a model of $\Omc_1$, and 
  \item for every model $\Imc_1$ of $\Omc_1$ there exists a model $\Imc_2$ of $\Omc_2$ such that 
  \begin{itemize}
    \item[] $A^{\Imc_1} = A^{\Imc_2}$ for all concept names $A \in \mn{sig}(\Omc_1)$ and 
    \item[] $r^{\Imc_1} = r^{\Imc_2}$ for all role names $r \in \mn{sig}(\Omc_1)$. 
  \end{itemize}
\end{enumerate}

\subsection{Alternating Turing Machines}

We briefly recall the definition of alternating Turing machines
(ATMs).  An ATM $M$ is a tuple $M = (Q, \Sigma, \Gamma, q_0, \Delta)$
where
\begin{itemize}
  \item $Q = Q_\exists \uplus Q_\forall \uplus \{q_a, q_r\}$ is the set of \emph{states} and 
  consists of \emph{existentital states} in $Q_\exists$, \emph{universal states} in $Q_\forall$, 
  an \emph{accepting state} $q_a$, and a \emph{rejecting state} $q_r$; 
  \item $\Sigma$ is the \emph{input alphabet}; 
  \item $\Gamma$ is the \emph{work alphabet} that contains the \emph{blank symbol} $\square$
  and satisfies $\Sigma \subseteq \Gamma$; 
  \item $q_0 \in Q_\exists \cup Q_\forall$ is the \emph{starting state}; and 
  \item $\Delta \subseteq Q \setminus \{q_a, q_r\} \times \Gamma \times Q \times \Gamma \times \{L, R\}$ is the \emph{transition relation}. 
\end{itemize}
We use $\Delta(q, \sigma)$ to denote $$\{(q', \sigma', M) \mid (q, \sigma, q', \sigma', M) \in \Delta\}$$
and assume w.l.o.g.\ $|\Delta(q, \sigma)| = 2$ for all such sets. 

A \emph{configuration} of an ATM is a word $wqw'$ with $w,w' \in \Sigma^*$ 
and $q \in Q$. 
This has the intended meaning that the one-sided tape contains $ww'$ followed by only blanks, 
the ATM is in state $q$, and the head is on the symbol just after $w$. 
The \emph{successor configurations} of a configuration are defined as usual in terms of the transition relation $\Delta$. 
A \emph{halting configuration} is of the form $wqw'$ with $q \in \{q_a, q_r\}$. 

A \emph{computation} of an ATM $M$ on a word $w$ is a (finite or infinite) sequence of configurations 
$K_0, K_1, \dots$ such that $K_0 = q_0w$ and $K_{i+1}$ is a successor configuration of $K_i$ for all $i \geq 0$. 
The \emph{acceptance} of a configuration $K = wqw'$ depends on $q$: 
if $q = q_a$ then $K$ is accepting; 
if $q = q_r$ then $K$ is rejecting; 
if $q \in Q_\exists$, then $K$ is accepting iff at least one successor configuration is accepting; 
if $q \in Q_\forall$, then $K$ is accepting iff both successor configuration are accepting. 
An ATM $M$ with starting state $q_0$ \emph{accepts} the word $w$ if 
the \emph{initial configuration} $q_0w$ is accepting. 
We use $L(M)$ to denote the language $M$ is accepting. 

% In the following we use a lemma by Chandra et al. \cite{chandra1981alternation}

% \begin{lemma}
%   \label{lem:chandra}
%   There is an exponentially space bounded ATM $M$ whose word problem
%   is 2\ExpTime -hard.
% \end{lemma}

\section{Proofs for Section~\ref{sec:alci_cr_exp_ub}}

We prove the correctness of the algorithm, stated as
Lemma~\ref{lem:corr1} in the main body of the paper. We also analyze
the running time of the algorithm.

We repeat Lemma~\ref{lem:corr1} here for the reader's convenience.

% {\color{red}I changed the writing a bit in the main part; you might
%   want to bring back here the sequence $\Mmc_0,\Mmc_1,\dots$;
%   also, the lemma is new which might require some changes in writing.}

\lemcorrone*

We split the proof of Lemma~\ref{lem:corr1} into a soundness part
(``if'' direction) and a completeness part (``only if'' direction).

%\medskip
     %      \textbf{Soundness (``if'').}
\subsection{Soundness}
Assume that our algorithm returns `satisfiable'.
This implies that there is a set $\Mmc$ of mosaics in which all
mosaics are good, that contains a mosaic
${M_0}$ with $L^{{M_0}} = L_0$ and $C_0^{\Imc^M_{L_0}} \neq \emptyset$, 
and for each $L \in \Abf_\Omc$ it contains a mosaic of level $L$.
We have to show that we can construct a model
$\Imc$ of $\Omc$ such that ${C_0}^{\Imc_{L_0}} \neq \emptyset$.

Note that throughout the following construction, when working with multiple mosaics we generally consider their interpretation domains
to be disjoint; this can be easily achieved by renaming. In the following, we
construct a sequence $\Imc^0,\Imc^1,\dots$ of A-interpretations and obtain 
the desired model \Imc in the limit. For bookkeeping purposes, along with $\Imc^0,\Imc^1,\dots$ we
construct a mapping $M$ that assigns
a mosaic $M(d)$ to every domain element $d \in \bigcup_{i \geq 0}\Delta^{\Imc^i}$.

We start with defining $\Imc^0 = (\prec, (\Imc^0_L)_{L \in \Abf_\Omc}, \rho^0)$ by setting
\begin{align*}
   \Imc^0_L &= \biguplus_{\substack{M \in \Mmc \\L^M=L}} \Imc^M \text{\qquad for all $L \in \Abf_\Omc$};\\
  \rho^0 &= \emptyset;
\end{align*}
and set $M(d)=M$ if $d \in \Delta^{\Imc^M}$, for all $d \in \Delta^{\Imc^0}$ and $M \in \Mmc^*$.

To construct $\Imc^{i+1} = (\prec, (\Imc^{i+1}_{L})_{L \in \Abf_\Omc}, \rho^{i+1})$ from $\Imc^{i}$ we start with 
$\Imc^{i+1} = \Imc^i$. Then for every $L \in \Abf_\Omc$:
\begin{enumerate}
  \item consider every $A \sqsubseteq_L \exists R . B \in \Omc$ and $d \in A^{\Imc^i_L}$ with
  $d \notin (\exists R.B)^{\Imc^i_L}$.
  Condition~1 of goodness of mosaics implies that there is a mosaic $M' \in \Mmc$ with $L^{M'} = L$
  and a $d' \in \Delta^{\Imc^{M'}}$ that satisfy the subconditions~$1(a)$ to $(c)$ w.r.t.\ $M(d)$.
  We do the following:
  \begin{itemize}
    \item (disjointly) add $\Imc^{M'}$ to $\Imc^{i+1}_L$ and
    \item add $(d,d')$ to $S^{\Imc^{i+1}_L}$ for all $R \sqsubseteq_L S \in \Omc$.
  \end{itemize}
  Also set $M(d) = M'$ for all $d \in \Delta^{\Imc^{M'}}$. 
  
  \item consider every concept refinement $L'{:}q(\bar x) \rfn L{:}A$ in \Omc such that
  $d \in A^{\Imc^i_L}$.
  Specifically let $Q = \{q \mid L'{:}q(\bar x) \rfn L{:}A \in \Omc \text{ and } d
  \in A^\Imc\}$. 
  If $Q \neq \emptyset$, then by Condition~2 of goodness of mosaics there is a mosaic $M' \in \Mmc$ with $L^{M'} = L$ and tuple $\bar e$ over $\Delta^{\Imc^{M'}}$
  such that $\bar e \in q(\Imc^{M'})$ for all $q \in Q$.
  We do the following:
  \begin{itemize}
    \item (disjointly) add $\Imc^{M'}$ to $\Imc^{i+1}_{L'}$,
    \item (disjointly) add $\bar e$ to $\rho^{i+1}_{L'}(d)$.
  \end{itemize}
  Also set $M(d) = M'$ for all $d \in \Delta^{\Imc^{M'}}$.
\end{enumerate}
% If an element $d \in \Delta^{\Imc^i}$ for any $i \geq 0$ was introduced as a copy of $e \in \Delta^{\Imc^M}$ for some $M \in \Mmc$,
%   then we use $d^\uparrow$ to denote $e$.
As announced, we take $\Imc$ to be the limit of the constructed sequence of $\Imc^0,\Imc^1,\dots$.

What remains to show is that $\Imc$ is a model of $\Omc$ and that $C_0$ is satisfied on $L_0$.
We do this step by step starting with the basic condition that $\Imc$ is an $A$-interpretation.
Subsequently, we will show that it satisfies all CIs and RIs in $\Omc$ and lastly that it satisfies all concept refinements in  $\Omc$ and also that $C_0$ is satisfied on $L_0$.
\begin{lemma}
  \label{lem:alci_cr_comp_intrp}
  $\Imc = (\prec, (\Imc_L)_{L \in \Abf_\Omc}, \rho)$ is an A-interpretation.
\end{lemma}
\begin{proof}
  We go through the three conditions of A-interpretations.
  \begin{itemize}
    \item our relation ``$\prec$'' is such that  the
    directed graph $(\AI,\{(L',L)\mid L \prec L' \})$ is a
    tree, since our algorithm did not abort;
    \item by definition, $\Delta^{\Imc^0_L}$ is non-empty for all $L \in \Abf_\Omc$; thus the same holds for $\Delta^{\Imc_L}$;
    \item whenever we added a tuple to the range of $\rho$ in the construction of $\Imc$, all the elements were part of a freshly added mosaic.
    Thus there is always at most one $d \in \Delta^{\Imc}$ such that $e$ occurs in $\rho_{L(e)}(d)$.
  \end{itemize}
\end{proof}

\begin{lemma}
  \label{lem:alci_abs_cr_soundness}
  $\Imc$ satisfies all CIs and RIs in $\Omc$.
\end{lemma}
\begin{proof}
  First, consider any role inclusion $R \sqsubseteq_L S \in \Omc$.
  For an edge $(d,e) \in R^{\Imc_L}$ there are two cases.
  If $(d,e)$ was part of a mosaic then the definition of mosaics and our construction imply $(d,e) \in S^{\Imc_L}$.
  Otherwise $(d,e)$ was added to $R^{\Imc_L}$ in Step~1 of the construction of $\Imc$
  which immediately implies $(d,e) \in S^{\Imc_L}$ since this is part of Step~1.

  Now we go through each form a CI in $\Omc$ can have.
  \begin{itemize}
    \item $\top \sqsubseteq_L A, A_1 \sqcap A_2 \sqsubseteq_L A, A \sqsubseteq_L \neg B, \neg B \sqsubseteq_L \neg A$, 
    are satisfied because of the definition of mosaics and the construction of $\Imc$.
    \item If $A \sqsubseteq_L \exists R.B \in \Omc$ and $d \in A^{\Imc_L}$ then Step~$1$ in the construction of $\Imc$ implies
    that there is a $d' \in B^{\Imc_L}$ with $(d,d') \in R^{\Imc_L}$.
    \item If $\exists R.B \sqsubseteq_L A \in \Omc$ and $(d,d') \in R^{\Imc_L}$ with $d' \in B^{\Imc_L}$, then there are two cases.  
    First, consider the case that $M(d) = M(d')$. 
    Then the definition of mosaics and our construction imply $d \in A^{\Imc_L}$.
    If $M(d) \neq M(d')$ then $(d,d')$ must have been added to $R^ {\Imc_L}$ in Step~$1$ of constructing $\Imc$ which in turn implies $d \in A^{\Imc_L}$.
  \end{itemize}  
\end{proof}
The following lemma establishes soundness of the algorithm.
\begin{lemma}
  $\Imc$ is a model of $\Omc$ with ${C_0}^{\Imc_{{L_0}}} \neq \emptyset$.
\end{lemma}
\begin{proof}
  By Lemma~\ref{lem:alci_cr_comp_intrp}, $\Imc$ is an A-interpretation and by Lemma~\ref{lem:alci_abs_cr_soundness}, it satisfies all CIs and RIs in $\Omc$. 
  Since ${M_0} \in \Mmc$, our construction of $\Imc^0$ implies that ${C_0}^{\Imc^0_{{L_0}}} \neq \emptyset$
  and by the construction of $\Imc$ this holds true for $\Imc$ as well.
  It only remains to show that all concept refinements are satisfied but this follows directly from Step~2 in the construction of $\Imc$
  since it ensures that all concept refinements are satisfied in $\Imc^i$ for each $i$ and thus also in $\Imc$.
\end{proof}

% \medskip
% \textbf{Completeness (``only if'').}

\subsection{Completeness}
Assume that there is a model $\Imc = (\prec,(\Imc_L)_{L \in \Abf_\Omc},\rho)$ of 
$\Omc$ with ${C_0}^{\Imc_{{L_0}}} \neq \emptyset$.
We have to show that our algorithm returns `satisfiable'.

Let $V \subseteq \Delta^{\Imc_L}$ be a set of domain elements for some $L \in \Abf_\Omc$. 
If $|V| \leq ||\Omc||$, we use $M_V = (L, \Jmc)$ to denote the mosaic \emph{of $V$ in $\Imc_L$} where $\Jmc$ is defined as follows:
\begin{alignat*}{2}
  \Delta^{\Jmc} &= V &&\\
  A^{\Jmc} &= A^{\Imc_L} \cap V &&\text{for all $A$ in $\mn{sig}(\Omc) \cap \Cbf$}\\
  r^{\Jmc} &= \{(d,e) \mid d,e \in V &&\text{ and } (d,e) \in r^{\Imc_L}\}  \\
  & &&\text{for all $r$ in $\mn{sig}(\Omc) \cap \Rbf$}
\end{alignat*}

It is easy to verify that $M_V$ is a mosaic since $\Imc$ is a model of $\Omc$ and $|\Delta^{\Jmc}| \leq ||\Omc||$.
Technically we have to rename elements which can be done by defining a 
simple injective function that maps $\Delta^{\Jmc}$ to $\Delta$.
If we renamed an element $d \in \Delta^{\Imc}$ to $e \in \Delta^{\Imc^{M_V}}$, then we use $e^\downarrow$ to denote $d$.
% We might also use this notation for tuples $\bar e$, where $\bar e^\downarrow$ implies that $\cdot^\downarrow$ is applied to all elements in the tuple,
% or functions $f$, where $f^\downarrow$ is the function obtained by applying $\cdot^\downarrow$ to all elements in $\ran (f)$.

The algorithm computes the final set $\Mmc^*$ of mosaics by starting with the set $\Mmc_0$  of all mosaics for $\Omc$
and then repeatedly and exhaustively eliminating mosaics that are not good.
Let us assume the computation did $n$ eliminations of mosaics that were not good.
We then use $\Mmc_0, \Mmc_1, \dots, \Mmc_n$ to denote the sets of the computation, implying $\Mmc_n = \Mmc^*$.

We now want to show that all the $M_V$s are contained in $\Mmc^*$ since that will trivially imply that our algorithm returns `satisfiable'.
\begin{lemma}
  \label{lem:mv_in_mstar}
%  Let $n$ be the number of iterations of our algorithm.
  For all $i \leq n$:
  $$\{M_V \mid \exists  L \in \Abf_\Omc \text{ s.t. } V \subseteq \Delta^{\Imc_L}
  \text{ with } |V| \leq ||\Omc|| \} \subseteq \Mmc_i.$$
\end{lemma}
\begin{proof}
The proof is by induction on $i$.
For $i = 0$, our assumption follows from the fact that all the $M_V$s are mosaics.

For $i > 0$, let us consider a mosaic $M_V \in \Mmc_i$ for any $L \in \Abf_\Omc$ and $V \subseteq \Delta^{\Imc_L}$ with $|V| \leq ||\Omc||$.
Let $A \sqsubseteq_L \exists R . B \in \Omc$, $d \in A^{\Imc^{M_V}}$, and
$d \notin (\exists R.B)^{\Imc^{M_V}}$.
We have to show that there are a mosaic $M' \in \Mmc_i$ with $L^{M'} = L$ and a domain element $d' \in {\Delta^{\Imc^{M'}}}$ that satisfy
Condition~1 of goodness of mosaics. From the construction of $M_V$ we know that $d^\downarrow \in (\exists R.B)^{\Imc_L}$ and $\Imc$ being a model of $\Omc$ implies
that there is an $e \in A^{\Imc_L}$ with $(d^\downarrow,e) \in R^{\Imc_L}$.
By the IH $M_{\{e\}} \in \Mmc_i$, and thus Condition~1 of goodness is satisfied.

For Condition~2 of goodness, let $d \in A^{\Imc^{M_V}}$ and $Q = \{q \mid L'{:}q(\bar x) \rfn L{:}A \in \Omc\} \neq \emptyset$.
The construction of $M_V$ implies that $d^\downarrow \in A^{\Imc_L}$ and 
since $\Imc$ is a model there is a tuple $\bar e \in \rho_{L'}(d)$
such that $\bar e \in q(\Imc_L)$ for all $q \in Q$.
Now we can apply the IH to get $M_{\{e \mid e \text{ in } \bar e\}} \in \Mmc_i$ and thus
Condition~2 of goodness is satisfied as well.
\end{proof}

Now since $\Imc$ is a model there is at least one element $d \in \Delta^{\Imc_L}$ for all $L \in \Abf_\Omc$,
and also an element $d_0 \in {C_0}^{\Imc_{{L_0}}}$.
We can now apply Lemma~\ref{lem:mv_in_mstar} to imply that $M_{\{d\}} \in \Mmc^*$ for all these mentioned domain elements
and thus the conditions for our algorithm to return `satisfiable' are met.

\subsection{Running Time}

Since the domain of a mosaic has a maximum of
$||\Omc||$ elements, there are at most $2^{{\mn{poly}(||\Omc||)}}$
different mosaics.  Consequently the length of the constructed
sequence $\Mmc_0,\Mmc_1,\dots$ is also bounded by
$2^{{\mn{poly}(||\Omc||)}}$.  Checking goodness of mosaics requires us
to compute answers to CQs.  Doing this in a brute force way results in
checking at most $||\Omc||^{||\Omc||} \in 2^{\mn{poly}(||\Omc||)}$
candidates for homomorphisms from a CQ into a mosaic (for a fixed
answer). In summary, the running time of our algorithm is at most
${2^{\mn{poly}(||\Omc||)}}$.

\section{Proofs for Section~\ref{sec:alci_all_ar_statements_ub}}

In the main part of the paper, we have omitted details regarding the
converses of Conditions~2-5 of edge candidates. We start with
giving a complete list of conditions, including the converses. 

\medskip

Let $M = ((\Imc_{L})_{L \in \Abf_\Omc},\rho, f_\mn{in}, f_\mn{out})$
and
$M' = ((\Imc'_{L})_{L \in \Abf_\Omc},\rho',
f'_{\mn{in}},f'_{\mn{out}})$. A set $E$ of $M,M'$-edges is an
\emph{edge candidate} if it satisfies the following conditions:
\begin{enumerate}
\item[1.] $R(d,e) \in E$ and $L(d) = L$ implies $S(d,e) \in E$, for all
$R\sqsubseteq_L S \in \Omc$;
\item[2.] if $\exists R . A \sqsubseteq_L B \in \Omc$, $R(d,d') \in E$, 
  and  $d' \in A^{\Imc'_L}$, then $d \in B^{\Imc_L}$;

\item[2$'$.] if $\exists R^- . A \sqsubseteq_L B \in \Omc$, $R(d,d') \in E$, and $d \in
A^{\Imc_L}$, then $d' \in B^{\Imc_L}$;

\item[3.] for all $L \in \Abf_\Omc$, all $(q, h) \in f_\mn{out}(L)$, where 
  $q = E_q \uplus q|_{\overline V}$ for $V = \mn{dom}(h)$,
  and all functions $g$ from $\overline V \cap \mn{var}(E_q)$ to $\Delta^{\Imc'_L}$
  such that $R(h(x), g(y)) \in E$ for all $R(x,y) \in E_q$,
  we have $(q|_{\overline V}, g) \in f'_{\mn{in}}(L)$;
  \item[3$'$.] for all $L \in \Abf_\Omc$, all $(q, h) \in f'_{\mn{out}}(L)$, where 
  $q = E_q \uplus q|_{\overline V}$ for $V = \mn{dom}(h)$,
  and all functions $g$ from $\overline V \cap \mn{var}(E_q)$ to $\Delta^{\Imc_L}$
   such that $R(h(x), g(y)) \in E$ for all $R(x,y) \in E_q$,
  we have $(q|_{\overline V}, g) \in f_{\mn{in}}(L)$;
  \item[4.] for all $R(d,d') \in E$ and all role refinements $L{:}q(\bar x, \bar y) \rfn L'{:}q_R(x,y) \in \Omc$,
 such that $q = q|_{\bar x} \uplus E_q \uplus q|_{\bar y}$,
   $q_R=C_x(x) \wedge R(x,y) \wedge C_y(y)$, 
   $C_x \in t(d)$, and $C_y \in t'(d')$:
  \begin{enumerate}
    \item $\rho_{L}(d)$ and $\rho'_{L}(d')$ are defined;
    \item $h:\bar x \mapsto \rho_{L}(d)$ is a homomorphism from
      $q|_{\bar x}$ to $\Imc_L$;
    \item $h':\bar y \mapsto \rho'_{L}(d')$ is a homomorphism
      from $q|_{\bar y}$ to $\Imc'_L$;
    \item $\{R(h(x), h'(y)) \mid R(x,y) \in E_q\} \subseteq E$;
  \end{enumerate}
  \item[4$'$.] for all $R(d',d) \in E$ and role refinements $L{:}q(\bar x, \bar y) \rfn L'{:}q_R(x,y) \in \Omc$,
 such that $q = q|_{\bar x} \uplus E_q \uplus q|_{\bar y}$,
   $q_R=C_x(x) \wedge R(x,y) \wedge C_y(y)$, 
   $C_x \in t'(d')$, and $C_y \in t(d)$:
  \begin{enumerate}
    \item $\rho_{L}(d)$ and $\rho'_{L}(d')$ are defined;
    \item $h:\bar x \mapsto \rho'_{L}(d')$ is a homomorphism from
      $q|_{\bar x}$ to $\Imc'_L$;
    \item $h':\bar y \mapsto \rho_{L}(d)$ is a homomorphism
      from $q|_{\bar y}$ to $\Imc_L$;
    \item $\{R(h(x), h'(y)) \mid R(x,y) \in E_q\} \subseteq E$;
  \end{enumerate}
  \item[5.] for all role abstractions $L'{:}R \abs L{:}q(\bar x, \bar y) \in \Omc$,
  where $q = q|_{\bar x} \uplus E_q \uplus q|_{\bar y}$,
  all homomorphisms $h$ from $q|_{\bar x}$ to $\Imc_L$,
  and all homomorphisms $g$ from $q|_{\bar y}$ to $\Imc'_L$
  such that $\{S(h(x), g(y)) \mid S(x,y) \in E_q\} \subseteq E$, 
  there are $d \in \Delta^{\Imc_{L'}}$ and
  $d' \in \Delta^{\Imc'_{L'}}$ with $\rho_L(d) = h(\bar x)$, 
  $\rho'_L(d') = g(\bar y)$, and $R(d,d') \in E$;
  \item[5$'$.] for all role abstractions $L'{:}R \abs L{:}q(\bar x, \bar y) \in \Omc$,
  where $q = q|_{\bar x} \uplus E_q \uplus q|_{\bar y}$,
  all homomorphisms $h$ from $q|_{\bar x}$ to $\Imc'_L$,
  and all homomorphisms $g$ from $q|_{\bar y}$ to $\Imc_L$
  such that $\{S(h(x), g(y)) \mid S(x,y) \in E_q\} \subseteq E$, 
  there are $d \in \Delta^{\Imc_{L'}}$ and
  $d' \in \Delta^{\Imc'_{L'}}$ with $\rho_L(d) = g(\bar y)$, 
  $\rho'_L(d') = h(\bar x)$, and $R(d',d) \in E$.

\end{enumerate}

We now prove the correctness of the algorithm, stated as
Lemma~\ref{lem:corr2} in the main body of the paper. We also analyze
the running time of the algorithm.

We repeat Lemma~\ref{lem:corr2} here for the reader's convenience.
\lemcorrtwo*

We split the proof of Lemma~\ref{lem:corr2} into a soundness part
(``if'' direction) and a completeness part (``only if'' direction).

\subsection{Soundness}
Assume that our algorithm returns `satisfiable' and thus satisfies Condition (i) and (ii) of Lemma~\ref{lem:corr2}.
We have to show that we can construct a model
$\Imc$ of $\Omc$ such that ${C_0}^{\Imc_{{L_0}}} \neq \emptyset$.

Note that throughout the following construction, when working with multiple mosaics we generally consider their interpretation domains
to be disjoint; this can be easily achieved by renaming. In the following, we
construct a sequence $\Imc^0,\Imc^1,\dots$ of A-interpretations and obtain 
the desired model \Imc in the limit. For bookkeeping purposes, along with $\Imc^0,\Imc^1,\dots$ we
construct a mapping $M$ that assigns
a mosaic $M(d)$ to every domain element $d \in \bigcup_{i \geq 0}\Delta^{\Imc^i}$.

We start with defining $\Imc^0 = (\prec, (\Imc^0_L)_{L \in \Abf_\Omc}, \rho^0)$ by setting
\begin{align*}
  \Imc^0_{L} &= \biguplus_{M \in \Mmc^*} \Imc^M_L \text{ for all } L \in \Abf_\Omc;\\
  \rho^0 &= \biguplus_{M \in \Mmc^*} \rho^M.
\end{align*}
and set $M(d)=M$ if $d \in \Delta^{\Imc^M}$, for all $d \in \Delta^{\Imc^0}$ and $M \in \Mmc^*$.

To construct $\Imc^{i+1} = (\prec, (\Imc^{i+1}_{L})_{L \in \Abf_\Omc}, \rho^{i+1})$ from $\Imc^{i}$ we start with 
$\Imc^{i+1} = \Imc^i$. Then consider every $L \in \Abf_\Omc$,
$A \sqsubseteq_L \exists R . B \in \Omc$ and $d \in A^{\Imc^i_L}$ with $d \notin (\exists R.B)^{\Imc^i_L}$.
By the goodness condition of mosaics, there is a mosaic $M' \in \Mmc$,
 a $d' \in B^{\Imc^{M'}_L}$ and an $M, M'$ edge candidate $E$ such that $R(d,d') \in E$. 

  For all $L \in \Abf_\Omc$, do the following:
  \begin{itemize}
    \item (disjointly) add $\Imc^{M'}_L$ to $\Imc^{i+1}_L$,
    \item (disjointly) add $\rho^{M'}$ to $\rho^{i+1}$, and
    \item add $\{(e, e') \mid R(e,e') \in E \text{ and } L(e) = L(e') = L\}$ to $R^{\Imc^{i+1}_L}$.
  \end{itemize}
  Also set $M(d) = M'$ for all $d \in \Delta^{\Imc^{M'}_{L}}$.

  \smallskip
  If an element $d \in \Delta^{\Imc^i}$ for any $i \geq 0$ was introduced as a copy of $e \in \Delta^{\Imc^M}$ for some $M \in \Mmc$,
  then we use $d^\uparrow$ to denote $e$.
As announced, we take $\Imc$ to be the limit of the constructed sequence of $\Imc^0,\Imc^1,\dots$.

What remains to show is that $\Imc$ is a model of $\Omc$ and that $C_0$ is satisfied on $L_0$.
We do this step by step starting with the basic condition that $\Imc$ is an $A$-interpretation.
After that we will show that it satisfies all CIs and RIs in $\Omc$ and lastly that it satisfies all abstraction and refinement statements in $\Omc$ and that $C_0$ is satisfied on $L_0$.

\begin{lemma}
  \label{lem:alci_abs_intrp}
  $\Imc = (\prec, (\Imc_L)_{L \in \Abf_\Omc}, \rho)$ is an A-interpretation.
\end{lemma}
\begin{proof}
  We go through the three conditions of A-interpretations.
  \begin{itemize}
    \item our relation ``$\prec$'' is such that  the
    directed graph $(\AI,\{(L',L)\mid L \prec L' \})$ is a
    tree, since our algorithm did not abort;
    \item by definition, $\Delta^{\Imc^0_L}$ is non-empty for all $L \in \Abf_\Omc$; thus the same holds for $\Delta^{\Imc_L}$;
    \item all the $\rho^M$ are already refinement functions and we only add them disjointly to $\Imc$ in every step of the construction. 
    Thus every object participates in at most one ensemble.
  \end{itemize}
\end{proof}

\begin{lemma}
  \label{lem:alci_abs_soundness}
  $\Imc$ satisfies all CIs and RIs in $\Omc$.
\end{lemma}
\begin{proof}
  First, consider any role inclusion $R \sqsubseteq_L S \in \Omc$.
  For an edge $(d,e) \in R^{\Imc_L}$ there are two cases.
  The edge might have been part of a mosaic.
  Then Condition~1 of mosaics and our construction imply $(d,e) \in S^{\Imc_L}$.
  Otherwise $(d,e)$ was added to $R^{\Imc_L}$ in Step~1 of the construction of $\Imc$
  and then Condition~1 of edge candidates implies $(d,e) \in S^{\Imc_L}$.

  Now we go through each form a CI in $\Omc$ can have.
  \begin{itemize}
    \item $\top \sqsubseteq_L A, A_1 \sqcap A_2 \sqsubseteq_L A, A \sqsubseteq_L \neg B, \neg B \sqsubseteq_L \neg A$, 
    are satisfied because of Condition~1 of mosaics and the construction of $\Imc$.
    \item If $A \sqsubseteq_L \exists R.B \in \Omc$ and $d \in A^{\Imc_L}$ then Step~$1$ in the construction of $\Imc$ implies
    that there is a $d' \in B^{\Imc_L}$ with $(d,d') \in R^{\Imc_L}$.
    \item If $\exists R.B \sqsubseteq_L A \in \Omc$ and $(d,d') \in R^{\Imc_L}$ with $d' \in B^{\Imc_L}$, then there are two cases.  
    First, consider the case that $M(d) = M(d')$. 
    Then Condition~1 of mosaics and our construction imply $d \in A^{\Imc_L}$.
    If $M(d) \neq M(d')$ then $(d,d')$ must have been added to $R^ {\Imc_L}$ in Step~$1$ of constructing $\Imc$ which 
    implies $d \in A^{\Imc_L}$ because of Conditions~2 and 2' of edge candidates.
  \end{itemize}  
\end{proof}

Before showing that $\Imc$ also satisfies all abstractions and refinements in $\Omc$
we show two intermediary lemmas.
Intuitively they say that the $L$-ensembles that are the result of matching abstraction CQs are never part of multiple mosaics.

\begin{lemma}
  \label{lem:alci_all_ar_conc_abs_only_in_mosaic}
  For all concept abstractions $L'{:}A \abs L{:}q(\bar x)$ in \Omc and
  all homomorphisms $h$ from $q$ to $\Imc_L$, 
  $M(d) = M(d')$ for all $d, d' \in \ran(h)$.
\end{lemma}
\begin{proof}
  Proof by contradiction.
  Assume there is a homomorphsim $h$ from $q$ to $\Imc_L$ such that for
  some elements $d, d' \in \ran(h)$, we have $M(d) \neq M(d')$.
  Let $M(d) = M_0$ and $V_0 = \{x \mid x \in \mn{var}(q) \text{ and } M(h(x)) = M_0\}$ be the set of variables in $q$ that are matched into $M_0$.
  Note that by our assumption $V_0$ is neither empty nor equal to $\mn{var}(q)$.
  Thus Condition~2 of mosaics applies, which states that there is a component $p_0$ of $q$ w.r.t.\ $V_0$
  with $p_0 = E_{p_0} \uplus p_0|_{\overline{V_0} \cap \mn{var}(p_0)}$ and
  $(p_0, h|_{V_0 \cap \mn{var}(p_0)}) \in f_\mn{out}^{M_0}(L)$.

  Let $V_1 = \overline{V_0} \cap \mn{var}(E_{p_0})$.
  Then $V_1$ is non-empty and $M(h(x)) \neq M_0$ for all $x \in V_1$, by our assumption.
  Furthermore, $M(h(x)) = M(h(y))$ for all $x,y \in V_1$, by the definition of components.
  Let $M_1 = M(h(x))$ denote this mosaic that is the same for all $x \in V_1$.
  The construction of $\Imc$ implies that in our algorithm there was an $M_0, M_1$ edge candidate $E$ with 
  $R(h(x), h(y)) \in E$ for all $R(x,y) \in E_{p_0}$.
  Condition~3 of edge candidates then yields $(p_0|_{\overline{V_0} \cap \mn{var}(p_0)}, h|_{V_1}) \in f_\mn{in}^{M_1}(L)$.
  For the rest of the proof let $p_1 = p_0|_{\overline{V_0} \cap \mn{var}(p_0)}$.

  Now we can apply Condition~4 of mosaics on $(p_1, h|_{V_1})$.
  Let $V_1' = \{x \mid x \in \mn{var}(p_1) \text{ and } M(h(x)) = M_1\}$.
  There are now two cases.
  If $V_1' = \mn{var}(p_1)$, then we immediately contradict Condition~4 of mosaics because w.r.t.\ $\mn{var}(p_0)$ there can be no component of $p_1$. 
  If $V_1' \subsetneq \mn{var}(p_1)$, then by Condition~4 of mosaics
  there is a component $p_1'$ with $p_1' = E_{p_1'} \uplus p_1'|_{\overline{V_1} \cap \mn{var}(p_1')}$ of $p_1$ w.r.t.\ $V_1'$ such that $(p_1', h|_{V_1' \cap \mn{var}(p_1')}) \in f_\mn{out}^{M_1}(L)$.

  The next step would be to choose $V_2$ as $V_2 = \overline{V_1} \cap \mn{var}(E_{p_1'})$ and repeat the back and forth between $f_\mn{out}$ and $f_\mn{in}$ of mosaics
  we just showed.
  The only difference is that from $M_2$ and onwards we also need Condition~3' of edge candidates, 
  since for example $p_1'$ might match back into $M_0$ via the edges in $E$ (which means $M_2$ would be $M_2 = M_0$).

  The number of iterations is, however, restricted by some $i \leq ||\Omc||$, since the size of the query is bound by $||\Omc||$ and
  consequently at some point we get a forbidden ingoing query $(p_{i}, h|_{V_i}) \in f_\mn{in}^{M_i}(L)$
  such that $h$ matches all variables of $p_{i}$ into $\Imc^{M_i}_L$.
  This contradicts Condition~4 of mosaics and thus every case leads to a contradiction of our assumption.
\end{proof}

\begin{lemma}
  \label{lem:alci_all_ar_role_abs_only_in_mosaic}
  For all role abstractions $L'{:}R \abs L{:}q(\bar x, \bar y)$ in \Omc and
  all homomorphisms $h$ from $q$ to $\Imc_L$,
  \begin{enumerate}
    \item $M(d) = M(d')$ for all $d, d' \in \ran(h|_{\bar x})$ and
    \item $M(e) = M(e')$ for all $e, e' \in \ran(h|_{\bar y})$.
  \end{enumerate}
\end{lemma}
\begin{proof}
  Point~$1$ and Point~$2$ can be proven analogously to Lemma~\ref{lem:alci_all_ar_conc_abs_only_in_mosaic}
  using Condition~3 of mosaics instead of Condition~2.
\end{proof}

Now we are ready to prove the following lemma which establishes the soundness of the algorithm.
The main thing left to prove is that all abstractions and refinements of $\Omc$ are satisfied.

\begin{lemma}
  $\Imc$ is a model of $\Omc$ with ${C_0}^{\Imc_{{L_0}}} \neq \emptyset$.
\end{lemma}
\begin{proof}
  We just proved in Lemma~\ref{lem:alci_abs_intrp} that $\Imc$ is an A-Intepretation and in Lemma~\ref{lem:alci_abs_soundness} that all CIs and RIs are satisfied. 
  Since by Condition~(i) of acceptance for our algorithm there is a mosaic $M \in \Mmc^*$ with
  $C_0^{\Imc^{M}_{L_0}}$, our construction of $\Imc^0$ implies that ${C_0}^{\Imc^0_{{L_0}}} \neq \emptyset$
  and by the construction of $\Imc$ this holds true for $\Imc$ as well.

  Condition~1 of mosaics implies that each mosaic satisfies all concept refinements 
  and since the satisfiability of concept refinements is independent of any roles we add,
   our construction of $\Imc$ naturally implies that
  $\Imc$ also satisfies all concept refinements.

  % Let $d \in C^{\Imc_{L}}$ and $M(d) = M$ for some concept refinement
  % $L'{:}q(\bar x) \rfn L{:}C$ in $\Omc$. 
  % Lemma~\ref{lem:alci_all_ar_type_assign_equal_to_model} implies
  % $C \in t^M(d^\uparrow)$ and Condition~1a of mosaics then yields that $\bar x \mapsto \rho^M_{L'}(d^\uparrow)$ is a type homomorphism from $q$ to $\Imc^M_{L'}$.
  % Now we can again use Lemma~\ref{lem:alci_all_ar_type_assign_equal_to_model} and the definition of $\Imc$
  % to achieve $\rho_{L'}(d) \in q(\Imc_{L'})$, as required.
  
  Let $(d_1, d_2) \in q_R(\Imc_{L})$ for some role refinement
  $L'{:}q(\bar x, \bar y) \rfn L{:}q_R(x, y)$ in $\Omc$ with
  $q = q|_{\bar x} \uplus E_q \uplus q|_{\bar y}$. 
  There are two cases.
  The $d_1^\uparrow$ and $d_2^\uparrow$ can be inside the same mosaic, that is $M(d_1) = M(d_2) = M$.
  Here Condition~1 of mosaics yields that the role refinement is satisfied in $M$ and thus also in $\Imc$ by our construction of $\Imc$. 
  
  In the second case $M(d_1) = M$, $M(d_2) = M'$ and $M \neq M'$.
  The definition of $\Imc$ then implies that there is an edge candidate $E$ between $M$ and $M'$ with 
  $R(d_1^\uparrow, d_2^\uparrow) \in E$ or $R^-(d_2^\uparrow, d_1^\uparrow) \in E$.
  Let us assume $R(d_1^\uparrow, d_2^\uparrow) \in E$ as the other case is analogous.

  By Condition~4a to 4d of edge candidates, there are mappings 
  $h: \bar x \mapsto \rho^M_L(d_1^\uparrow)$ and 
  $h': \bar y \mapsto \rho^{M'}_L(d_2^\uparrow)$ such that 
  $\bar x \bar y \mapsto \rho^M_L(d_1^\uparrow)\rho^{M'}_L(d_2^\uparrow)$ is a homomorphism from
  $q$ to $\Imc_{\cup}$ where $\Imc_{\cup}$ is the
  union of the interpretations of the two mosaics and the edges from the edge candidate defined as
  \begin{align*}
    \Delta^{\Imc_\cup} &= \Delta^{\Imc_L^M} \uplus \Delta^{\Imc_L^{M'}};\\
    A^{\Imc_\cup} &= A^{\Imc_L^{M}} \uplus A^{\Imc_L^{M'}};\\
    R^{\Imc_\cup} &= R^{\Imc_L^{M}} \uplus R^{\Imc_L^{M'}} \uplus \{(h(x), (h'(y)) \mid R(x,y) \in E_q\};
  \end{align*}
  where $A$ ranges over concept names and $R$ over roles.

  Since $\Imc_\cup$ is part of $\Imc_L$ and both $\rho^M$ and $\rho^{M'}$ are part of $\rho$ we have
   $(\rho_{L'}(d_1), \rho_{L'}(d_2)) \in q(\Imc_{L'})$, as required. 
  If earlier in the proof we would have had $R^-(d_2^\uparrow, d_1^\uparrow) \in E$ instead of 
  $R(d_1^\uparrow, d_2^\uparrow) \in E$ then we use Condition~4' of edge candidates instead of Condition~4.

  Let $h: \bar x \mapsto \bar e$ be a homomorphism from $q$ to $\Imc_L$ for some concept abstraction
  $L'{:}A \abs L{:}q(\bar x)$ in $\Omc$.
  Lemma~\ref{lem:alci_all_ar_conc_abs_only_in_mosaic} tells us that $h$ maps all variables in $q$ to the domain elements of a single mosaic $M$.
  Now Condition~1 of mosaics implies that the concept abstraction is satisfied in $M$ and thus also in $\Imc$ by our construction of $\Imc$ 

  Let $h: \bar x\bar y \mapsto \bar e_1\bar e_2$ be a homomorphism from $q$ to $\Imc_L$ for some role abstraction
  $L'{:}R \abs L{:}q(\bar x, \bar y)$ in $\Omc$.  
  We use $h^\uparrow$ to denote replacing every element $d \in \ran(h)$ with $d^\uparrow$.
  Lemma~\ref{lem:alci_all_ar_role_abs_only_in_mosaic} tells us that 
  $h^\uparrow|_{\bar x}$ maps all variables in $q|_{\bar x}$ to the domain elements of a single mosaic $M$ and
  $h^\uparrow|_{\bar y}$ maps all variables in $q_{\bar y}$ to the domain elements of a single mosaic $M'$.
  We consider two cases.
  
  If $M = M'$ then Condition~1 of mosaics implies that the role abstraction is satisfied in $M$ 
  and thus also in $\Imc$ by our construction of $\Imc$.
  
  If $M \neq M'$ then by the construction of $\Imc$, there is an edge candidate $E$ between $M$ and $M'$ such that
  $h^\uparrow$ is a homomorphism from $q$ to $\Imc_\cup$
  where we construct $\Imc_\cup$ as in the role refinement case.
  Thus, either Condition~5 or Condition~5' of edge candidates must apply.
  If Condition~5 applies, then there are $d \in \Delta^{\Imc_L^M}$ and $d' \in \Delta^{\Imc_L^{M'}}$ 
  with $\rho_L^M(d) = \bar e_1^\uparrow$, $\rho_L^{M'}(d') = \bar e_2^\uparrow$, 
  and $R(d,d') \in E$. Again it is easy to verify that our construction of $\Imc$ then satisfies $L'{:}R \abs L{:}q(\bar x, \bar y)$
  since we copied all the components of the mosaics and edges in the edge candidate.
  If Condition~5' of edge candidates applies the proof works analogously.
\end{proof}

% completeness
% completeness
% completeness
% completeness
% completeness

\subsection{Completeness}
Assume that there is a model $\Imc = (\prec,(\Imc_L)_{L \in \Abf_\Omc},\rho)$ of 
$\Omc$ with a domain element $d_0 \in {C_0}^{\Imc_{{L_0}}}$.
We have to show that our algorithm returns `satisfiable'.

We assume w.l.o.g.\ that $\Imc$ does not contain any \emph{unnecessarily large ensemble},
that is, an $L$-ensemble $\bar e$ for any $L \in \Abf_\Omc$ with $||\bar e|| > ||\Omc||$.
We can assume this since any query in $\Omc$ is naturally of size at most $||\Omc||$.

For brevity we define $\mu(d) = \{e \mid e \text{ element in $\rho_L(d)$ for any $L \in \Abf_\Omc$}\}$ for all $d \in \Delta^\Imc$.
For all $d \in \Delta^\Imc$ we define $\mu_1(d) = \{d\}$ and exhaustively compute 
$$ \mu_{i+1}(d) = 
\mu_{i}(d) \cup
\bigcup_{e \in \mu_i(d)} \{e' \mid e \in \mu_i(e')\} \cup
\bigcup_{e \in \mu_i(d)} \mu_i(e)
$$
We use $\mu^*$ to denote the function at which this process stabilizes.
Thus $\mu^*(d)$ contains all elements that can be reached by iterated refinements via $\rho_L$
and abstractions via the inverse of $\rho_L$. 
Intuitively, these are organized in a tree whose elements are the members of $\mu^*(d)$ and 
with an edge $(e,e')$ if $e'$ occurs in $\rho_L(e)$.
Also note that because there are no unnecessarily large ensembles, $|\mu^*(d)| \leq ||\Omc||^{||\Omc||}$.

% For any $d \in \Delta^\Imc$,
% a mosaic $M_d = ((\Imc'_{L})_{L \in \Abf_\Omc},\rho',t', f_\mn{in}', f_\mn{out}')$ of $d$ \emph{in} $\Imc$
% is defined as follows for all $L \in \Abf_\Omc$:
% %
% \begin{alignat*}{2}
%   \Delta^{\Imc'_L} &= \{e \mid e \in {\Delta}^{\Imc_L} \cap \mu^*(d)\},&&\\
%   A^{\Imc'_L} &=      \{e \mid e \in A^{\Imc_L} \cap \mu^*(d)\}                   &&\text{for all concept names $A$},\\
%   R^{\Imc'_L} &=      \{(e_1, e_2) \mid e_1, e_2 \in {\Delta}^{\Imc_L} &&\cap \mu^*(d) \text{ and }\\
%               &      \quad \;\;\;(e_1,e_2) \in R^{\Imc_L}\}                                   &&\text{for all role names $R$},\\
%   \rho'_L(e') &=      \bar e \text{ with } \rho_L(e') = \bar e                    &&\text{for all $e' \in \Delta^{\Imc'}$, and}\\
%   t'(e')      &=      \{C \mid e' \in C^{\Imc_L}\}                                &&\text{for all } e' \in {\Delta}^{\Imc'_L}\text{.}
% \end{alignat*}

For any $d \in \Delta^\Imc$,
we define a mosaic $M_d$ of $d$ \emph{in} $\Imc$.
The first four components of $M_d$ are as follows, for all $L \in \Abf_\Omc$:
\begin{alignat*}{2}
  \Delta^{\Imc_L^{M_d}}     &= \{e \mid e \in {\Delta}^{\Imc_L} \cap \mu^*(d)\};          &&\\
  A^{\Imc^{M_d}_L}          &= \{e \mid e \in A^{\Imc_L} \cap \mu^*(d)\};                  &&\\
  r^{\Imc^{M_d}_L}          &= \{(e_1, e_2) \mid e_1, e_2 \in {\Delta}^{\Imc_L}           && \cap \mu^*(d) \text{ and }\\
                            &\quad \;\;\; (e_1,e_2) \in r^{\Imc_L}\}   ;                   &&\\
  \rho^{M_d}_{L}(e)        &= \bar e \text{ with } \rho_L(e) = \bar e                   &&\text{for all $e \in \Delta^{\Imc^{M_d}}$;}
\end{alignat*}
where $A$ ranges of concept names and $r$ over role names in $\mn{sig}(\Omc)$.
What remains is to define $f_{\mn{in}}^{M_d}$ and $f_{\mn{out}}^{M_d}$.

For all concept abstractions
$L'{:}A \abs L{:}q(\bar x)$ in $\Omc$, all
non-empty $V \subsetneq \mn{var}(q)$, all homomorphisms $h$
from $q|_V$ to $\Imc_L$,
and all components $p$ of $q$ w.r.t.\ $V$ such that $p = E_p \uplus p|_{\overline{V}}$
consider the following cases:
\begin{enumerate}
  \item if $\ran(h) \subseteq \Delta^{\Imc^{M_d}_L}$ and $h$ can be extended to a homomorphism from $q$ to $\Imc_L$,
    then add $(p, h|_{V \cap \mn{var}(p)})$ to $f_\mn{out}^{M_d}(L)$;
  \item if $\ran(h) \subseteq \Delta^{\Imc^{M_d}_L}$ and $h$ cannot be extended to a homomorphism from $ q|_V \cup p$ to $\Imc_L$,
    then add $(p, h|_{V \cap \mn{var}(p)})$ to $f_\mn{out}^{M_d}(L)$;
  \item for all homomorphisms $g$ from $q|_V \cup E_p$ to $\Imc_L$ that extend $h$
    but cannot be extended to a homomorphism from $q|_V \cup p$ to $\Imc_L$,
    if $\ran(g|_{\overline V \cap \mn{var}(E_p)}) \subseteq \Delta^{\Imc^{M_d}_L}$, 
    then add $(p|_{\overline V}, g|_{\overline V \cap \mn{var}(E_p)} )$ to $f_\mn{in}^{M_d}(L)$.
\end{enumerate}

For all role abstractions
$L'{:}R \abs L{:}q(\bar x, \bar y)$ in $\Omc$, all
non-empty $V \subsetneq \mn{var}(q)$ with $V \neq \mn{set}(\bar x)$ and $V \neq \mn{set}(\bar y)$, 
all homomorphisms $h$ from $q|_V$ to $\Imc_L$,
and all components $p$ of $q$ w.r.t.\ $V$ such that $p = E_p \uplus p|_{\overline{V}}$
consider the following cases:
\begin{enumerate}
  \setcounter{enumi}{3}
  \item if $\ran(h) \subseteq \Delta^{\Imc^{M_d}_L}$ and $h$ can be extended to a homomorphism from $q$ to $\Imc_L$,
    then add $(p, h|_{V \cap \mn{var}(p)})$ to $f_\mn{out}^{M_d}(L)$;
  \item if $\ran(h) \subseteq \Delta^{\Imc^{M_d}_L}$ and $h$ cannot be extended to a homomorphism from $ q|_V \cup p$ to $\Imc_L$,
    then add $(p, h|_{V \cap \mn{var}(p)})$ to $f_\mn{out}^{M_d}(L)$;
  \item for all homomorphisms $g$ from $q|_V \cup E_p$ to $\Imc_L$ that extend $h$
    but cannot be extended to a homomorphism from $q|_V \cup p$ to $\Imc_L$,
    if $\ran(g|_{\overline V \cap \mn{var}(E_p)}) \subseteq \Delta^{\Imc^{M_d}_L}$, 
    then add $(p|_{\overline V}, g|_{\overline V \cap \mn{var}(E_p)})$ to $f_\mn{in}^{M_d}(L)$.
\end{enumerate}

%{\color{blue} 
%Let $M^0_d = M_d$.
%For $M^i_d$, if there is a level $L \in \Abf_\Omc$ with $\Delta^{\Imc_L^{M^i_d}} = \emptyset$, 
%then choose an element $e \in \Delta^{\Imc_L}$, 
%construct $M_e$, and set $M_d^{i+1} = M_d^i \cup M_e$. 
%Let $M_d^*$ be the mosaic at which this process stabilizes.
%{\color{red} rename all the following $M_d$ to $M_d^*$, maybe need to rename $t'$, $\rho'$ etc. occurences}}

Strictly speaking, the tuple $M_d$ defined above is not a mosaic as the domain elements 
are not from the fixed set $\Delta$. Since, however, $|\mu^*(d)| \leq ||\Omc||^{||\Omc||}$, this can 
easily be achieved by renaming. 
If we renamed an element $d \in \Delta^{\Imc}$ to $e \in \Delta^{\Imc^M}$, then we use $e^\downarrow$ to denote $d$.
We might also use this notation for tuples $\bar e$, where $\bar e^\downarrow$ implies that $\cdot^\downarrow$ is applied to all elements in the tuple,
or functions $f$, where $f^\downarrow$ is the function obtained by applying $\cdot^\downarrow$ to all elements in $\ran (f)$.

Now that we have defined how we read mosaics out of the model we first need to prove that they are indeed mosaics 
and second that they don't get eliminated in the mosaic removal subroutine of our algorithm. 
If we proved these two things, then showing that our algorithm must returns `satisfiable' will be trivial.
\begin{lemma}
  \label{lem:alchi_ub_md_mosaic}
  $M_d$ is a mosaic.
\end{lemma}
\begin{proof}
We go through all seven Conditions of mosaics.
Conditions~1 follows directly from $\Imc$ being a model
and the definition of $M_d$.

% Let $L \in \Abf_\Omc$
% and $d \in A^{\Imc^{M_d}_L}$ for some concept
% refinement $L'{:}q(\bar x) \rfn L{:}A$ in $\Omc$.
% Since $\Imc$ is a model of $\Omc$ there is a $\bar e \in q(\Imc_{L'})$
% such that $\rho_{L'}(d) = \bar e$.  
% This implies $\rho^{M_d}_{L'}(d) = \bar e$ and it also implies $\bar e$ is an answer to $q$ on $\Imc^{M_d}_{L'}$ w.r.t.\ $t^{M_d}$ according to
% the definition of $M_d$.  
% Thus $\bar x \mapsto \rho^{M_d}_{L'}(d)$ is a type homomorphism from $q$ to $\Imc_{L'}^{M_d}$.  
% A similar argument can be made for Conditions~4b, 4c and~4d.
% For Conditions~4c and 4d note that $M_d$ also contains elements
% on higher abstraction levels than $L(d)$ if they are in $\mu^*(d)$.

For Condition~2 consider a concept abstraction $L'{:}A \abs L{:}q(\bar x)$ in $\Omc$, 
a non-empty $V \subsetneq \mn{var}(q)$,
and a homomorphism $h$ from $q|_V$ to $\Imc^{M_d}_L$.
There are now two cases.
The first one is that $h^\downarrow$ can be extended to a homomorphism from $q$ to $\Imc_L$.
Then $(p, h|_{V \cap \mn{var}(p)}) \in f_\mn{out}^{M_d}(L)$ for all components $p$ of $q$ w.r.t.\ $V$,
according to Case~1 of the $M_d$ construction.

If $h^\downarrow$ can not be extended to a homomorphism from $q$ to $\Imc_L$ then there has to be
at least one component $p$ of $q$ w.r.t.\ $V$ such that $h^\downarrow$ cannot be extended to a homomorphism from $q|_V \cup p$ to $\Imc_L$.
Case~2 of the $M_d$ construction then yields $(p, h|_{V \cap \mn{var}(p)}) \in f_\mn{out}^{M_d}(L)$.
Thus Condition~2 is satisfied.
Condition~3 can be proven analogously.

For Condition~4 let $(q_0, h_0) \in f_\mn{in}^{M_d}(L)$, $V_0 \subseteq \mn{var}(q_0)$, and $g_0$ be a homomorphism 
from $q_0|_{V_0}$ to $\Imc^{M_d}_L$ that extends $h_0$. 
We have to show that for some component $p_0$ of $q_0$ w.r.t.\ $V_0$ there is $(p_0, g_0|_{V_0 \cap \mn{var}(p_0)}) \in f_\mn{out}^{M_d}$.

Let us assume $(q_0, h_0)$ was added to $f_\mn{in}^{M_d}(L)$ because of Case~3 in the construction of $M_d$.
This implies that there is 
\begin{enumerate}
  \item a concept abstraction $L'{:}A \abs L{:}q(\bar x)$ in $\Omc$,  
  \item a non-empty $V \subsetneq \mn{var}(q)$,
  \item a component $p$ of $q$ w.r.t.\ $V$,
  \item and a homomorphism $g$ from $q|_V \cup E_p$ to $\Imc_L$ that cannot be extended to a homomorphism
  from $q|_V \cup p$ to $\Imc_L$
\end{enumerate}
such that $p|_{\overline V} = q_0^\downarrow$ and $g|_{\overline V \cap \mn{var}(E_p)} = h_0^\downarrow$.

By our assumption $\ran (g_0) \subseteq \Delta^{\Imc_L^{M_d}}$.
The definition of $M_d$ implies that $g_0^\downarrow$ is also a homomorphism from $q|_{V_0}$ to $\Imc_L$.
Point~4 from above yields that there has to be a component $p_0$ of $q_0$ w.r.t.\ $V_0$ such that we cannot extend $g_0^\downarrow$
to a homomorphism from $q|_{V_0} \cup p_0$ to $\Imc_L$.

These are exactly the conditions required for Case~2 in the construction of $M_d$.
Thus by Case~2 we have $(p_0, g_0|_{V_0 \cap \mn{var}(p_0)}) \in f_{\mn{out}}^{M_d}$, as required. 
If $(q_0, h_0)$ was added to $f_\mn{in}^{M_d}(L)$ because of Case~6 then the proof is analogous.
\end{proof}

Now we are ready to prove that all these $M_d$ are contained in the final set of mosaics our algorithm arrives at.
Formally we prove this by doing an induction on the number of iterations in our mosaic removal subroutine.

\begin{lemma}
  \label{lem:alci_abs_all_ar_all_md_in_mi}
  Let $n$ be the number of iterations of our algorithm.
  For every $i \leq n$: $\{M_d \mid d \in \Delta^{\Imc}\} \subseteq \Mmc_i$.  
\end{lemma}
\begin{proof}
  For $i=0$, this follows from Lemma \ref{lem:alchi_ub_md_mosaic}.
  
  For $i > 0$, choose any $M_d \in \Mmc_i$ with $d \in \Delta^\Imc$ and set $M = M_d$.

  Let $d' \in \Delta^{\Imc^M_L}$ with $d' \in A^{\Imc^M_L}$ and $d' \notin (\exists R . B)^{\Imc^M_L}$ for any $A \sqsubseteq_L \exists R . B \in \Omc$.
  We have to show that  there is a mosaic $M' \in \Mmc_i$, 
    an element $e' \in B^{\Imc^{M'}_{L}}$, 
    and an edge candidate $E$ such that $R(d',e') \in E$.

  From the definition of $M_d$ it follows that $d'^\downarrow \in (\exists R.C)^{\Imc_{L}}$ 
  and because $\Imc$ is a model there must be an element $e \in \Delta^{\Imc_{L}}$ such that
  $e \in B^{\Imc_L}$ and $(d'^\downarrow,e) \in R^{\Imc_{L}}$.
  We choose $M_{e}$ as $M'$ and use $e'$ to denote the $e' \in \Delta^{\Imc^{M_e}}$ with $e'^\downarrow = e$.
  Now we construct an edge candidate $E$ for $M_d$ and $M_{e}$ that contains $R(d', e')$:
  \begin{align*}
    E = &\{R(d', e') \mid d' \in \Delta^{\Imc^{M_d}_L} \text{, } e' \in \Delta^{\Imc^{M_{e}}_L} \text{, and } \\
    &(d'^\downarrow,e'^\downarrow) \in R^{\Imc_L} \text{ for any } L \in \Abf_{\Omc} \text{ and } R \in \Rbf\}
  \end{align*}

  What remains to show is that $E$ satisfies all edge candidate conditions.
  Conditions~1 to 2'  and 4 to 5' follow directly from $\Imc$ being a model and the construction of $M_d$, $M_{e}$ and $E$.

  For Condition~3 consider an $L \in \Abf_\Omc$,
  a pair $(q_0, h_0) \in f_\mn{out}^{M_d}(L)$ 
  where $q_0 = E_{q_0} \uplus q_0|_{\overline{V_0}}$ with $V_0 = \mn{dom}(h_0)$,
  and a function $g_0$ from $\overline{V_0} \cap \mn{var}(E_{q_0})$ to $\Delta^{\Imc^{M_e}_L}$
  such that $R(h_0(x), g_0(y)) \in E$ for all $R(x,y) \in E_{q_0}$.
  We need to show that $(q_0|_{\overline{V_0}}, g_0) \in f^{M_e}_\mn{in}(L)$.

  There are six cases in the construction of $M_d$ concerning $f_{\mn{in}}^{M_d}$ and $f_{\mn{out}}^{M_d}$.
  For the following proof we only consider the first three, that is, the ones for concept abstractions.
  Extending the proof to consider all six cases is, however, trivial as Case~4 behaves analogously to Case~1 and likewise for Cases 2 and 5, and 3 and 6.

  If $(q_0, h_0)$ was added to $f_\mn{out}^{M_d}(L)$ because of Case~1, then 
  \begin{enumerate}
    \item $q_0$ is a component of $q$ w.r.t.\ some non-empty $V \subsetneq \mn{var}(q)$ for some concept abstraction $L'{:}A \abs L{:}q(\bar x)$ in $\Omc$;
    \item $h_0^\downarrow = h|_{V \cap \mn{var}(q_0)}$ for some homomorphism $h$ from $q|_V$ to $\Imc_L$;
    \item $h$ can be extended to a homomorphism $h^*$ from $q$ to $\Imc_L$;
    \item $V_0 = V \cap \mn{var}(q_0)$.
  \end{enumerate}
  Point~3 and the construction of $M_d$ imply that $\ran(h^*) \subseteq \{d'^\downarrow \mid d' \in \Delta^{\Imc_L^{M_d}}\}$ and thus by the definition of $M_e$ and refinement functions 
  we cannot extend $h \cup g_0^\downarrow$ to a homomorphism from $q|_V \cup q_0$ to $\Imc_L$ as this would make the elements in $\ran(g_0^\downarrow)$ part of two $L$-ensembles.
  By definition $\ran(g_0) \subseteq \Delta^{\Imc_L^{M_e}}$ and as such Case~3 of $M_d$ construction yields
  $(q_0|_{\overline V}, g_0|_{\overline V \cap \mn{var}(E_{q_0})}) \in f_\mn{in}^{M_e}$.
  Point~1 and Point~4 imply that $\overline{V_0} \cap \mn{var}(E_{q_0}) \subseteq \overline V \cap \mn{var}(E_{q_0})$ and since the left side of the subset relation
  is exactly the domain of $g_0$ we can simplify the previous expression to $(q_0|_{\overline V}, g_0) \in f_\mn{in}^{M_e}$.
  Again we can use Point~4 to arrive at the desired result of $(q_0|_{\overline{V_0}}, g_0) \in f_\mn{in}^{M_e}$.

  If $(q_0, h_0)$ was added to $f_\mn{out}^{M_d}(L)$ because of Case~2, then 
  \begin{enumerate}
    \item $q_0$ is a component of $q$ w.r.t.\ some non-empty $V \subsetneq \mn{var}(q)$ for some concept abstraction $L'{:}A \abs L{:}q(\bar x)$ in $\Omc$;
    \item $h_0^\downarrow = h|_{V \cap \mn{var}(q_0)}$ for some homomorphism $h$ from $q|_V$ to $\Imc_L$;
    \item $V_0 = V \cap \mn{var}(q_0)$.
  \end{enumerate}
  This time Case~2 of $M_d$ construction itself already implies that we cannot extend $h$ (and the same must hold true for $h \cup g_0^\downarrow$) to a homomorphism from $q|_V \cup q_0$ to $\Imc_L$.
  By definition $\ran(g_0) \subseteq \Delta^{\Imc_L^{M_e}}$ and as such Case~3 of $M_e$ construction yields
  $(q_0|_{\overline V}, g_0|_{\overline V}) \in f_\mn{in}^{M_e}$.
  We can apply the same reasoning as for Case~1 of $M_e$ construction to arrive at $(q_0|_{\overline{V_0}}, g_0) \in f_\mn{in}^{M_e}$, as required.

  Condition~3' of edge candidates can be proven in an analogous way and consequently 
  $E$ is an edge candidate for $M_d$ and $M_e$.
  That $M_e \in \Mmc_i$ follows from the induction hypothesis and thus $M_d$ is good in $\Mmc_i$ for all $d \in \Delta^{\Imc}$.
\end{proof}
By our initial assumtion there is at least one element $d \in \Delta^{\Imc_L}$ for all $L \in \Abf_\Omc$,
and also an element $d_0 \in {C_0}^{\Imc_{{L_0}}}$.
We can now use Lemma~\ref{lem:alci_abs_all_ar_all_md_in_mi} to imply that $M_{\{d\}} \in \Mmc^*$ for all these domain element
and thus the conditions for our algorithm to return `satisfiable' are met.

\subsection{Running Time}
Since the domain of a mosaic has a maximum of $||\Omc||^{||\Omc||}$ elements, there are at most  $2^{2^{{\mn{poly}(||\Omc||)}}}$ different mosaics.
Consequently the length of the constructed
sequence $\Mmc_0,\Mmc_1,\dots$ is also bounded by
$2^{2^{{\mn{poly}(||\Omc||)}}}$.
Checking goodness of mosaics requires us to compute answers to CQs. 
Doing this in a brute force way results in checking at most $||\Omc||^{||\Omc||^{||\Omc||}} \in 2^{2^{\mn{poly}(||\Omc||)}}$ candidates for homomorphisms from a CQ into a mosaic (for a fixed answer).
In summary, the running time of our algorithm is at most $2^{2^{{\mn{poly}(||\Omc||)}}}$.
% Our algorithm constructs all mosaics so we start by analyzing the size of a mosaic,
% that is the number of elements in all components of the mosaic.
% For simplicity we consider the collection of interpretations as one big interpretation.
% The domain has a maximum of $|\Omc|^{|\Omc|}$ elements which can be estimated by $2^{\mn{poly}(|\Omc|)}$,
% where $\mn{poly}$ is an unspecified polynomial.
% Adding concept names and role names does not exceed this bound: $(2^{\mn{poly}(|\Omc|)})^2 \cdot |\Omc|^2 \in 2^{\mn{poly}(|\Omc|)}$.
% Similar for the refinement function and type assignments: $2^{\mn{poly}(|\Omc|)} \cdot |\Omc|^2 \in 2^{\mn{poly}(|\Omc|)}$.
% This results in at most $2^{2^{\mn{poly}(|\Omc|)}}$ different mosaics.

%%% ---------------------------- LOWER BOUNDS below ----------------------------------------------%%%
%%% ---------------------------- LOWER BOUNDS below ----------------------------------------------%%%
%%% ---------------------------- LOWER BOUNDS below ----------------------------------------------%%%
%%% ---------------------------- LOWER BOUNDS below ----------------------------------------------%%%
%%% ---------------------------- LOWER BOUNDS below ----------------------------------------------%%%
%%% ---------------------------- LOWER BOUNDS below ----------------------------------------------%%%
%%% ---------------------------- LOWER BOUNDS below ----------------------------------------------%%%
%%% ---------------------------- LOWER BOUNDS below ----------------------------------------------%%%

\section{Proofs for Section~\ref{sec:alc_ra_2exp_lb}}

We provide the detailed proof of Theorem~\ref{thm:firsthard}, which we
repeat here for the reader's convenience. 
\thmfirsthard*
Let $\Omc, A_0, q$ be the input for simple CQ evaluation. We construct
an $\ALC^\mn{abs}[\textnormal{ra}]$ ontology $\Omc'$ with three abstraction
levels $L_1 \prec L_2 \prec L_3$ such that $A_0$ is $L_1$-satisfiable
w.r.t.\ $\Omc'$ iff \mbox{$\Omc, A_0 \not \models q$}.

We may assume w.l.o.g.\ that the \ALCI-ontology $\Omc$ is in
\emph{normal form}, meaning that every CI in it is of the form
$$
A_1 \sqcap \dots \sqcap A_n \sqsubseteq B_1 \sqcup \dots \sqcup B_m
\quad \text{ or } \quad
A \sqsubseteq C
$$
where $A_i,B_i,A$ are concept names and $C$ ranges over concepts of
the form $\exists R . B$, $\forall R . B$, or $\neg B$ with $R$ a
(potentially inverse) role and $B$ a concept name. It is routine to
show that every \ALCI-ontology \Omc can be converted in polynomial
time into an \ALCI-ontology $\Omc'$ in this form that is a
conservative extension of \Omc.

% \begin{align*}
%   A_1 \sqcap \dots \sqcap A_n &\sqsubseteq B_1 \sqcup \dots \sqcup B_m\\
%   A &\sqsubseteq \exists r.B \qquad &A &\sqsubseteq \exists r^-.B \\
%   A &\sqsubseteq \forall r.B \qquad &A &\sqsubseteq \forall r^-.B\\
%   A &\sqsubseteq \neg B
% \end{align*}
% where $A, B, A_i$ and $B_i$ are concept names.
% This is w.l.o.g.\ because
% {\color{blue} cite something, Grau et al. ``Acyclicity Conditions and Their Application to
% Query Answering in Description Logics'' is close but not quite. Otherwise argue why its true.}
% of the form $\top \sqsubseteq C_{\Omc}$ and $C_{\Omc}$ is in NNF.
We next show that CIs of the form $A \sqsubseteq \forall r^-.B$ can be
removed. Let $\widehat \Omc$ be obtained from \Omc by
replacing %doing the following for
every such CI 
% $\Omc$:
with 
%\begin{itemize}
  % \item replace every occurence of $\forall r^-.B$ in $\widehat \Omc$ with a fresh concept name $\overline B$ and
  % \item add the CIs
$$\top \sqsubseteq B \sqcup \overline B,\ \overline B \sqsubseteq \forall r.\overline A,
\text{  and } \overline A \sqsubseteq \neg A$$ where $\overline B$ and
$\overline A$ are fresh concept names. It is routine to prove the
following.
%\end{itemize}
% (the three CIs in point two are the normalisation of $\neg B \sqsubseteq \forall r.\neg \overline B$)
\begin{restatable}{lemma}{lemalcabsracons}
  \label{lem:alc_abs_ra_cons}
  $\Omc, A \models q$ iff $\widehat \Omc, A \models q$.
\end{restatable}
\begin{proof}
  We prove this by showing that $\widehat \Omc$ is a conservative extension of $\Omc$.
  Point $1$ of conservative extensions is straightforward as 
  $\mn{sig}(\widehat \Omc) = \mn{sig}(\Omc) \cup \{\overline B, \overline A \mid A \sqsubseteq \forall r^-.B \in \Omc\}$.
  For Point $2$ consider a model $\Imc$ of $\widehat \Omc$ and domain element
  $d \in A^\Imc$ with $A \sqsubseteq \forall r^-.B \in \Omc$.
  The three CIs we added when replacing $A \sqsubseteq \forall r^-.B$ imply that there is no
  $e \in \Delta^\Imc$ with both $e \in (\neg B)^\Imc$ and $(e,d) \in r^\Imc$.
  Thus $d \in (\forall r^-.D)^\Imc$ and consequently $\Imc$ is a model of $\Omc$.

  For the third point we construct a model $\Imc_2$ of $\widehat \Omc$  
  given a model $\Imc_1$ of $\Omc$ as follows:
  \begin{align*}
    \Delta^{\Imc_2} &= \Delta^{\Imc_1}\\
    A^{\Imc_2} &= A^{\Imc_1}               \text{ for all $A \in \mn{sig}(\Omc) \cap \Cbf$}\\
    \overline A^{\Imc_2} &= \{d \mid d \in \Delta^{\Imc_2} \setminus A^{\Imc_2}\} \text{ for all } A \sqsubseteq \forall r^-.B \in \Omc \\
    \overline{B}^{\Imc_2} &= \{d \mid d \in \Delta^{\Imc_2} \setminus B^{\Imc_2}\} \text{ for all } A \sqsubseteq \forall r^-.B \in \Omc \\
    r^{\Imc_2} &= r^{\Imc_1} \text{ for all $r \in \mn{sig}(\Omc) \cap \Rbf$}
  \end{align*}
  We are left arguing that $\Imc_2$ is a model of $\widehat \Omc$.

  It is easy to see that the extensions of concept and role names from $\mn{sig}(\Omc)$ coincide in the two 
  interpretations and hence the CIs that simply got copied from $\Omc$ to $\widehat \Omc$ are all still satisfied.
  Let $\overline A$ and $\overline{B}$ be the concept names used in one replacement step for a CI $A \sqsubseteq \forall r^-.B \in \Omc$.
  It is clear by the construction of ${\overline A}^{\Imc_2}$ and $\overline{B}^{\Imc_2}$ that 
  the CIs 
  $\top \sqsubseteq B \sqcup \overline{B},
  \text{ and } \overline A \sqsubseteq \neg A$ are satisfied.

  For $\overline{B} \sqsubseteq \forall r.\overline A$, we argue by contradiction.
  Assume there is a $d \in \overline{B}^{\Imc_2}$, $e \in (\neg \overline A)^{\Imc_2}$, and $(d,e) \in r$.
  By the construction of $\overline A^{\Imc_2}$ we know $e \in A^{\Imc_2}$.
  Hence in the original model there are now $d \in (\neg B)^{\Imc_1}$, $e \in A^{\Imc_1}$ and $(e,d) \in r^{\Imc_1}$.
  This contradicts the assumption that $\Imc_1$ is a model, since the CI $A \sqsubseteq \forall r^-.B \in \Omc$ is now no longer satisfied.
\end{proof}
We may thus assume that \Omc contains no CIs of the form
$A \sqsubseteq \forall r^-.B$. We now convert \Omc into an
\ALC-ontology $\Omc_\ALC$. Introduce a fresh role name $\widehat r$
for every role name $r$ used in \Omc. Then start from \Omc, replace
every CI $A \sqsubseteq \exists r^-.B$ with
$ A \sqsubseteq \exists \widehat r . B$, and add
$\exists \widehat r.A \sqsubseteq B$ for every CI
$A \sqsubseteq \forall r.B$ in $\Omc$.\footnote{It is not important to
  achieve normal form here.}  Clearly, representing an inverse role
$r^-$ with a fresh role name $r$ is problematic from the perspective
of the CQ~$q$. 

We now assemble the $\ALC^\mn{abs}[\textnormal{ra}]$-ontology $\Omc'$.
On level~$L_1$, we want a model of $\Omc_\ALC$ and thus include
$C \sqsubseteq_{L_1} D$ for all $C \sqsubseteq D$ in $\Omc_\ALC$.
The interpretation on level $L_2$ is a copy of that on level $L_1$,
but we add an $r$-edge for every $\widehat r^-$-edge. This is achieved
by adding the role abstractions
$$
  L_2{:} r \abs L_1{:} r(x,y) \ \  \text{ and } \ \   L_2{:} r \abs L_1{:}
  \widehat r(y,x). 
$$
We also need to copy the concept names but do not want to use concept
abstractions. We thus introduce a fresh role name $r_A$ for every
concept name $A$ and put $A \equiv_{L_i} \exists r_A  . \top$ for
$i \in \{1,2\}$ and
$$
  L_2{:} r_A \abs L_1{:} r_A(x,y)
$$
for all concept names $A$ in $\Omc_\ALC$. Let $\widehat q(\bar x,z)$
be obtained from $q$ by dequantifying all variables and adding the
atom $\top(z)$ for a fresh variable $z$. Furthermore, let $r_q$ be a fresh
role name and add the following statements to $\Omc'$:
$$
\begin{array}{l}
  L_3{:}r_q  \abs L_1{:}\widehat q(\bar x, z)\\[1mm]
  \exists r_q.\top \sqsubseteq_{L_3} \bot
\end{array}
$$

This concludes the construction of $\Omc'$ and now we will prove that we can use it to reduce from the complement of the simple CQ evaluation problem.
\begin{restatable}{lemma}{lemfirsthardcorr}
  \label{lem:firsthardcorr}
  $\Omc, A_0 \not \models q$ iff $A_0$ is $L_1$-satisfiable w.r.t.\ $\Omc'$.
\end{restatable}

%\lemalcabsracons*

To prepare for the proof of Lemma~\ref{lem:firsthardcorr}, we first
establish an auxiliary lemma. For interpretations $\Imc, \Jmc$, we write
$\Imc \subseteq \Jmc$ if $\Imc$ can be found in
$\Jmc$, that is, if the following conditions are satisfied:
\begin{align*}
  \Delta^\Imc &\subseteq \Delta^\Jmc\text{,}\\
  A^\Imc &\subseteq A^\Jmc \text{ for all concept names $A$, and}\\
  r^\Imc &\subseteq r^\Jmc \text{ for all role names $r$.}
\end{align*}
Similarly, for a signature $\Sigma$ we use $\Imc_{|\Sigma}$ to denote
the restriction of $\Imc$ to the concept and role names in $\Sigma$.
Let $\Omc'^-$ be the fragment of $\Omc'$ in which all statements
have been removed that refer to level $L_3$. Then we have the following.
\begin{lemma}
  \label{lem:alc_abs_ra_equal}
  \begin{enumerate}
    \item[]
    \item If $\Imc$ is a model of $\Omc$,
    then there is a model $\Jmc$ of $\Omc'^-$ such that $\Jmc_{L_2|\Sigma} = \Imc$ with $\Sigma = \mn{sig}(\Omc)$;
    \item If $\Jmc$ is a model of $\Omc'^-$, 
    then there is a model $\Imc$ of $\Omc$ such that $\Imc \subseteq \Jmc_{L_2}$.
  \end{enumerate}
\end{lemma}
\begin{proof}
  ``Point~$1$''.
  Let $\Imc$ be a model of $\Omc$.
  We construct $\Jmc = (\prec, \Jmc_{L_1}, \Jmc_{L_2}, \rho)$ as follows:
  \begin{align*}
    \Delta^{\Jmc_{L_1}} &= \Delta^{\Imc}                          &&\\
    A^{\Jmc_{L_1}} &= A^{\Imc}                                    &&\text{for all $A \in \mn{sig}(\Omc) \cap \Cbf$}\\
    r_A^{\Jmc_{L_1}} &= \{(d,d) \mid d \in A^\Imc\}               &&\text{for all $A \in \mn{sig}(\Omc) \cap \Cbf$}\\
    r^{\Jmc_{L_1}} &= r^{\Imc}                                    &&\text{for all $r \in \mn{sig}(\Omc) \cap \Rbf$}\\
    {\widehat r}^{\Jmc_{L_1}} &= \{(e,d) \mid (d,e) \in r^\Imc\}  &&\text{for all $r \in \mn{sig}(\Omc) \cap \Rbf$}
  \end{align*}
  We let $\Jmc_{L_2}$ be an isomorphic copy of $\Jmc_{L_1}$ and let $\rho$ map each domain element in $\Delta^{\Jmc_{L_2}}$ to its isomorphic copy 
  in $\Delta^{\Jmc_{L_1}}$. 
  Now we prove that $\Jmc$ is a model of $\Omc'^-$.

  Let us first consider the labeled CIs in $\Jmc_{L_1}$.
  All the CIs that didn't change going from $\Omc$ to $\Omc_\ALC$ are still satisfied.
  For each CI of the form $A \sqsubseteq \exists r^-.B$ in $\Omc$, we introduced a CI 
  $A \sqsubseteq_{L_1} \exists \widehat r.B$ in $\Omc'^-$. It is trivial to verify that our new CI is satisfied by the construction of ${\widehat r}^{\Jmc_{L_1}}$.

  For each CI of the form $A \sqsubseteq \forall r.B$ in $\Omc$, we added a CI $\exists {\widehat r}.A \sqsubseteq_{L_1} B$  to $\Omc'^-$.
  Assume we have a domain element $d \in (\exists {\widehat r}.A)^{\Jmc_{L_1}}$.
  Our construction of $\widehat{r}^{\Jmc_{L_1}}$ then implies that there is an $e \in A^{\Jmc_{L_1}}$ with $(e,d)\in r^{\Jmc_{L_1}}$.
  Now we can again use the construction of $\Jmc$ to jump back to the model $\Imc$ and obtain $e \in (\forall r.B)^{\Imc}$ and $d \in B^\Imc$. 
  Thus also $d \in B^{\Jmc_{L_1}}$, as required.

  The rest of the labeled CIs are all of the form $A \equiv_{L_1} \exists r_a.\top$ and it is clear that they are satisfied by our construction.
  Since $\Jmc_{L_2}$ is an isomorphic copy we can argue for the $L_2$-CIs in exactly the same way.
  Again $\Jmc_{L_2}$ being an isomorphic copy and the way we defined $\rho$ makes it easy to see that the two types of
  role abstractions are satisfied as well. Note that we already defined ${\widehat r}$ to point in the inverse direction of $r$ 
  for any two domain elements that are connected by $r$.
  Looking at the construction it is straightforward to see that $\Jmc_{L_2|\Sigma} = \Imc$ with $\Sigma = \mn{sig}(\Omc)$.

  ``Point~$2$''. Let $\Jmc$ be a model of $\Omc'^-$.
  We construct $\Imc$ as follows:
  \begin{align*}
    \Delta^\Imc &= \{d \mid d \in \Delta^{\Jmc_{L_1}} \text{ and $d$ is in the extension of} \\
    &\text{a concept or role name in $\Jmc_{L_1}$}\}\\
    A^{\Imc} &= A^{\Jmc_{L_1}} \qquad \text{ for all $A \in \mn{sig}(\Omc) \cap \Cbf$}\\
    r^{\Imc} &= r^{\Jmc_{L_1}} \cup \{(d,e) \mid (e,d) \in \widehat r^{\Jmc_{L_1}}\}\\
    &\text{for all $r \in \mn{sig}(\Omc) \cap \Rbf$}
  \end{align*}

  All the CIs in $\Omc$ that do not contain any universal or existential restriction are obviously still satisfied since we copied them to $\Omc'$.
  Same for CIs of the form $A \sqsubseteq \exists r.B$.
  Since $\Omc$ is normalized we then only need to consider CIs of the form $A \sqsubseteq \forall r.B$ and $A \sqsubseteq \exists r^-.B$.

  Consider a CI of the form $A \sqsubseteq \exists r^-.B \in \Omc$ and $d \in A^\Imc$.
  Since there is a CI $A \sqsubseteq \exists \widehat r.B \in \Omc'$ and $d \in A^{\Jmc_{L_1}}$ there is an $e \in B^{\Jmc_{L_1}}$ with 
  $(d,e) \in \widehat r^{\Jmc_{L_1}}$ and thus by our construction of $\Imc$ also $(e,d) \in r^\Imc$ which satisfies the mentioned CI.

  On the other hand let there be a CI $A \sqsubseteq \forall r.B \in \Omc$, $d \in A^\Imc$ and $(d,e) \in r^\Imc$ an edge such that
   $(e,d) \in \widehat r^{\Jmc_{L_1}}$.
  By our construction $d \in A^{\Jmc_{L_1}}$ and therefore $e \in (\exists \widehat r.A)^{\Jmc_{L_1}}$.
  The CI $\exists \widehat r.A \sqsubseteq_{L_1} B \in \Omc'$ then gives us $e \in B^{\Jmc_{L_1}}$ and consequently $e \in B^{\Imc}$.
  Thus $\Imc$ is a model of $\Omc$.

  To show $\Imc \subseteq \Jmc_{L_2}$ we construct an isomorphic copy $\Imc'$ of $\Imc$ with $\Imc' \subseteq \Jmc_{L_2}$ 
  that is by nature still a model of $\Omc$.
  We do this by replacing each domain element $d \in \Delta^{\Imc}$ (implies $d \in \Delta^{\Jmc_{L_1}}$) with its abstraction,
  that is, the $e \in \Delta^{\Jmc_{L_2}}$ with $\rho(e) = d$.
  This works because the role abstractions $L_2{:}r \abs L_1(r(x,y))$ force $\rho$ to be such a one to one mapping for
  all domain elements in the extension of a role name or concept name (concept names because of the CIs $A \equiv_{L_1} \exists r_A.\top$).

  Now, obviously, $\Delta^{\Imc'} \subseteq \Delta^{\Jmc_{L_2}}$.
  The only interesting case for the subset relation between all the concept and role name extensions 
  are the roles we added with $\{(d,e) \mid (e,d) \in \widehat r^{\Jmc_{L_1}}\}$.
  Let $(d,e) \in r^{\Imc'}$ be such a role.
  The construction of $r^{\Imc'}$ implies that $(\rho(e), \rho(d)) \in \widehat r^{\Jmc_{L_1}}$ and 
  with the role abstraction $L_2{:}r \abs L_1{:}\widehat r(y,x) \in \Omc'$ we get $(d,e) \in r^{\Jmc_{L_2}}$.
  Thus $\Imc' \subseteq \Jmc_{L_2}$.
\end{proof}

Now Lemma~\ref{lem:firsthardcorr} is a direct consequence of
Lemma~\ref{lem:alc_abs_ra_equal} and the statements in
$\Omc' \setminus \Omc'^-$.

% Now we construct an ontology $\Omc_{\mn{abs}}'$ that extends $\Omc_{\mn{abs}}$ to check whether $q$ matches into $L_2$ 
% and make the ontology unsatisifiable if that is the case.
% For this purpose let $\widehat q(\bar x)$ be, as in the previous section, the CQ obtained from $q$ by dequantifying all variables,
% $z$ a fresh variable, and $r_q$ a fresh role name.
% %
% \begin{align*}
%   L_3{:}r_q  &\abs L_1{:}\widehat q(\bar x, z)\\
%   L_3 {:}\, \exists r_q.\top &\sqsubseteq \bot
% \end{align*}

% \begin{lemma}
%   $\Omc, A \not \models q$ iff $L_1{:} A$ is satisfiable w.r.t.\ $\Omc_{\mn{abs}}'$.
% \end{lemma}
% \begin{proof}
%   The ``$\Rightarrow$'' direction follows directly from Lemma~\ref{lem:alc_abs_ra_cons}, Point~1 of Lemma~\ref{lem:alc_abs_ra_equal}, and the two statements we added to construct $\Omc'_{\mn{abs}}$.
%   We can prove the ``$\Leftarrow$'' analogously, just that we use Point~2 of Lemma~\ref{lem:alc_abs_ra_equal}.
% \end{proof}
\section{Proofs for Section~\ref{sec:alc_ca_2exp_lb}}

% {\color{red} reference something that shows tree models are enough? 
% Also need to reference the connection between $\ALC^\mn{sym}$ and 2ExpTime lower bound.}
%
% \begin{lemma}
%   \label{lem:alc_ca_lb}
%   For every $\ALC^\mn{sym}$ ontology $\Omc$, we can construct an $\ALC[ca]$ ontology $\Omc'$ with two abstraction levels $L \prec L'$ such that
%   $\{\Imc \mid \Imc \text{ is a tree model of } \Omc\}
%   = \{\Imc_L|_{\mn{sig}(\Omc)} \mid \Imc \text{ is a model of  } \Omc' \text{ and } \Imc_L \text{ is a tree interpretation}\}$.
%   % \sout{$\{\Imc \mid \Imc \text{ is tree interpretation and } \Imc \models \Omc\}
%   % = \{\Imc_L|_{\mn{sig}(\Omc)} \mid \Imc \text{ is A-tree interpretation and } \Imc \models \Omc'\}$}
% \end{lemma}
%

Our aim is to prove Theorem~\ref{thm:secondhard}. We again repeat the
theorem for the reader's convenience.

\thmsecondhard*

Recall that
$\ALC^\mn{sym}$ is \ALC with a single role name $s$ that must be
interpreted as a reflexive and symmetric relation. We reduce simple CQ
evaluation in $\ALC^\mn{sym}$. 

Let $\Omc$ be an $\ALC^\mn{sym}$
ontology, $A_0$ a concept name, and $q$ a Boolean CQ such that it is
to be decided whether \mbox{$\Omc,A_0 \models q$}. We assume w.l.o.g.\ that \Omc
is in \emph{negation normal form} (NNF), that is, negation is only applied to concept names, but not to compound concepts.
For an $\ALC^\mn{sym}$ concept $C$ we use $\overline{C}$ to denote the result of converting $\neg C$ to NNF.
With $\mn{cl}(\Omc)$ we denote the smallest set that contains all concepts used in $\Omc$ (possibly inside a CQ) 
and is closed under subconcepts and under $\overline{\, \cdot\,}$.
Let $\exists s.C_1, \dots, \exists s.C_{n}$ be all
concepts in $\mn{cl}(\Omc)$ that quantify existentially over $s$.

Let $C_0 = A_0 \sqcap (E^0 \sqcup E^1)$.
We construct (in polynomial time) an $\ALC[\textnormal{ca}]$
ontology~$\Omc'$ such that $\Omc,A_0 \not \models q$ iff $C_0$ is $L$-satisfiable w.r.t.\ $\Omc'$. 
In $\Omc'$, we represent the role name $s$ by the
composition $r^-;r$ where $r$ is a normal (neither reflexive nor
symmetric) role name. We use all concept names from $\mn{sig}(\Omc)$ as
well as
 \begin{itemize}
  \item a concept name $A_C$ for every $C \in \mn{cl}(\Omc)$; % that say which $\ALC^\mn{sym}$ concept should be satisfied,
  \item concept names $E^0, E^1$ that represent endpoints of symmetric edges; 
  % we alternate between $E^0$-endpoints and $E^1$-endpoints for technical purposes,
  \item concept names $N^i_j$ for all $i \in \{0,1\}$ and $j \in \{0, \dots, n\}$
   where $j$ denotes the number of $s$-children of the current node
   (in a tree-shaped model),
 \item concept names $M^{i,j}_{k}$ for all $i \in \{0,1\}$,
   $j \in \{1, \dots, n\}$, and $k \in \{1, \dots, j\}$ that represent midpoints of the
   composition $r^-;r$;
   % symmetric edges; for example $M^{0,3}_2$ will be the $2$nd midpoint
   % for an $N_3^0$ element,
 %  \item a fresh concept name $E$ that represents any endpoints,
 \item auxiliary role names $\widehat r$ an $u$.
  %  that points in the inverse direction of  $r$, and
  % \item a role name $u$ that connect the midpoints.
\end{itemize}

% We need some notion of how we transform a model of $\Omc'$ to a model of $\Omc$.
% Basically we only care about the endpoints, restrict ourselves to the signature of $\Omc$, 
% and put the symmetric role wherever we have a $r^-;r$ role composition.

% Now we are ready to state the main lemma.
% %
% \begin{lemma}
%   \label{lem:alc_ca_lb_main}
%   For every $\ALC^\mn{sym}$ ontology $\Omc$, one can construct in
%   polynomial time an $\ALC[\textnormal{ca}]$ ontology $\Omc'$ with two
%   abstraction levels $L \prec L'$ such that
%   %
%   \begin{enumerate}

%   \item 

%   \end{enumerate}
  
%   $\{\Imc \mid \Imc \text{ is a rooted tree-shaped normal model of }
%   \Omc\}
%   = \{\mn{sym}_\Omc(\Imc_L) \mid \Imc \text{ is model of } \Omc' \text{ and } \mn{sym}_\Omc(\Imc_L) \text{ is rooted}\\
%   \text{tree-shaped normal w.r.t.\ $\Omc$}\}$.
% \end{lemma}

W.l.o.g.\ we assume that $A_0 \in \mn{sub}(\Omc)$. For every concept $C \in \mn{cl}(\Omc)$, $\Omc'$ contains CIs that
axiomatize the semantics of the corresponding concept names~$A_C$:
 % We begin by explaining which CIs we add to $\Omc'$ and then we construct concept abstractions that will force the symmetric structure of $r$ (see Figure~\ref{fig:ca_gadget}).
 % For the $A_C$ concept names we add CIs that give them their intended semantics in $\Omc'$.
 \begin{align}
   A_{B}            &\equiv_L B               &&\text{ for all $B \in 
                                                 \mn{cl}(\Omc) \cap \Cbf$}\\
   A_{\overline C}  &\equiv_L \neg A_C        &&\text{ for all $C \in \mn{cl}(\Omc)$}\\
   A_{C \sqcap D}   &\equiv_L A_C \sqcap A_D  &&\text{ for all $C \sqcap D \in \mn{cl}(\Omc)$}\\
   A_{C \sqcup D}   &\equiv_L A_C \sqcup A_D  &&\text{ for all $C \sqcup D \in \mn{cl}(\Omc)$}\\
   A_C              &\sqsubseteq_L A_D        &&\text{ for all $C \sqsubseteq D \in \Omc$}.   
 \end{align}
% where $B$ is a concept name and $C,D$ are $\ALC^\mn{sym}$ concepts.

%  {\color{blue} Every element that satisfies an $A_C$ concept is an endpoint of a symmetric edge.
%  For all $C \in \mn{sub}(\Omc)$ we add the following CI to $\Omc'$.
%  \begin{align*}
%     A_C \sqsubseteq E^0 \sqcup E^1
%  \end{align*}
%  }
 %
 At endpoints of symmetric edges, we guess the number of children and
 introduce corresponding midpoints and endpoints.  Note that each
 midpoint has two $r$-successor endpoints and we will later merge one
 of them with the $\widehat r$-predecessor.  For all $i \in \{0,1\}$,
 $j \in \{1, \dots, n\}$, and $k \in \{1, \dots, j\}$ we add the
 following CIs to $\Omc'$:
 \begin{align}
  E^i &\sqsubseteq_L N^i_0 \sqcup \dots \sqcup N^i_n \\
  N^i_j &\sqsubseteq_L \bigsqcap_{1 \leq k \leq j} \exists {\widehat r}.M^{i,j}_k\\      
  M^{i,j}_k &\sqsubseteq_L \exists r.E^0 \sqcap \exists r.E^1          
 \end{align}
For an existential restriction $\exists s.C$,
there has to be a child that satisfies $C$ or the element itself satisfies $C$ (reflexivity).
Similarly, if we have a universal restriction $\forall s.C$ then $C$ is satisfied in all children and the element itself.
For all $i \in \{0,1\}$ and $j \in \{0, \dots, n\}$ we add the following CIs to $\Omc'$:
 \begin{align}
  A_{\exists s.C} \sqcap N^i_j &\sqsubseteq_L A_C\, \sqcup  \bigsqcup_{1 \leq k \leq j} \forall {\widehat r}.(M^{i,j}_k \rightarrow \notag\\
   &\forall r.(E^{1-i} \rightarrow A_C)) \\
  A_{\forall s.C} &\sqsubseteq_L A_C\\
   \exists r.A_{\forall s.C} &\sqsubseteq_L \forall r.A_C
 \end{align}

 \begin{figure}
  \centering
  \includegraphics[width=0.6\columnwidth]{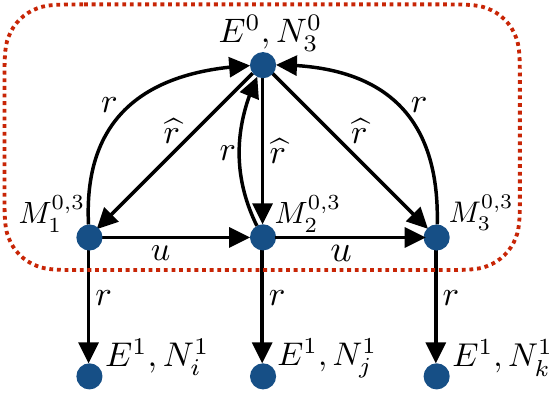}
  \caption{Example of the symmetric structure of $r$ between endpoints ($E^i$). 
  Inside the dotted line is an $L$-ensemble for the case of three children ($N^0_3$).}
  \label{fig:ca_gadget}
 \end{figure}

 For utility purposes (because CQs in abstractions have to be connected), we connect successive midpoints with the $u$-role.
 For all $i \in \{0,1\}$, $j \in \{2, \dots, n\}$, and $k \in \{1, \dots, j-1\}$ we add the following CI to $\Omc'$:
\begin{align}
  M^{i,j}_k &\sqsubseteq_L \exists u.M^{i,j}_{k+1}
\end{align}
 Now we construct the concept abstractions that will ``create'' $L$-ensembles as depicted in Figure~\ref{fig:ca_gadget}.
 First, we add concept abstractions such that the CQ in it matches onto $L$-ensembles consisting of 
 one $N^i_j$ element and the $j$ corresponding $\widehat r$-successors (midpoints).
For all $i \in \{0,1\}$, $j \in \{1, \dots, n\}$, and $k \in \{1, \dots, j\}$ we add the following concept abstraction to $\Omc'$:
 \begin{align}
   L'{:}\top &\abs L{:} N^i_j(x_0) \land \bigwedge_{1 \leq k \leq j} \widehat r(x_0, x_k) \land M^{i,j}_k(x_j)
 \end{align}
 Intuitively for a midpoint $M^{i,j}_k$ we want the $r$-successor satisfying $E^i$ to match back onto the $\widehat r$-predecessor also satisfying $E^i$.
 For this purpose, the CQs in the concept abstractions below
 contain for each midpoint an $r$-edge from the midpoint to the mentioned endpoint.
 As the concept abstractions below match onto the same $L$-ensembles as the ones above and partial overlaps are forbidden by the refinement function semantics, 
 we achieve the desired structure.
 
 For each $i \in \{0,1\}$ and $j \in \{1, \dots, n\}$ we add the following $j$ concept abstractions to $\Omc'$:
 \begin{align}
  L'{:}\top &\abs L{:} E^i(x_0) \land r(x_1, x_0) \land \notag\\
   &\bigwedge_{1 \leq k < j} M^{i,j}_k(x_k) \land u(x_k, x_{k+1}) \land M^{i,j}_{k+1}(x_{k+1}) \notag\\
   & \qquad \vdots \notag\\
   L'{:}\top &\abs L{:} E^i(x_0) \land r(x_j, x_0) \land \notag\\
   &\bigwedge_{1 \leq k < j} M^{i,j}_k(x_k) \land u(x_k, x_{k+1}) \land M^{i,j}_{k+1}(x_{k+1})
\end{align}

Let $\widehat q$ be obtained from $q$ by first replacing every concept atom $C(x) \in q$ with $A_C(x)$,
then adding $D(x)$ with $D = E^0 \sqcup E^1$ for every variable $x \in \mn{var}(q)$, and
lastly replacing every role atom $s(x,y) \in q$ with $r(z,x) \land r(z,y)$, where $z$ is a fresh variable.
Now we add a concept abstraction that checks whether the query matches into level $L$:
\begin{align}
  L'{:}\bot \abs L{:}\widehat q
\end{align}

This concludes the construction of $\Omc'$ and now we will prove that we can use it to reduce from the complement of the simple CQ evaluation problem.

\begin{lemma}
  \label{lem:5.2equivalence}
  $\Omc,A_0 \not \models q$ iff $C_0$ is $L$-satisfiable w.r.t.\ $\Omc'$.
\end{lemma}
We split the proof of Lemma~\ref{lem:5.2equivalence} into a soundness part
(``if'' direction) and a completeness part (``only if'' direction).

\subsection{Soundness}
%%%%%%%%%%%%%%%%%%%%%%%%%%%%%%%%%%%%%%%%%%%%%%%%%%%%%%%%%%%%%%%%%%%%%%%%%%%%%%%%%%%%%%%%%%%%%%%%%%%%%%%%%%%%%%%%%%%%%%%%%%%%%%%%%%%%%%%%%%%%%%%%%%%%%%%%%%%%%%%%
 Assume that
  $\Omc,A_0 \not\models q$. Then there exists a model \Imc of \Omc
  with $A_0^\Imc \neq \emptyset$ and $\Imc\not\models q$. Using a
  standard unraveling construction, we can obtain from \Imc a
  tree-shaped interpretation \Jmc with $A_0^\Jmc \neq \emptyset$ and
  $\Jmc\not\models q$, see for example \cite{DBLP:conf/cade/Lutz08}.
  Here, an interpretation $\Imc$ (of $\ALC^\mn{sym}$) is
  \emph{tree-shaped} if the undirected graph $G_\Imc = (V,E)$ with
  $V = \Delta^\Imc$ and
  $E = \{\{d,e\} \mid (d,e) \in s^\Imc \text{ and } d \neq e\}$ is a
  tree. We may choose an element as the root and then speak about
  children and parents.

  We may thus assume that \Imc itself is tree-shaped. By adding
  subtrees, we may easily achieve that if
  $d \in (\exists s . C_i)^\Imc$, then there is a \emph{child} $e$ of
  $d$ such that $d \in C_i^\Imc$, that is, existential restrictions
  are always satisfied in children in the tree-shaped \Imc.  By
  dropping subtrees, we may then additionally achieve that the
  outdegree of every node is bounded by $n$.

We have to construct an A-interpretation $\Jmc = (\Abf_{\Omc'},\prec,(\Jmc_L)_{L \in \Abf_{\Omc'}},\rho)$ 
that is a model of $\Omc'$ and such that $C_0^{\Jmc_L} \neq \emptyset$.

In the following, we
construct a sequence $\Jmc^0,\Jmc^1,\dots$ of A-interpretations and obtain 
the desired model \Jmc in the limit. 
Let $v_0 \in \Delta^{\Imc}$ be the root of $\Imc$.
We start with defining $\Jmc^0 = (\prec, (\Jmc^0_L)_{L \in \Abf_\Omc}, \rho^0)$:
\begin{align*}
  \Delta^{\Jmc^0_L} &= \{v_0\}\\
  A^{\Jmc^0_L} &= \{v_0\}       &&\text{for all $A \in \mn{cl}(\Omc) \cap \Cbf$ with $v_0 \in A^\Imc$}\\
  ({A_C})^{\Jmc^0_L} &= \{v_0\} &&\text{for all $C \in \mn{cl}(\Omc)$ with $v_0 \in C^\Imc$}\\
  ({E^0})^{\Jmc^0_L} &= \{v_0\}\\
  \Delta^{\Jmc_{L'}^0} &= \{d'\} &&\text{with $d'$ being a fresh domain element}
  % ({N^0_{|c_G(v_0)|}})^{\Jmc^0_L} &= \{v_0\}
\end{align*}

To construct $\Jmc^{i+1} = (\prec, (\Jmc^{i+1}_L)_{L \in \Abf_\Omc}, \rho^{i+1})$ from $\Jmc^{i}$ we start with 
$\Jmc^{i+1} = \Jmc^i$. Now we construct a structure similar to Figure~\ref{fig:ca_gadget}.
Let $d \in \Delta^{\Jmc^i_L}$ with $d \in E^i$ for some $i \in \{0,1\}$ but $d \not \in N^i_j$ for all $i \in \{0,1\}$ and $j \in \{0, \dots, n\}$.
Let $Q = e_1, \dots, e_m$ be the children of $d$ in $\Imc$. We do the following (disjoint) operations for all such $d$ :
  \begin{itemize}
    \item add $d$ to $(N^i_m)^{\Jmc^0_L}$;
    % \item add $Q$ to $\Delta^{\Jmc_L^i}$;
    \item for all $l \in \{1, \dots, m\}$ we \begin{itemize}
      \item add a fresh domain element $e_l$ to $\Delta^{\Jmc_L^i}$,
      \item add $e_l$ to $(E^{1-i})^{\Jmc_L^i}$,
      \item add $e_l$ to $A$ for all $A \in \mn{cl}(\Omc) \cap \Cbf$ with $e_l \in A^\Imc$,
      \item add $e_l$ to $A_C$ for all $C \in \mn{cl}(\Omc)$ with $e_l \in C^\Imc$,
      \item add a fresh domain element $v_l$ to $\Delta^{\Jmc_L^i}$,
      \item add $v_l$ to $(M^{i,j}_l)^{\Jmc_L^i}$,
      \item add $(d, v_l)$ to $\widehat r^{\Jmc_L^i}$,
      \item add $(v_l, d)$ to $r^{\Jmc_L^i}$, 
      \item add $(v_l, e_l)$ to $r^{\Jmc_L^i}$ and
      \item add $(v_{l-1}, v_l)$ to $u^{\Jmc_L^i}$ if $l \geq 2$;
    \end{itemize}
  \item if $m \geq 0$, then \begin{itemize}
    \item add a fresh element $d'$ to $\Delta^{\Jmc_{L'}^i}$, and
    \item add $(d, v_1, \dots, v_l)$ to $\rho^i_L(d')$.
  \end{itemize}
  \end{itemize}
  \smallskip
Note that $m \leq n$ is always true because of how we defined tree-shapedness for $\Imc$ above.
As announced, we take $\Jmc$ to be the limit of the constructed sequence of $\Jmc^0,\Jmc^1,\dots$.

What remains to show is that $\Jmc$ is a model of $\Omc'$ and that $C_0$ is satisfied on $L$.
We do this step by step starting with the basic condition that $\Jmc$ is an $A$-interpretation.
Subsequently, we will show that $C_0$ is satisfied on $L$ and that all statements in $\Omc$ are satisfied.
\begin{lemma}
  \label{lem:ainterpret}
  $\Jmc =  (\Abf_{\Omc'},\prec,(\Jmc_L)_{L \in \Abf_{\Omc'}},\rho)$ is an A-interpretation.
\end{lemma}
\begin{proof}
  To prove the lemma we go through the three conditions of A-interpretations.
  \begin{itemize}
    \item we only have two abstraction levels with $L \prec L'$. Consequently
    $(\Amc_{\Jmc},\{(L',L)\mid L \prec L' \})$ is obviously a tree;
    \item by definition, $\Delta^{\Jmc^0_L}$ $\Delta^{\Jmc^0_{L'}}$ are non-empty and thus the same holds for $\Delta^{\Jmc_L}$ and $\Delta^{\Jmc_{L'}}$;
    \item whenever we add a tuple to the range of $\rho$ in the construction of $\Jmc$, the elements in it are one endpoint $e$ and a number of fresh midpoints.
    In the same construction step, we add $e$ to $N^i_j$ for some $i \in \{1, 2\}$ and $j \in \{1, \dots, n\}$ which implies that we will never consider it again in 
    another step of constructing $\Jmc$.
    Thus there is always at most one $d \in \Delta^{\Imc}$ such that $e$ occurs in $\rho_{L}(d)$.
  \end{itemize}
\end{proof}

Now we are ready to prove the following lemma which establishes the soundness of the reduction.
\begin{lemma}
  \label{lem:aimodel}
  $\Jmc =  (\Abf_{\Omc'},\prec,(\Jmc_L)_{L \in \Abf_{\Omc'}},\rho)$ is a model of $\Omc'$ 
  and $C_0^{\Jmc_L} \neq \emptyset$.
\end{lemma}
\begin{proof}
  We use the fact that $\Jmc$ is indeed an A-interpretation as proven in the previous lemma.
  First, we prove the second part of the claim.
  Since we assumed that $C_0^{\Imc} \neq \emptyset$, the construction of $\Jmc_L$ 
  implies the desired $C_0^{\Jmc_L} \neq \emptyset$.

  Now we show that $\Jmc$ is a model of $\Omc'$.
  We prove below that the CIs and CA of (9) to (11) and (15) are satisfied in $\Jmc$.
  For all other statements it is trivial to verify that they are satisfied
  since it follows directly from the construction of $\Jmc$.
  
  For (9) let us assume $d \in (A_{\exists s.C} \sqcap N^i_j)^{\Jmc_L}$ for some $i \in \{0,1\}$ and $j \in \{1, \dots, n\}$. 
  Then by the definition of $\Jmc$, we have $d \in (\exists s.C)^\Imc$.
  The tree-shapedness, as defined above, implies that there is a child $e$ of $d$ in $\Imc$ such that $e \in C^\Imc$.
  % In the first case our construction of $\Jmc$ implies $d \in ({A_C})^{\Jmc_L}$ which satisfies (9).
  % In the second case 
  The construction of $\Jmc$ then lets us obtain a midpoint $v \in \Delta^{\Jmc_L}$ with 
  $(d, v) \in \widehat r^{\Jmc_L}$, $(v, e) \in r^{\Jmc_L}$, and $e \in A_C$ 
  and thus (9) is satisfied.

  For (10) let us assume $d \in (A_{\forall s.C})^{\Jmc_L}$.
  By definition of $\Jmc$ we have $d \in (A_{\forall s.C})^\Imc$ and because $\Imc$ is a model of $\Omc$
  the construction of $\Jmc$ immediately implies $d \in A_C$ and thus (10) is satisfied.

  For (11) let us assume $d \in \exists r.(A{\forall s.C})^{\Jmc_L}$.
  By the construction of $\Jmc$, this can only be the case for a midpoint, that is for all $e \in \Delta^{\Jmc_L}$ with $(d,e) \in r^{\Jmc_L}$ and $e \in A_{\forall s.C}$
  we have $e \in \Delta^\Imc$ and thus $e \in (\forall s.C)^\Imc$.
  The construction then also implies that for all $e' \in \Delta^{\Jmc_L}$ with $(d,e') \in r^{\Jmc_L}$ we have $(e,e') \in s^\Imc$ (and $(e', e) \in s^\Imc$).
  Now we can use that $\Imc$ is a model to obtain $e' \in C^\Imc$ and thus $e' \in (A_C)^{\Jmc_L}$, as required. 

  For (15) we argue by contradiction. 
  Assume that there is a homomorphism $h$ from $\widehat q$ to $\Jmc_L$.
  We show that $h|_{\mn{var}(q)}$ is a homomorphism from $q$ to $\Imc$ contradicting our original assumtion that $\Imc \not \models q$.
  Consider any $C(x) \in q$.
  By definition of $\widehat q$ we have $A_C(x) \in \widehat q$ and $(E^0 \sqcup E^1)(x) \in \widehat q$. 
  Consequently $\widehat h$ maps $x$ to an endpoint $e = h(x)$ in $\Jmc_L$  with $e \in A_C^{\Jmc_L}$ and
  the construction of $\Jmc_L$ then implies that $e \in C^\Imc$.

  Consider any $s(x,y) \in q$.
  By definition of $\widehat q$, we have $r(z,x) \in \widehat q$ and $r(z,y) \in \widehat q$ with $z$ a variable not occurring in $q$.
  Furthermore, $x$ and $y$ are again mapped to endpoints like in the concept atom case.
  The construction of $\Jmc_L$ then implies that either $h(x)$ is a child of $h(y)$ or the inverse is true in $\Imc$.
  Thus $(h(x), h(y)) \in s^\Imc$ and $(h(y), h(x)) \in s^\Imc$ which proves that $h|_{\mn{var}(q)}$ is a homomorphism
  from $q$ to $\Imc$ and consequently leading to a contradiction since we assumed that $\Imc \not \models q$.

\end{proof}

\subsection{Completeness}
Assume that $C_0$ is $L$-satisfiable w.r.t.\ $\Omc'$ and let
  $\Jmc = (\Abf_{\Omc'},\prec,(\Jmc_L)_{L \in \Abf_{\Omc'}},\rho)$ be
  a model of $\Omc'$ with $C_0^{\Jmc_L} \neq \emptyset$.
  We need to construct a model $\Imc$ of $\Omc$ such that $A_0^\Imc \neq \emptyset$
  and $\Imc \not \models q$.

   We define $\Imc$ as follows:
$$
\begin{array}{r@{\;}c@{\;}l}
  \Delta^{\Imc} &=&  \{d \mid d \in \Delta^{\Jmc_L} \text{ and } d \in ({E^0} \sqcup {E^1})^{\Jmc_L}\}                     \\[1mm]
  s^{\Imc}      &=&       \{(d,d) \mid d \in ({E^0} \sqcup {E^1})^{\Jmc_L}\}\; \cup                                    \\[1mm]
                & &                \{(d,d') \mid d,d' \in ({E^0} \sqcup {E^1})^{\Jmc_L} \text{ and } e \in \Delta^{\Jmc_L}\} \\[1mm]
                & &                \text{ with } (e,d) \in r^\Jmc \text{ and } (e,d') \in r^\Jmc \}                 \\[1mm]
  A^\Imc        &=& A^\Jmc      \text{\qquad for all $A \in \mn{sig}(\Omc) \cap \Cbf$}
\end{array}
$$

What remains to show is that $\Imc$ is a model of $\Omc$,  $A_0^\Imc \neq \emptyset$, and $\Imc \not \models q$. 
To prove this we first prove an intermediary lemma which intuitively says that the $A_C$ satisfied at the endpoints in $\Jmc$
coincide with the concepts $C$ satisfied by the corresponding element in $\Imc$. 
\begin{lemma}
  \label{lem:alc_ca_lb_sym_eq_ca}
  For every element $d \in (E^0 \sqcup E^1)^{\Jmc_L}$ and $\ALC^\mn{sym}$ concept $C \in \mn{cl}(\Omc)$ we have
  $d \in (A_C)^{\Jmc_L}$ iff $d \in C^\Imc$
\end{lemma}
\begin{proof}
  To prove the ``if'' direction of the lemma, we do an induction on $C$.
\begin{itemize}
  \item If $C = B$, then (1) and the definition of $\Imc$ give us $d \in B^\Imc$.
  \item If $C = \neg B$, then (2) and the definition of $\Imc$ give us $d \in (\neg B)^\Imc$.
  \item If $C = D \sqcap D'$, then (3) implies $d \in (A_D \sqcap A_{D'})^{\Jmc_L}$.
  Now we can apply the IH twice to get $d \in (D \sqcap D')^\Imc$.
  \item If $C = D \sqcup D'$, then (4) implies $d \in (A_D \sqcup A_{D'})^{\Jmc_L}$.
  Now we can apply the IH twice to get $d \in (D \sqcup D')^\Imc$.
  \item If $C = \exists s.D$, then we use multiple CIs and concept abstractions for our reasoning.
  Firstly, $d \in N^i_j$ for some $i \in \{0,1\}$ and $j \in \{0,\dots, n\}$ because of (6).
  Next, we consider (9).
  If $j = 0$, then $d \in A_D$ and by the IH $d \in D^\Imc$ and thus $d \in (\exists s.D)^\Imc$, as required.
  If $j > 0$, then (7) results in $j$ (at least one) $\widehat r$-successors of $d$
  that satisfy $M^{i,j}_k$ for all $k \in \{1, \dots, j\}$.

  Let $M = \{e \mid (d,e) \in \widehat r^{\Jmc_L} \text{ and } e \in M^{i,j}_k \text{ for some } k \in \{1, \dots, j\}\}$ be the set of those $\widehat r$-successors.
  Point (8) tells us that each $e \in M$ has two $r$-successors, one that satisfies $E^i$ and one that satisfies $E^{1-i}$.
  For the next step, we are specifically interested in the $E^{1-i}$ successor in conjunction with (9).

  Point (9) together with the previous part of the proof imply that from $d$ there is a $\widehat r;r$ path 
  via one element $v \in M$ to an element $e \in (A_D \sqcap E^{1-i})^{\Jmc_L}$.
  If we apply the IH on $e$ we get $e \in D^\Imc$.
  What remains to show is that there is also an $r$ edge pointing from $v$ to $d$.

  The concept abstractions in (13) make $d$ and all the $v' \in M$ an $L$-ensemble $\bar e$.
  In particular, $v$ is part of $\bar e$.
  % Let us consider an element $v' \in M$ with $v' \in (M^{i,j}_1)^{\Jmc_L}$.
  Point~$(12)$ implies that there is a chain $\{(v_1, v_2), \dots, (v_{j-1}, v_{j})\} \subseteq u^{\Jmc_L}$ 
  such that $v_l \in  M^{i,j}_l$ for all $l \in \{1, \dots, j\}$.
  Point~$(14)$ takes this chain together with one $r$-successor that satisfies $E^i$ and makes it an $L$-ensemble $\bar e'$.
  It is easy to see that $\bar e$ and $\bar e'$ must overlap partially and the semantics of refinement functions
  then implies that they overlap fully, that is, $\bar e = \bar e'$.
  Thus the elements in $M$ are connected by a chain of $u$-edges.

  Now consider the following concept abstraction of Point~(14).
  \begin{align*}
    L'{:}\top &\abs L{:} E^i(x_0) \land r(x_k, x_0) \land \\
    &\bigwedge_{1 \leq l < j} M^{i,j}_l(x_l) \land u(x_l, x_{l+1}) \land M^{i,j}_{l+1}(x_{l+1}) 
  \end{align*}
  This forces another $L$-ensemble $\bar e''$ into existence. 
  In this ensemble, $v$ has an $r$-successor that satisfies $E^i$.
  Again it is straightforward to see that $\bar e, \bar e'$ and $\bar e''$ partially overlap and must therefore be equal.
  This implies that the $r$-successor of $v$ is $d$.
  We have now shown that there is an $r^-;r$ path from $d$ via $v$ to $e$ and $e \in D^\Imc$.
  Thus by the definition of $\Imc$, we have $d \in (\exists s.D)^\Imc$, as required.
  
  \item If $C = \forall s.D$, then first let us consider $d$ itself. 
  By (10) we have $d \in (A_D)^{\Jmc_L}$ and with the IH we obtain $d \in D^\Imc$. 
  Otherwise let $e \in \Delta^\Imc$ be an element with $e \neq d$ and $(d,e) \in s^\Imc$.
  We have to show that $e \in D^\Imc$.
  The definition of $\Imc$ implies that $e \in (E^i)^{\Jmc_L}$ for some $i \in \{0,1\}$ and
  that there is a $v \in \Delta^{\Jmc_L}$ with $(v,d) \in r^{\Jmc_L}$ and $(v,e) \in r^{\Jmc_L}$.
  This $v$ then has to satisfy $\exists r.A_{\forall s.D}$ and due to (11) we then have $e \in (A_D)^{\Jmc_L}$.
  Now we can use the IH to obtain $e \in D^\Imc$ and thus $d \in (\forall s.D)^\Imc$, as required.

\end{itemize}
\noindent
  ``only if''.  We argue by contrapositive. If
  $d \not \in (A_C)^{\Jmc_L}$, then $d \in (A_{\overline C})^{\Jmc_L}$ by (2).
   We can now apply the `$\Leftarrow$' direction
  to get $d \in \overline C^\Imc$, and thus $d \not \in C^\Imc$.
\end{proof}

Now proving that $\Imc$ is a model of $\Omc$ is very straightforward.
Let $d \in C^\Imc$ for any $d \in \Delta^\Imc$.
We have to show that $d \in D^\Imc$ for all CIs $C \sqsubseteq D \in \Omc$.
Lemma~\ref{lem:alc_ca_lb_sym_eq_ca} implies that $d \in (A_C)^{\Jmc_L}$.
Point (5) in the construction of $\Omc'$ implies $d \in (A_D)^{\Jmc_L}$ and now we can again apply
Lemma~\ref{lem:alc_ca_lb_sym_eq_ca} to obtain $d \in D^\Imc$. 
Thus $\Imc$ is a model of $\Omc$.

By assumption $C_0^\Jmc \neq \emptyset$ and thus there is a $d_0 \in (A_0 \sqcap (E^0 \sqcup E^1))^{\Jmc_L}$.
The construction of $\Imc$ then implies $d_0 \in C_0^\Imc$.

Lastly, we have to show that $\Imc \not \models q$.
We prove this by contradiction.
Assume that there is a homomorphism $h$ from $q$ to $\Imc$.
We want to show that we can extend $h$ to also be a homomorphism from $\widehat q$ to $\Jmc_L$ 
which would make $\Jmc$ no longer a model of $\Omc'$ because of (15).

First, consider $h$ as a partial homomorphism from $\widehat q$ to $\Jmc_L$.
Let $C(x) \in \widehat q$ be a concept atom.
There are two cases. 
If $C = A_D$ for some $D(x) \in q$ then $h(x) \in C^{\Jmc_L}$ because of Lemma~\ref{lem:alc_ca_lb_sym_eq_ca}.
Otherwise $C = E^0 \sqcup E^1$ and then $h(x) \in C^{\Jmc_L}$ simply by the construction of $\Imc$.
What remains to show is that all role atoms in $\widehat q$ are satisfied.

Let $r(z,x) \in \widehat q$ be a role atom. 
The definition of $\widehat q$ implies that there is a second role atom $r(z,y)$ in $\widehat q$
and that $s(x,y)$ or $s(y,x) \in q$.
The two cases are analogous so let us assume w.l.o.g.\ that $s(x,y) \in q$.
Since the construction of $s^\Imc$ is defined in such a way that $(h(x), h(y)) \in s^\Imc$ if and only if 
there is an $e \in \Delta^{\Jmc_L}$ with $(e, h(x)) \in r^{\Jmc_L}$ and $(e, h(y)) \in r^{\Jmc_L}$
we can simply extend $h$ by $h(z) = e$ and now we have $(h(z), h(x)) \in r^{\Jmc_L}$, as required.
Thus $\Jmc_L \models \widehat q$ which leads to a contradiction since (15) would now imply that $\Jmc$ is not a model of $\Omc'$.

% \end{proof}

\section{Proofs for Section~\ref{sect:repfree}}
\label{app:repfree}

We prove Proposition~\ref{thm:reprefok} and complete the proof of
Theorem~\ref{thm:firstundec}, starting with the former.

\thmreprefok*
\begin{proof}
  Assume that $C_0$ is $L_0$-satisfiable w.r.t.\ $\Omc$.  Then there
  is a model 
  $\Imc = (\Abf_\Omc,\prec,(\Imc_L)_{L \in {\AI}},\rho)$ of
  $\Omc$ with $C_0^{\Imc_{L_0}} \neq \emptyset$.  Assume that for some
  $L_1,L_2 \in \Abf_\Omc$ and $d_0 \in \Delta^{\Imc_{L_1}}$ we have
  $\rho_{L_2}(d_0) = \bar d$ with $\bar d = d_1 \cdots d_n$ such that
  an element $e_0$ occurs in positions $i_1, \dots, i_m$ of $\bar d$
  with $m \geq 2$.

  We construct a model
  $\Imc' = (\Imc_\Omc,\prec,(\Imc'_L)_{L \in {\AI}},\rho')$ of $\Omc$
  with $C^{\Imc'_L} \neq \emptyset$ such that $e_0$ no longer occurs
  multiple times in $\rho'_{L_2}(d_0)$ (in fact, it does not occur at
  all anymore). Let $e_1,\dots,e_m$ be fresh domain elements.  We
  define an A-interpretation
  $\Imc' = (\Abf_\Omc,\prec,(\Imc'_L)_{L \in {\AI}},\rho')$ that
  differs from \Imc only in $\Imc'_{L_2}$ and $\rho'$. To start, set
    $$
    \Delta^{\Imc'_{L_2}} = \Delta^{\Imc_{L_2}} \cup \{ e_1,\dots,e_m \}.
    $$
    For every $d \in \Delta^{\Imc'_{L_2}}$, put $d^\uparrow=d$ if
    $d \in \Delta^{\Imc_{L_2}}$ and $d^\uparrow=e_0$ if
    $d \notin \Delta^{\Imc_{L_2}}$. Then define
    $$
    \begin{array}{rcl}
      A^{\Imc'_{L_2}} &=& \{ d \in \Delta^{\Imc'_{L_2}} \mid 
                          d^\uparrow \in A^{\Imc_{L_2}} \} \\[1mm]
      r^{\Imc'_{L_2}} &=& \{ (d,e) \in \Delta^{\Imc'_{L_2}} \times \Delta^{\Imc'_{L_2}} \mid 
                          (d^\uparrow, e^\uparrow) \in r^{\Imc_{L_2}} \}
    \end{array}
    $$
    for all concept names $A$ and role names $r$ and define $\rho'$
    like $\rho$ except that $\rho'(d_0)$ is obtained from $\rho(d_0)$
    by replacing the occurrence of $e$ in position $i_j$ with $e_j$,
    for $1 \leq j \leq m$.
    
 %  \begin{alignat*}{2}
 %    C^{\Imc'_{L_2}}      &= C^{\Imc_{L_2}}    && \cup \{e_i \mid e \in C^{\Imc_{L_2}}, 1 \leq i \leq m\}\\
 %    R^{\Imc'_{L_2}}      &= R^{\Imc_{L_2}}   && \cup \{(d, e_i) \mid (d,e) \in R^{\Imc_{L_2}}, 1 \leq i \leq m\}\\
 %                         &                             && \cup\{(e_i, d) \mid (e,d) \in R^{\Imc_{L_2}}, 1 \leq i \leq m\}\\
 %                         &                             && \cup\{(e_i, e_j) \mid (e,e) \in R^{\Imc_{L_2}}, 1 \leq i,j \leq m\}.
 % \end{alignat*}
 % %
 Clearly, $C^{\Imc_L} \neq \emptyset$ implies $C^{\Imc'_L} \neq
 \emptyset$
 and $\Imc'$ satisfies all role inclusions in \Omc since \Imc does.
 Moreover, it is easy to show the following by induction on the
 structure of $C$:
 \\[2mm]
 {\bf Claim.} For all \ALCI-concepts $C$ and $d \in \Delta^{\Imc'_{L_2}}$:
 $d \in C^{\Imc'_{L_2}}$ iff
 $d^\uparrow \in C^{\Imc_{L_2}}$.
 \\[2mm]
 It follows that $\Imc$ also satisfies all concept inclusion in \Omc.
 We next show that it also satisfies all concept and role refinements,
 and thus is a model of \Omc.

 Let $q(\bar x)$ be a CQ such that \Omc contains a concept refinement
 $L_2{:}q(\bar x) \rfn L_1{:}C$. %, and let $\bar x = x_1 \cdots
     %      x_n$.
 Further assume that $d' \in C^{\Imc'_{L_1}}$.  Then
 $d' \in C^{\Imc_{L_1}}$ by the claim and thus $\rho(d')$ is defined and the
 function $h$ that satisfies $h(\bar x)=\rho(d')$ is a homomorphism
 from $q$ to $\Imc_{L_2}$. By definition of $\rho'$ and the
 claim,
 the function $h'$ that satisies $h'(\bar x)=\rho'(d')$ is a homomorphism
 from $q$ to $\Imc'_{L_2}$. The case of role refinements is similar.

 By applying this construction multiple times, we eventually arrive at
 a repetition-free model of \Omc, as desired.
\end{proof}

We next present the remaining details of the proof of
Theorem~\ref{thm:firstundec} which we repeat here for the reader's convenience.

\thmfirstundec*

Recall that we want to construct an
ontology $\Omc$ and choose a concept name $S$ and abstraction level
$L$ such that $S$ is $L$-satisfiable w.r.t.\ $\Omc$ iff $M$ does not
halt on the empty tape. A part of \Omc has already been given in the
main part of the paper. We proceed with the construction of \Omc.

We first step up the initial configuration:
\begin{align*}
  S &\sqsubseteq_L A_{q_0} \sqcap A_\square  \sqcap B_\shortrightarrow)\\
  B_\shortrightarrow &\sqsubseteq_L \forall t.(A_\square \sqcap B_\shortrightarrow)
\end{align*}
Transitions are again indicated by marker concepts:
$$
  A_q \sqcap A_\sigma \sqsubseteq_L \forall c.B_{q', \sigma', M}
$$
for all $q \in Q$ and $\sigma \in \Gamma$ such that $\delta(q,\sigma)
= (q',\sigma',M)$. We implement transitions as follows:
$$
  B_{q, \sigma, M} \sqsubseteq_L A_\sigma \quad
  \exists t.B_{q, \sigma, L} \sqsubseteq_L A_q \quad
  B_{q, \sigma, R} \sqsubseteq_L \forall t.A_q
$$
for all $q \in Q$, $\sigma \in \Gamma$, and $M \in \{L, R\}$.

We next mark cells that are not under the head:
$$
\begin{array}{r@{\;}c@{\;}lcr@{\;}c@{\;}l}
  A_q &\sqsubseteq_{L_n}& \forall t.H_{\shortleftarrow} &&
                                                              \exists t.A_q &\sqsubseteq_{L_n}& H_{\shortrightarrow} \\[1mm]
    H_{\shortleftarrow} &\sqsubseteq_{L_n}& \forall
                                               t.H_{\shortleftarrow} &&
\exists t.H_{\shortrightarrow} &\sqsubseteq_{L_n}& H_{\shortrightarrow}
\end{array}
$$
for all $q \in Q$.
We can now say that cells that are not under the head do not change:
$$
  (H_\shortleftarrow \sqcup H_\shortrightarrow) \sqcap A_\sigma \sqsubseteq_{L_n} \forall c_i.A_\sigma
$$
for all $\sigma \in \Gamma$ and $i \in \{1,2\}$.
State, content of tape, and head position must be unique:
$$
\begin{array}{c}
  A_q \sqcap A_{q'} \sqsubseteq_{L_n} \bot
  \quad A_\sigma \sqcap A_{\sigma'} \sqsubseteq_{L_n} \bot \\[1mm]
(H_\shortleftarrow \sqcup H_\shortrightarrow) \sqcap A_q \sqsubseteq_{L_n} \bot
\end{array}
$$
for all $q, q' \in Q$ and $\sigma,\sigma' \in \Gamma$ with $q \neq q'$
and $\sigma \neq \sigma'$.

Since we reduce from the complement of the halting problem, 
we do not want the halting state to be reached: 
$$
  A_{q_h} \sqsubseteq_L \bot. 
$$
This concludes the construction of $\Omc$. 
\begin{lemma}
  \label{lem:rep_free_corr}
  $S$ is $L$-satisfiable w.r.t.\ $\Omc$ iff $M$ does not halt on the
  empty tape.
\end{lemma}
% %
\begin{proof}
  ``if''. Assume that $M$ does not halt on the empty tape. We 
  use $|w|$ to denote the length of a word $w \in \Sigma^*$ and
  construct an A-interpretation \Imc as follows:
  \begin{align*}
    \Delta^{\Imc_L}               &= \{c_i t_j \mid i,j \in \mathbb{N}\}\\
    S^{\Imc_L}                    &= \{c_0 t_0\}\\
    A_q^{\Imc_L}                  &= \{c_i t_j \mid K_i = wqw' \text{ and } |w| = j\}\\
    A_\sigma^{\Imc_L}             &= \{c_i t_j \mid K_i = \sigma_0 \cdots \sigma_k q \sigma_{k+1} \cdots \sigma_l \text{ and either }\\
                                  & \quad\;\;\; \sigma = \sigma_j  \text{ or } j > l \text{ and } \sigma = \square\}\\
    B_{\shortrightarrow}^{\Imc_L} &= \{c_0 t_j \mid j \in \mathbb{N}\}\\
    B_{q', \sigma', M}^{\Imc_L}   &= \{c_i t_j \mid K_{i-1} = wq\sigma w' \text{, } |w| = j \text{ and } \\
                                  & \quad\;\;\; \delta(q, \sigma) = (q', \sigma', M)\}\\
    H_{\shortrightarrow}^{\Imc_L}          &= \{c_i t_j \mid K_i = wqw' \text{ and } j < |w|\}\\
    H_{\shortleftarrow}^{\Imc_L}           &= \{c_i t_j \mid K_i = wqw' \text{ and } j > |w|\}\\
    t^{\Imc_L}                    &= \{(c_i t_j, c_i t_{j+1}) \mid i,j \in \mathbb{N}\}\\
    c^{\Imc_L}                    &= \{(c_i t_j, c_{i+1} t_j) \mid i,j \in \mathbb{N}\}
  \end{align*}
  for all $q \in Q$, $\sigma \in \Gamma$, and $M \in \{L, R\}$.

  For $\Imc_{L'}$ we simply define a singleton domain $\Delta^{\Imc_{L'}} = \{d\}$ and no concept extensions;
  $\rho$ we leave undefined for all elements.
  It is straightforward to verify that \Imc is a model of \Omc, and
  thus $S$ is $L$-satisfiable w.r.t.\ \Omc.

  \smallskip

  ``only if''. Assume that $S$ is $L$-satisfiable w.r.t.\ \Omc and let
  \Imc be a model of \Omc with $S^{\Imc_L} \neq \emptyset$. We
  identify a grid in $\Imc_L$ in the form of a mapping
  $\pi: \mathbb{N} \times \mathbb{N} \rightarrow \Delta^{\Imc_L}$ such
  that for all $i,j \in \mathbb{N} \times \mathbb{N}$,
  $(\pi(i,j),\pi(i+1,j)) \in t^{\Imc_L}$ and
  $(\pi(i,j),\pi(i,j+1)) \in c^{\Imc_L}$. We construct such a $\pi$ 
  in several steps:
  \begin{itemize}

  \item 
  Start with choosing some $d \in S^{\Imc_L}$ and set
  $\pi(0,0)=d$.

\item Complete the diagonal. If $\pi(i,i)$ is defined and
  $\pi(i+1,i+1)$ undefined, then choose $d,e \in \Delta^{\Imc_L}$ such
  that $(\pi(i,i),d) \in t^{\Imc_L}$ and $(d,e) \in c^{\Imc_L}$ and
  set $\pi(i+1,i)=d$ and $\pi(i+1,i+1)=e$.

\item Add grid cells in the upwards direction. If
  $\pi(i,i), \pi(i+1,i),\pi(i+1,i+1)$ are defined and $\pi(i,i+1)$ is
  undefined, then choose $d \in \Delta^{\Imc_L}$ such that
  $(\pi(i,i),d) \in c^{\Imc_L}$ and set $\pi(i+1,i)=d$. We have to
  argue that $(\pi(i+1,i),\pi(i+1,i+1)) \in t^{\Imc_L}$. There is some
  $e \in \Delta^{\Imc_L}$ with $(\pi(i+1,i),e) \in t^{\Imc_L}$.  The
  elements $\pi(i,i),\pi(i+1,i),\pi(i,i+1),\pi(i+1,i+1),e)$ cannot all
  be distinct because then the first concept abstraction in \Omc
  applies, meaning that $\Imc$ cannot be a model of \Omc. Thus at
  least two out of these five elements have to be identical; then, the
  first concept abstraction does not apply since we work under the
  repetition-free semantics. The additional concept abstraction in
  \Omc, however, rules out any identifications except that of
  $\pi(i+1,i+1)$ and $e$. Consequently,
  $(\pi(i+1,i),\pi(i+1,i+1)) \in t^{\Imc_L}$ as required.
  
\item Add grid cells in the downwards direction. Analogous to
  the upwards case.

  \end{itemize}
  We can now read off the computation of $M$ on the empty tape from the
  grid in a straightforward way, using the concept names $A_\sigma$
  for the tape content and $A_q$ for the state and head position.
  Since $A_{q_h}^{\Imc_L}=\emptyset$, the computation is non-terminating.
\end{proof}

\section{Proofs for Section~\ref{sect:nontree}}

We complete the proof of Theorem~\ref{thm:nontreeundec} which we repeat here for the reader's convenience.  

\thmnontreeundec*

Recall that
we want to construct an ontology $\Omc$ and choose a concept name $S$
and abstraction level $L$ such that $S$ is $L$-satisfiable w.r.t.\
$\Omc$ iff $M$ does not halt on the empty tape. A part of \Omc has
already been given in the main body of the paper. The construction of
\Omc is completed by adding the concept inclusions presented in
Appendix~\ref{app:repfree}.

\begin{lemma}
  \label{lem:dag_corr}
  $S$ is $L$-satisfiable w.r.t.\ $\Omc$ iff $M$ does not halt on the
  empty tape.
\end{lemma}
\begin{proof}
  ``if''. Assume that $M$ does not halt on the empty tape. We
  construct an A-interpretation \Imc as follows:
  \begin{align*}
    \Delta^{\Imc_L} &= \{c_i t_j \mid i,j \in \mathbb{N}\}\\
    A_t^{\Imc_L} &= \{c_0 t_j \mid j \in \mathbb{N}\}\\
    A_c^{\Imc_L} &= \Delta^{\Imc_L}\\
    S^{\Imc_L} &= \{c_0 t_0\}\\
    A_q^{\Imc_L}                  &= \{c_i t_j \mid K_i = wqw' \text{ and } |w| = j\}\\
    A_\sigma^{\Imc_L}             &= \{c_i t_j \mid K_i = \sigma_0 \cdots \sigma_k q \sigma_{k+1} \cdots \sigma_l \text{ and either }\\
                                  & \quad\;\;\; \sigma = \sigma_j  \text{ or } j > l \text{ and } \sigma = \square\}\\
    B_{\shortrightarrow}^{\Imc_L} &= \{c_0 t_j \mid j \in \mathbb{N}\}\\
    B_{q', \sigma', M}^{\Imc_L}   &= \{c_i t_j \mid K_{i-1} = wq\sigma w' \text{, } |w| = j \text{ and } \\
                                  & \quad\;\;\; \delta(q, \sigma) = (q', \sigma', M)\}\\
    H_{\shortrightarrow}^{\Imc_L}          &= \{c_i t_j \mid K_i = wqw' \text{ and } j < |w|\}\\
    H_{\shortleftarrow}^{\Imc_L}           &= \{c_i t_j \mid K_i = wqw' \text{ and } j > |w|\}\\
    X_1^{\Imc_L} &= \{c_i t_j \mid i \bmod 2 = 0 \text{ and } j \bmod 2 = 0\}\\
    X_2^{\Imc_L} &= \{c_i t_j \mid i \bmod 2 = 0 \text{ and } j \bmod 2 = 1\}\\
    X_3^{\Imc_L} &= \{c_i t_j \mid i \bmod 2 = 1 \text{ and } j \bmod 2 = 0\}\\
    X_4^{\Imc_L} &= \{c_i t_j \mid i \bmod 2 = 1 \text{ and } j \bmod 2 = 1\}\\
    t^{\Imc_L} &= \{(c_i t_j, c_i t_{j+1}) \mid i,j \in \mathbb{N}\}\\
    c^{\Imc_L} &= \{(c_i t_j, c_{i+1} t_j) \mid i,j \in \mathbb{N}\}
  \end{align*}
  for all $q \in Q$, $\sigma \in \Gamma$, and $M \in \{L, R\}$.
  
  % {\color{red} For every answer $\bar e \in q_i(\Imc_L)$ with $i \in \{1, \dots, 4\}$, we introduce a fresh element $d$, 
  % add $d$ to $\Delta^{\Imc_{L_i}}$ and $U_i^{\Imc_{L_i}}$, and set $\rho_L(d) =\bar e$.}

  We use $\bar e^i_j$ to denote a $4$-tuple $\bar e^i_j = c_i t_j \cdot c_i t_{j+1} \cdot c_{i+1} t_j \cdot c_{i+1} t_{j+1}$
  with $c_i, t_j \in \Delta^{\Imc_L}$ and $i,j \in \mathbb{N}$.
  We define four sets of $4$-tuples that represent the answers to the four CQs in the abstractions:
  \begin{align*}
    Q_1 &= \{\bar e^i_j \mid  i \bmod 2 = 0 \text{ and } j \bmod 2 = 0\}\\
    Q_2 &= \{\bar e^i_j \mid  i \bmod 2 = 0 \text{ and } j \bmod 2 = 1\}\\
    Q_3 &= \{\bar e^i_j \mid  i \bmod 2 = 1 \text{ and } j \bmod 2 = 0\}\\
    Q_4 &= \{\bar e^i_j \mid  i \bmod 2 = 1 \text{ and } j \bmod 2 = 1\}
  \end{align*}
  Now we define the $\Imc_{L_k}$ for all $k \in \{1, \dots, 4\}$:
  \begin{align*}
    \Delta^{\Imc_{L_k}} &= \{d^i_j \mid \bar e^i_j \in Q_k\}\\
    U_i^{\Imc_{L_k}} &= \Delta^{\Imc_{L_k}}
  \end{align*}
  Next we add $\bar e^i_j$ to $\rho_L(d^i_j)$ for all $k \in \{1, \dots, 4\}$ and $d^i_j \in \Delta^{\Imc_{L_k}}$.
  It is straightforward to verify that \Imc is a model of \Omc, and
  thus $S$ is $L$-satisfiable w.r.t.\ \Omc.

  \smallskip

  ``only if''. Assume that $S$ is $L$-satisfiable w.r.t.\ \Omc and let
  \Imc be a model of \Omc with $S^{\Imc_L} \neq \emptyset$. We
  identify a grid in $\Imc_L$ in the form of a mapping
  $\pi: \mathbb{N} \times \mathbb{N} \rightarrow \Delta^{\Imc_L}$ such
  that for all $i,j \in \mathbb{N} \times \mathbb{N}$,
  $(\pi(i,j),\pi(i+1,j)) \in t^{\Imc_L}$ and
  $(\pi(i,j),\pi(i,j+1)) \in c^{\Imc_L}$. We construct such a $\pi$ 
  in several steps:
  \begin{itemize}

    \item 
    Start with choosing some $d \in S^{\Imc_L}$ and set
    $\pi(0,0)=d$.
  
  \item Complete the bottom horizontal. If $\pi(i,0)$ is defined and
    $\pi(i+1,0)$ undefined, then choose $d \in \Delta^{\Imc_L}$ such
    that $(\pi(i,0),d) \in t^{\Imc_L}$ and
    set $\pi(i+1,0)=d$.

  \item Complete the verticals. 
    If $\pi(i,j)$ is defined and
    $\pi(i,j+1)$ undefined, then choose $d \in \Delta^{\Imc_L}$ such
    that $(\pi(i,j),d) \in c^{\Imc_L}$ and
    set $\pi(i,j+1)=d$.

  \item Complete all horizontals.
    Now we have to argue that we can find horizontals such that the verticals and horizontals form a grid.
 
    For this purpose, we do an induction for all $j \geq 0$ proving that
    $(\pi(i,j),\pi(i+1,j)) \in t^{\Imc_L}$ and
    either $\pi(i,j) \in X_k^{\Imc_L}$ and $\pi(i+1,j) \in X_{k+1}^{\Imc_L}$ for any $k \in \{1,3\}$ or 
    $\pi(i,j) \in X_k^{\Imc_L}$ and $\pi(i+1,j) \in X_{k-1}^{\Imc_L}$ for any $k \in \{2,4\}$.

    For $j = 0$ this follows from the bottom horizontal we completed and the CIs $S \sqsubseteq X_1$, $X_1 \sqsubseteq \forall c.X_3 \sqcap \forall t.X_2$,
    and $X_2 \sqsubseteq \forall c.X_4 \sqcap \forall t.X_1$.

    Now let us argue for $j > 0$.
    By the IH we have $(\pi(i,j-1),\pi(i+1,j-1)) \in t^{\Imc_L}$.
    We can complete the two vertical edges as defined above to obtain
    $(\pi(i,j-1),\pi(i,j)) \in c^{\Imc_L}$ and $(\pi(i+1,j-1),\pi(i+1,j)) \in c^{\Imc_L}$.
    Now it is straightforward to prove that by the IH and the CIs concerning the $X_k$ 
    that one of the CQs $q_k$ with $k \in \{1,2,3,4\}$ as defined in Figure~\ref{fig:undec_part_ord_grid}
    can be matched onto $\pi(i,j-1) \pi(i+1,j-1) \pi(i,j) \pi(i+1,j)$.  
    The abstraction for $q_i$ followed by the refinement for $U_i$ then imply $(\pi(i,j),\pi(i+1,j)) \in t^{\Imc_L}$, as required.
    \end{itemize}
    We can now read off the computation of $M$ on the empty tape from the
    grid in a straightforward way, using the concept names $A_\sigma$
    for the tape content and $A_q$ for the state and head position.
    Since $A_{q_h}^{\Imc_L}=\emptyset$, the computation is non-terminating.
  
  \end{proof}

  \section{Proofs for Section~\ref{sect:qvar}}

  We complete the proof of Theorem~\ref{thm:qvar} which we repeat here for the reader's convenience.

  \thmqvar*

  Recall that we want to construct an $\ALCI$ ontology $\Omc$ and choose a concept name $S$
  and abstraction level $L$ such that $S$ is $L$-satisfiable w.r.t.\
  $\Omc$ iff $M$ does not halt on the empty tape. 
  We generate an infinite $t$-path with outgoing
  infinite $c$-paths from every node:
  $$
  \begin{array}{r@{\;}c@{\;}lcr@{\;}c@{\;}l}
    S &\sqsubseteq_L& A_t && A_t &\sqsubseteq_L& \exists t . A_t
    \\[1mm]
    A_t &\sqsubseteq_L& \exists c . A_c && A_c &\sqsubseteq_L& \exists c . A_c.
  \end{array}
  $$

  Next we want to form a grid. It is not fully connected to avoid overlapping ensembles (see Figure~\ref{fig:quant_grid}) 
  but pairs of elements that have distance $1$ in a normal grid have at most distance $3$ here.
  To achieve this, we first label the grid with concept names $X_1,\dots,X_4$ as
  shown in Figure~\ref{fig:quant_grid}, using the following CIs:
  \begin{figure}
    \begin{center}
    \includegraphics[width=0.9\columnwidth]{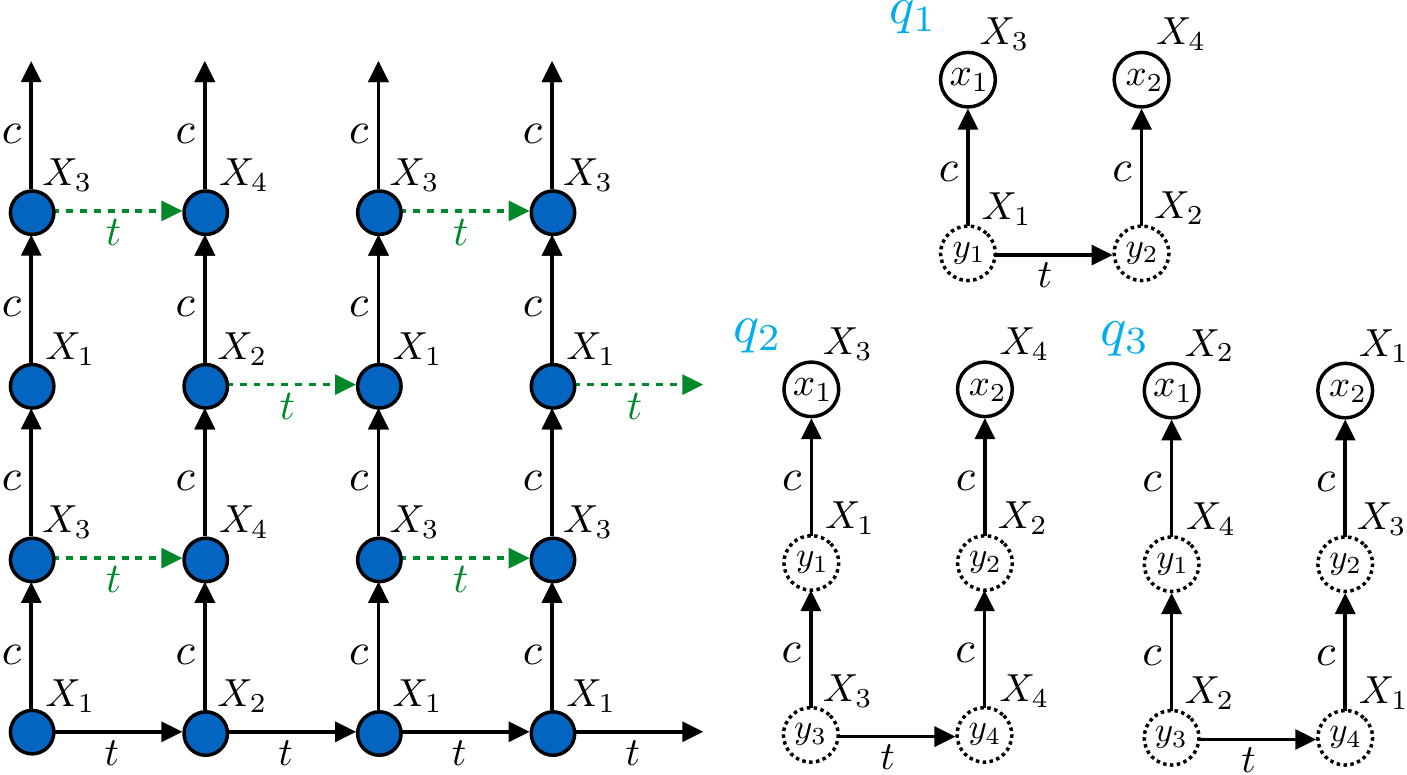}
    \end{center}
    \caption{Grid structure and CQs.}
    \label{fig:quant_grid}
    \vspace*{-5mm}
  \end{figure}

  $$
  \begin{array}{c}
    X_1 \sqsubseteq_L \forall c.X_3 \sqcap \forall t.X_2\qquad
    X_2 \sqsubseteq_L \forall c.X_4 \sqcap \forall t.X_1 \\[1mm]
    X_3 \sqsubseteq_L \forall c.X_1 \qquad
    X_4 \sqsubseteq_L\forall c.X_2 \qquad S \sqsubseteq_L X_1.
  \end{array}
  $$

  Now we close every other grid cell, as indicated by the dashed edges in Figure~\ref{fig:quant_grid}.
  We define three CQs $q_1,q_2, q_3$, as shown on the right-hand side of Figure~\ref{fig:quant_grid}.
  The dots around the $y_i$ indicate that they are quantified variables. 
  For example $q_1$ is defined as follows:
  \begin{align*}
    q_1(x_1,x_2) = {}&\exists y_1 \exists y_2 (X_1(y_1) \land X_2(y_2) \land X_3(x_1) \land \\
  & X_4(x_2) \land  t(y_1,y_2) \land c(y_1,x_1) \land c(y_2,x_2))
  \end{align*}
  The following concept abstraction and refinement then make use of the three CQs to add the additional horizontal edges on level $L$. 
  For $i \in \{1,2,3\}$:
  $$
  \begin{array}{l}
    L'{:}\, U_i \abs L{:}\,q_i \\[1mm]
    L{:}\, q_i \land t(x_1, x_2) \rfn L'{:}\,U_i.
  \end{array}
  $$
  Note that in Figure~\ref{fig:quant_grid} this makes every pair of nodes connected by a dashed $t$-edge an ensemble.
  It is straightforward to see that these ensembles never overlap.

  What remains to show is how to encode the runs of the turing machine into the grid.
  We add all the concept inclusions presented in Appendix~\ref{app:repfree}.
  The only thing missing is to ensure that the tape cells which have no $t$-successor follow the same rules.
  Intuitively, for every $\forall t. D$ appearing on the right side of a CI 
  we have to add a copy of that CI where we replace $\forall t. D$ with $\forall c.\forall t. \forall c^-.D$.
  Similar for CIs with $\exists t. D$ on the left side which will be replaced by $\exists c.\exists t. \exists c^-.D$.

  For the transitions of the turing machine, we thus add the following CIs to $\Omc$:
  $$
    \exists c.\exists t. \exists c^-.B_{q, \sigma, L} \sqsubseteq_L A_q \quad
    B_{q, \sigma, R} \sqsubseteq_L \forall c.\forall t.\forall c^-.A_q
  $$

  And analogously for marking cells not under the head, we add the following CIs to $\Omc$:
  $$
  \begin{array}{r@{\;}c@{\;}lcr@{\;}c@{\;}l}
    A_q &\sqsubseteq_{L_n}& \forall c.\forall t.\forall c^-.H_{\shortleftarrow} &&
    \exists c.\exists t. \exists c^-.A_q &\sqsubseteq_{L_n}& H_{\shortrightarrow} \\[1mm]
      H_{\shortleftarrow} &\sqsubseteq_{L_n}& \forall c.\forall t.\forall c^-.H_{\shortleftarrow} &&
      \exists c.\exists t. \exists c^-.H_{\shortrightarrow} &\sqsubseteq_{L_n}& H_{\shortrightarrow}
  \end{array}
  $$
  
  The correctness is captured by the following lemma which can be proved analogously to Lemma~\ref{lem:dag_corr}. 
  In fact, the proof is an even simpler version because we only have two abstraction levels instead of five.

  \begin{lemma}
    \label{lem:quant_corr}
    $S$ is $L$-satisfiable w.r.t.\ $\Omc$ iff $M$ does not halt on the
    empty tape.
  \end{lemma}
  \begin{proof}
    ``if''. Assume that $M$ does not halt on the empty tape. We
    need construct a model of $\Omc$ that $L$-satisfies $S$. 
    We use the model constructed in the proof of the `if'-direction of Lemma~\ref{lem:dag_corr}.
    Apart from reducing the size and number of ensembles (since they are more sparse in the grid now) and adjusting the refinement function accordingly, this is a model of our ontoloty and $L$-satisfies $S$. 
    
    ``only if''.
     Assume that $S$ is $L$-satisfiable w.r.t.\ \Omc and let
    \Imc be a model of \Omc with $S^{\Imc_L} \neq \emptyset$. We
    identify a grid in $\Imc_L$ in the form of a mapping
    $\pi: \mathbb{N} \times \mathbb{N} \rightarrow \Delta^{\Imc_L}$ such
    that for all $i,j \in \mathbb{N} \times \mathbb{N}$,
    $(\pi(i,j),\pi(i,j+1)) \in c^{\Imc_L}$ and either 
    $(\pi(i,j),\pi(i+1,j)) \in t^{\Imc_L}$ or $(\pi(i,j-1),\pi(i+1,j-1)) \in t^{\Imc_L}$.
    We construct such a $\pi$ in several steps:
  
    \begin{itemize}
  
    \item 
    Start with choosing some $d \in S^{\Imc_L}$ and set
    $\pi(0,0)=d$.
    
    % \item 
    % Start with choosing some $d \in S^{\Imc_L}$ and set
    % $\pi^0_0=d$.
  
    \item Complete the bottom horizontal. If $\pi(i,0)$ is defined and
      $\pi(i+1,0)$ undefined, then choose $d \in \Delta^{\Imc_L}$ such
      that $(\pi(i,0),d) \in t^{\Imc_L}$ and
      set $\pi(i+1,0)=d$.
  
      % \item Complete the bottom horizontal. If $\pi^i_0$ is defined and
      %   $\pi^{i+1}_0$ undefined, then choose $d \in \Delta^{\Imc_L}$ such
      %   that $(\pi^i_0,d) \in t^{\Imc_L}$ and
      %   set $\pi^{i+1}_0=d$.
  
    \item Complete the verticals. 
      If $\pi(i,j)$ is defined and
      $\pi(i,j+1)$ undefined, then choose $d \in \Delta^{\Imc_L}$ such
      that $(\pi(i,j),d) \in c^{\Imc_L}$ and
      set $\pi(i,j+1)=d$.
  
      % \item Complete the verticals. 
      % If $\pi^i_j$ is defined and
      % $\pi^i_{j+1}$ undefined, then choose $d \in \Delta^{\Imc_L}$ such
      % that $(\pi^i_j,d) \in c^{\Imc_L}$ and
      % set $\pi^i_{j+1}=d$.
  
    \item Complete horizontals.
      Now we have to argue that in our grid there are the horizontals as depicted in Figure~\ref{fig:quant_grid}.
      Formally, we show that for each $i,j \in \mathbb{N}$, either $(\pi(i,j),\pi(i+1,j)) \in t^{\Imc_L}$ or $(\pi(i,j-1),\pi(i+1,j-1)) \in t^{\Imc_L}$
  
      We do this by proving the following two statements:
      \begin{enumerate}
        \item if $i$ is even and $j$ is odd, then $(\pi(i,j),\pi(i+1,j)) \in t^{\Imc_L}$;
        \item if $i$ is odd and $j$ is even, then $(\pi(i,j),\pi(i+1,j)) \in t^{\Imc_L}$.
      \end{enumerate}
      The proof is by induction on $j$.
      From the CIs in $\Omc$ and and the already completed bottom horizontal and verticals it follows that for all $i,j \in \mathbb{N}$, $\pi(i,j)$ satisfies $X_k$ in the pattern depicted 
      in Figure~\ref{fig:quant_grid}.
      Thus for Statement~1 and $j=1$, the concept abstraction and refinement that use $q_1$ imply $(\pi(i,1), \pi(i+1,1)) \in t^{\Imc_L}$, as required.
      For odd $j > 1$, we can apply the IH to obtain $(\pi(i,j-2), \pi(i+1,j-2)) \in t^{\Imc_L}$ and use the abstraction and refinement using $q_2$
      to imply $(\pi(i,j), \pi(i+1,j)) \in t^{\Imc_L}$.
  
      The proof for Statement~2 is very similar. For $j=0$ it is already proven by our completed bottom horizontal.
      For even $j > 0$, it again follows that $(\pi(i,j-2), \pi(i+1,j-2)) \in t^{\Imc_L}$ and we can thus use the abstraction and refinement that use $q_3$
      to imply $(\pi(i,j), \pi(i+1,j)) \in t^{\Imc_L}$, as required.

      \end{itemize}

      We can now read off the computation of $M$ on the empty tape from the
      grid in a straightforward way, using the concept names $A_\sigma$
      for the tape content and $A_q$ for the state and head position.
      Since $A_{q_h}^{\Imc_L}=\emptyset$, the computation is non-terminating.
  \end{proof}

  This concludes the proof for $\ALCI^\mn{abs}[\textnormal{ca,cr}]$.
  The observation in Section~\ref{sect:undecbasic} that allows us to express universal quantification for role names
  can be easily adapted to work for inverse roles as well. Let $L_1,L_2$ be two abstraction levels with $L_1 \prec L_2$.
  If we want to simulate the CI $C \sqsubseteq_{L_1} \forall r^-.D$, then the following slight variation 
  of the role refinement and concept abstraction presented in that section works:
  $$
  \begin{array}{l}
   L_2{:}\,A(x) \land r(x,y) \rfn L_1{:}\, r(x,y) \land C(y)\\[1mm]
   L_1{:}\, D \abs L_2{:}\, A(x)
  \end{array}
  $$
  That makes the proof for $\EL^\mn{abs}[\textnormal{ca},\textnormal{cr},\textnormal{rr}]$ a minor 
  variation of the one above and thus we have proven Theorem~\ref{thm:qvar}.
\end{document}